\documentclass[preprint,10pt]{elsarticle}

\usepackage{amssymb}
\usepackage{amsmath}
\usepackage{algorithmic}
\usepackage{algorithm}
\usepackage{epsfig}
\usepackage{graphicx}
\usepackage{color}
\usepackage[scriptsize]{subfigure}
\usepackage{amsmath,amsfonts,amssymb,mathrsfs,amsgen,amsbsy}
\usepackage{amsthm}
\usepackage{amssymb}
\usepackage{thmtools, thm-restate}

\declaretheorem[style=definition,name=Definition]{defn}

\declaretheorem[style=definition,name=Fact]{fact}
\declaretheorem[style=definition,name=Example]{example}

\newtheoremstyle{examplecon}
{\topsep} {\topsep}%
{\upshape}
{}
{\bfseries\scshape}
{.}
{ }
{\thmname{#1} \thmnumber{ #2}\thmnote{#3}\normalfont\enspace(continued)}

\theoremstyle{examplecon}
\newtheorem*{examplecont}{Example}

\def\XC{\mathcal X}

\def\AC{\mathcal A}

\def\EC{\mathcal E}
\def\FC{\mathcal F}

\def\MC{\mathcal M}

\def\KC{\mathcal K}

\def\set#1{\{#1\}}              
\def\tuple#1{\langle#1\rangle}  

\def\And{\wedge}                
\def\LAND{\bigwedge}
\def\Or{\vee}                   
\def\LOR{\bigvee}                

\newcommand{\card}[1]{|#1|}

\newcommand{\virg}[1]{``#1''}

\newcommand{\inv}[1]{\underline{#1}^{+}}

\usepackage[usenames,dvipsnames]{pstricks}

\newcommand{\grounded}[1]{\ensuremath{G_{#1}}}
\newcommand{\logicalcomplement}[1]{\ensuremath{{#1}^{*}}}
\newcommand{\varx}[0]{\ensuremath{\bar{x}}}
\newcommand{\darg}[2]{\ensuremath{(\set{#1}, #2)}}
\newcommand{\n}[0]{\ensuremath{not~}}
\newcommand{\truen}[0]{\ensuremath{\lnot}}
\newcommand{\name}[0]{\ensuremath{\triangleq}}

\newcommand{\wseq}[2]{\ensuremath{\underbrace{#1\cdot \ldots \cdot #1}_{#2}}}
\newcommand{\cone}[0]{\ensuremath{A}}
\newcommand{\ctwo}[0]{\ensuremath{B}}
\newcommand{\cthree}[0]{\ensuremath{C}}
\newcommand{\cfour}[0]{\ensuremath{D}}

\newcommand{\Brian}{\textit{Brian}}
\newcommand{\Brianphone}{\textit{Brianphone}}
\newcommand{\office}{\textit{office}}

\journal{Artificial Intelligence}

\begin{document}

\begin{frontmatter}



\title{Automata for Infinite Argumentation Structures}

\author[unibs]{Pietro Baroni}
\ead{pietro.baroni@ing.unibs.it}

\author[unibs]{Federico Cerutti\corref{cor}}
\ead{federico.cerutti@ing.unibs.it}

\author[liv]{Paul E. Dunne}
\ead{P.E.Dunne@liverpool.ac.uk}

\author[unibs]{Massimiliano Giacomin}
\ead{massimiliano.giacomin@ing.unibs.it}

\address[unibs]{Dipartimento di Ingegneria dell'Informazione, University of Brescia, via Branze, 38, 25123, Brescia, Italy}
\address[liv]{Department of Computer Science, Ashton Building, University of Liverpool, Liverpool L69 7ZF, United Kingdom}

\cortext[cor]{Corresponding author}

\begin{abstract}
The theory of abstract argumentation frameworks ({\sc af}s) has, in the main, focused on finite
structures, though there are many significant contexts where argumentation can be regarded as a process involving
\emph{infinite} objects. 
To address this limitation, in this paper we propose a novel approach for describing infinite {\sc af}s using tools from formal language theory.
In particular, the possibly infinite set of arguments is specified through the language recognized by a deterministic finite automaton while a suitable formalism, called \emph{attack expression}, is introduced to describe the relation of attack between arguments.
The proposed approach is shown to satisfy some desirable properties which can not be achieved through other ``naive'' uses of formal languages.
In particular, the approach is shown to be expressive enough to capture (besides any arbitrary finite structure) a large variety of infinite {\sc af}s including two major examples from previous literature and two sample cases from the domains of multi-agent negotiation and ambient intelligence.
On the computational side, we show that several decision and construction problems which are
known to be polynomial time solvable in finite {\sc af}s are decidable in the context of the proposed formalism and we provide the relevant algorithms. Moreover we obtain additional results concerning the case of \emph{finitary} {\sc af}s.
\end{abstract}


\end{frontmatter}

%
%
\section{Introduction}
\label{section:sec-introduction}
The theory of abstract argumentation frameworks ({\sc af}s) has advanced considerably since its
original formulation in the work of Dung~\cite{dung:1995}. Now recognised as a core research topic within the
field of AI in general and its sub-disciplines concerned with knowledge representation and multiagent systems
in particular, {\sc af}s have proven a powerful modelling tool to address reasoning
issues in contexts where classical deductive logic is not the most suitable technique. An overview
of the role of argumentation in AI may be found in recent surveys such as that
of Bench-Capon and Dunne~\cite{Bench-Capon2007} or the comprehensive collection of introductory articles in
Rahwan and Simari~\cite{RahwanSimari:2009}. In total matters of semantics, algorithms and computational
complexity have occupied many researchers to the extent that their key properties are, now, reasonably
well understood. Partly in consequence of such understanding, a rich body of subsequent work has emerged 
promoting developments of Dung's basic formalism in order to encompass scenarios within which the purely abstract
approach of \cite{dung:1995} is felt to be too limiting. Among the many notable contributions
of this nature one finds proposals such as the preference-based {\sc af}s of Amgoud and Cayrol~\cite{amgoudcayrol2002};
the value-based model of Bench-Capon~\cite{BC03a}; {\sc eaf}s from Modgil~\cite{Modgil:2009}; recursive attacks in
the {\sc afra} mechanism from Baroni~\emph{et al.}~\cite{Baroni-et-al-IJAR:2010}; divers treatments of
weighted frameworks such as Dunne~\emph{et al.}~\cite{DHMcBPW:2011}; as well as sophisticated developments
of the basic ``binary attack'' concept from \cite{dung:1995} such as the {\sc adf} model from
Brewka and Woltran~\cite{BW-KR:2010,BDW:2011} and the constrained {\sc af}s of Coste-Marquis {\em et al.}~\cite{CMDM-KR:2006,Devred-et-al:2010}.

Amidst the wealth and variety of treatments stemming from \cite{dung:1995} one can, however, note a
common factor: invariably discussion is focused on \emph{finite} environments be they finite sets of
basic arguments or finite attack relationships over these. In contrast, consideration of \emph{infinite} scenarios
has been largely neglected.
This turns out to be a limitation from a theoretical, conceptual, and practical perspective.

From a theoretical viewpoint, infinite frameworks extend (and, in a sense, complete) the range of investigation on abstract argumentation semantics and their properties.
In fact, infinite {\sc af}s have been the subject of specific attention in the seminal paper by Dung, whose fundamental results do not rely on finiteness. Subsequently, infinite {\sc af}s have sometimes been considered \emph{per se} as significant testbeds for examining semantics properties which, though holding in the finite case, may be challenging to prove or fail to hold in the infinite case. For instance, the existence of semi-stable extensions \cite{Caminada:2006} is guaranteed for finite frameworks, while an infinite framework admitting no semi-stable extensions has been devised in \cite{camver:2010} (and will be recalled in Section \ref{subsec:camver}) and the existence of semi-stable extensions for \emph{finitary} frameworks has been proved in \cite{Weydert:2011}.
Apart from theoretical interest, this kind of results may be useful to shed new light on fundamental issues underlying the definition of different semantics, thus enabling a broader view and deeper understanding in comparing and assessing them.

From a conceptual perspective, considering finite frameworks corresponds to (i) adopting a closed view of the argumentation process, which is bounded to terminate after a finite number of steps and (ii) excluding reasoning about infinite domains.
Assumption (i) contrasts with the intrinsically open nature of the argumentation process, arising from the fundamental distinction between the concepts
of ``demonstration by \emph{proof}'' and ``persuasion through \emph{argument}''. That is to say, as noted
in \cite[p.~620]{Bench-Capon2007}: ``Arguments are \emph{defeasible}: the reasoning that formed a persuasive case, in
the light of changes of viewpoint or awareness of information not previously available, may fail to convince. This defeasibility
is never removed: an argument may cease to be challenged and so accepted, but the \emph{possibility} of challenge remains.''.
In other words, the finite view can (at best)
describe a ``snapshot'' of the notional complete context within which an argumentation process could evolve due to potential (but as yet unvoiced) challenge to the conclusions derived.
On the other hand, limitation (ii) prevents the use of the abstract argumentation formalism in contexts where reasoning has to deal with open-ended scenarios.
For instance, considering an open time horizon, one may want to encompass the existence of infinite arguments associated with infinitely many time instants: this kind of approach has been proposed by Pollock to model reasoning about temporal persistence of beliefs \cite{Pollock98}.

From the practical perspective, it has to be remarked that systems giving rise to a potentially infinite automated production of arguments may well occur in actual applications.
On one hand, there are well-known correspondences between argumentation frameworks and other kinds of reasoning systems potentially producing infinite derivations: an example concerning logic programming is in fact given in Appendix A of \cite{dung:1995} and will be recalled in this paper.
On the other hand, abstract argumentation is widely adopted as a general tool to model dialogues with different purposes (e.g. deliberation, negotiation, persuasion) between self-interested agents in a multi-agent system.
In such a context the opportunistic behavior of each agent, driven by \virg{selfish} criteria, and the absence of global coordination and shared information may lead to non-terminating argument exchanges \cite{baroni-et-al:self-stab}.

Adding to the considerations above the fact that, as already recalled, Dung's original work not only addresses infinite {\sc af}s as objects of interest but also establishes a number of fundamental properties of such in the context of the basic semantics put forward for {\sc af}s, it appears that the very limited coverage of infinite frameworks in the subsequent literature represents an important and, to some extent, surprising lacuna in the field.

This work contributes to fill this gap by addressing the problem of defining finite specifications of infinite {\sc af}s through formal languages and proposing an approach based on finite automata.
This proposal consists of two basic elements: a description of the (infinite) set of arguments through a finite automaton, and a description of the attack relations linking arguments together through an \emph{attack expression}.
More precisely, the attack expression specifies a mapping between regular expressions describing sets of arguments, with the intended meaning that a set $S$ is attacked by the elements of the set obtained from $S$ through the mapping. The combination of the automaton describing the set of arguments and of the attack expression will be called the \emph{{\sc af} specification} and will be shown to be expressive enough to encompass a variety of infinite argumentation frameworks, including the major examples available in the literature.

Clearly, a sufficiently expressive specification formalism needs to be complemented by some suitable computational mechanism for the evaluation of argument justification status, which represents the main goal of any application of computational argumentation, either in finite or infinite contexts.
The proposed approach is shown to be satisfactory also from this viewpoint, being able to support the definition of suitable algorithms for some \virg{standard} computational problems in argumentation semantics.

The paper is organised as follows. 
We lay down the context of the work, discuss motivations, and introduce application examples in Section \ref{section:motivations}.
We then recall the necessary technical background on Dung's argumentation framework and on formal languages in 
Section \ref{section:sec-background}.
In Section \ref{section:naive-rep} we state a set of basic requirements for a representation formalism for infinite frameworks and show that they are not satisfied by a straightforward \emph{naive} approach one may adopt.
Section \ref{section:sec-generic-representation} introduces and illustrates with a trailing set of examples the \emph{{\sc af} specification} formalism for infinite frameworks and shows that it is expressive enough to capture also \virg{regular} finite frameworks and frameworks that can be regarded as composed by a finite subframework and one or more infinite subframeworks.
Section \ref{sec:computing} introduces effective computational procedures\footnote{According to \cite[p.~55]{Rosser1939} (but also \cite[p.~210]{Mendelson2010} and others) an \emph{effective computational procedure} is ``a method each step of which is precisely predetermined and which is certain to produce the answer in a finite number of steps''.} for several standard decision problems in argumentation semantics in the context of the {\sc af} specification formalism.
Section \ref{section:example-cases} demonstrates the suitability of the approach, both on the representation and on the computational side, by analyzing in detail its application to four examples: two infinite frameworks previously introduced in the literature for the sake of theoretical analysis and two examples of infinite argumentation in multi-agent systems taken from Section \ref{section:motivations}.
Section~\ref{sec-related} discusses related works  
whilst Section \ref{sec-conclusions} concludes and
proposes some areas for future work. \ref{subsection:flt} recalls the basics of formal language and automata theory to make the paper self-contained, while proofs of all technical results are collected in \ref{appendix-proofs}.

%
%

\section{Context and motivations}\label{section:motivations}

Quoting Prakken \cite{Prakken2010}, Dung's paper on abstract argumentation framework \virg{was a breakthrough in three ways: it provided a general and intuitive semantics for the consequence notions of argumentation logics (and for nonmonotonic logics in general); it made a precise comparison possible between different systems (by translating them into his abstract format); and it made a general study of formal
properties of systems possible, which are inherited by instantiations of his framework.}
Due to its abstract nature, Dung's formalism \virg{is best seen not as a formalism for directly representing argumentation-based inference problems but as a tool for analysing particular argumentation systems and for developing a metatheory of such systems}.
In this perspective, investigation of infinite {\sc af}s finds its motivations in the variety of more concrete contexts and systems where infinite structures play a role and that can be modeled using {\sc af}s. A (non exhaustive) account is given in the following subsections.

\subsection{Argumentation models and systems}

Infinite entities have been encompassed at a foundational level by all the main literature approaches to formal argumentation.
In fact, both before and after Dung's work, other \virg{less abstract} approaches have been defined which formalize arguments and their structure in various ways, with the common property of referring to a generic (and often only partially specified) language, encompassing infinite structures.

In the formalization of defeasible reasoning by Simari and Loui \cite{SimariLoui92} an \emph{argument structure} comprises a possibly infinite set of sentences supporting a conclusion, and, in turn, the set of all possible argument structures is infinite.

In the theory of \emph{assumption-based argumentation} \cite{bondar-etal:1997} an assumption-based framework consists of a set of beliefs and a set of assumptions which are possibly infinite subsets of a language consisting of countably many sentences. Again the infinite case is explicitly considered in the theoretical analysis of semantics properties in the framework.

In Vreeswijk's \emph{abstract argumentation systems} \cite{vreeswijk:1997a,Baronietal2000} arguments are restricted to have a finite set of premises, but the set of arguments is possibly infinite and infinite argumentation sequences (involving either finite or infinite sets of arguments) are used as a formal tool for extension evaluation, by introducing a notion of \virg{limit} for infinite argumentation sequences. In \cite{vreeswijk:1997a} it is remarked in particular that some desirable limit properties of infinite argumentation sequences may not hold when an infinite set of arguments is considered. This is left as an open problem, which, as to our knowledge, has still to be solved.

In DEFLOG \cite{Verheij03} the central notion is the \virg{dialectical interpretation} of a theory, which basically is a (possibly infinite) set of sentences related by two connectives representing support and defeat relations respectively.

Finally, in the more recent ASPIC formalism \cite{Amgoudetal2006} and its ASPIC+ evolution \cite{Prakken2010} arguments with an infinite number of subarguments (and hence infinite sets of arguments) are encompassed and none of the main properties relies on argument finiteness.

It can hence be stated that the consideration of infinite structures and derivations has been consistently regarded as a basic, one could even say \virg{natural}, feature in argumentation formalisms. It has however to be acknowledged that this feature has typically been regarded as problematic when moving from (more or less abstract) theory to complete specification (and implementation) of actual argumentation-based reasoning systems.
In fact, the unbounded open nature of argumentative reasoning has often been contrasted with the practical needs and limitations of resource-bounded agents.

From a philosophical stance, this contrast has been pointed out by Pollock \cite{Pollock92} by introducing the distinction between \emph{justified beliefs} and \emph{warranted propositions}.
Quoting Pollock, \virg{at each stage of reasoning, if the reasoning is correct then a belief held on the basis of that reasoning is justified, even if subsequent reasoning will mandate its retraction} while \virg{in contrast to justification, warrant is what the system of reasoning is ultimately striving for. A proposition is warranted in a particular epistemic situation if and only if, starting from that epistemic situation, an ideal reasoner unconstrained by time or resource limitations would ultimately be led to believe the proposition}.
In this view, warrant may be regarded as a sort of unattainable goal for a pratical resource-bounded agent, which needs to be content with the more limited notion of justification.

From the practical side, an example of the problems in dealing with infinite structures is presented in \cite{GarciaSimariDeLP} in the context of the \emph{DeLP (Defeasible Logic Programming)} system.
Here, one of the central notions is the one of an \virg{argumentation line}, which is basically a sequence of argument structures where each element of the sequence is a defeater of its predecessor.
As discussed in \cite{GarciaSimariDeLP} infinite argumentation lines may easily emerge for various reasons (e.g. self-defeating arguments, reciprocal defeaters, non-concordant sets of supporting arguments). Since managing infinite argumentation lines is regarded as undesirable, suitable restrictions are introduced in the formal definition of argumentation line to avoid these cases.

A different kind of restriction is adopted in the argumentation-based approach to \emph{Defeasible Logic} proposed in \cite{Governatorietal04}. Here, an argument is a possibly infinite proof tree for a literal $p$ (associated with the tree root). By definition, however, only finite arguments can be acceptable (this rather drastic choice is motivated by the goal of avoiding the risk of supporting \virg{well-known fallacies such as circular argument and infinite regress}) while infinite arguments keep anyway the power to prevent justification of other arguments.

In the context of the logic-based approach to argumentation of Besnard and Hunter \cite{BesnardHunter2008} the classical chicken and egg dilemma is used as a common sense reasoning example giving rise to an infinite sequence of arguments, each being a counterargument to the preceding one (such a sequence is called \emph{dispute} in this context).
In fact, dilemmas (of various nature and possibly more interesting than chicken and egg) are recognized as a significant case of infinite reasoning with conflicting arguments also in non-technical literature.

An example is given in the novel \virg{Runaround} by Isaac Asimov where, on the planet Mercury, a robot, called SPD-13, receives by two spacemen the order to accomplish a mission which requires to collect selenium from a pool. 
The robot is programmed to obey three basic rules which can be synthesized as follows. The first rule states that the robot has to protect human lives.
The second rule states that the robot has to obey orders unless they conflict with the first rule.
The third rule states that the robot has to protect itself unless this conflicts with the first rule.
After a long wait, the spacemen send out another (less capable) robot to look for SPD-13. From the report of the second robot they realize that the mind of SPD-13 is in a loop (which has caused a sort of \virg{drunkenness}) which can be described as follows: when the robot gets near the selenium pool it perceives an unforeseen danger, the third rule is activated and the robot builds an argument to go away, prevailing on the previous decision to obey the order. When the robot is sufficiently far from the pool, its danger perception decreases and so according to the second rule it builds a new argument which leads it to turn back towards the pool, prevailing over the previous decision. When it gets sufficiently closer to the pool, it feels the danger again and the process restarts.

From the representation point of view in \cite{BesnardHunter2008} this kind of problems is tackled by imposing some restriction in the definition of the \emph{argument tree} structure, which is meant to capture all the disputes concerning a specific argument (represented by the root of the tree). More precisely, the premises of an argument added to the tree are forbidden to be a subset of the union of the premises of its ancestors. With this constraint, the argument tree for the chicken and egg dilemma reduces to a two-length chain, but, as observed in \cite[p.~62]{BesnardHunter2008} \virg{the argument tree is merely a representation of the argumentation} and \virg{altough the argument tree is finite, the argumentation here is infinite and unresolved}.

An explicit restriction to finite structures is also adopted in a recent work concerning the study of postulates and properties of logic-based instantiations of abstract argumentation \cite{GorogiannisHunter11}. In this work, a propositional logic with a countable set of propositional letters is used as a basis for the argumentation process. It ensues that the set of all arguments is countably infinite (though, by definition, an argument is assumed to be built on a finite set of formulae). However, when introducing the notion of \emph{argument graph}, where nodes are arguments and the arcs represent the attack relation, the authors restrict the consideration to graphs with a finite number of nodes.  

\subsection{Multi-agent systems}\label{subsec:multiag}

From the previous subsection it appears that while the potential existence of infinite structures is widely acknowledged in non-abstract argumentation contexts too, there is a prevailing attitude to overlook the difficult problem of actually managing them, by ascribing their genesis to undesirable/pathological conditions which can be avoided at the implementation level with proper programming and preventive checks on the knowledge base.
While it is certainly true that infinite argumentation structures may arise from uninteresting/undesirable conditions, we remark that these do not exhaust the range of cases where such structures may arise and that, whatever the underlying reason, their emergence can not always be prevented.
In fact, there are concrete situations where systematic well-founded argument generation mechanisms may incur in an open-ended non-terminating behavior.

Multi-agent (and more generally distributed) systems provide a major case for this statement, under the non restrictive and fairly standard assumptions of self-interest and absence of a global reasoner to which all information is available.

For instance in \cite{BikakisAntoniouKDE2010} argumentation semantics of Defeasible Logic is extended to the case of a multi-context system for distributed ambient intelligence. Each context corresponds to the local knowledge and reasoning of an agent and arguments
of different contexts are interrelated through mapping rules. As to undesirable circularities, it is observed that \virg{loops in the local knowledge bases can be easily detected and removed without needing to interact with other agents. However, even if there are no loops in the local theories, the global knowledge base may contain loops caused by mapping rules.}
Such loops in the global knowledge base may cause infinite argumentation lines. 
In \cite{BikakisAntoniouKDE2010} this problem is dealt with by adopting the specific assumptions that (i) each agent uses its own vocabulary and is therefore the unique responsible of the evaluation of some literals, (ii) the agents behave in a fully cooperative manner in the process of justification evaluation.

In fact, unless one adopts some restrictive assumptions of this kind, the possible onset of non terminating behaviors in argument defeat status computation is inherent to multi-agent system.
This is formally proved in \cite{baroni-et-al:self-stab} where it is shown that approaches to distributed defeat status computation (see for instance \cite{McBurneyParsons01,Parsonsetal98}) usually rely on assumptions like a predefined unmodifiable number of agents, the existence of a centralized structure, or the obligation to reveal the entire inner structure of the arguments an agent has built.
Removing these restrictions and assuming a multi-agent system with the general properties of \emph{unlimited cardinality, autonomy, asynchronism, dynamism, and uncostrained communication} (see \cite{baroni-et-al:self-stab} for details), an impossibility result is obtained showing that no self-stabilizing algorithm can exist for defeat status computation according to any semantics which is \emph{valid} (namely obeys some fairly general constraint on the defeat status assignment).
The paper provides two practical examples of non terminating behavior: a distributed version of the three liars paradox introduced by Pollock \cite{Pollock94} and a negotiation dialogue for resource exchange among three agents\footnote{By the way, in the context of this example, in \cite{baroni-et-al:self-stab} it is remarked that the existence of circularities at the global level is not critical \emph{per se}, since they do not give rise to any problem if there are further attacking arguments breaking the cycle. Thus, simply forbidding cycles turns out to be a too drastic measure in general.}.

In the following subsections we provide two extended examples of infinite argumentation in multi-agent systems, namely an infinite negotiation process and a distributed reasoning process involving the components of an ambient intelligence system. These examples are inspired to the application contexts considered in \cite{baroni-et-al:self-stab} and \cite{BikakisAntoniouKDE2010} and will be used in Section \ref{section:example-cases} to demonstrate the application of the formalism proposed in this paper.

\subsubsection{An example in multi-agent negotiation}\label{subsubsec:intr-ex-negot}

In multi-agent systems, independent and possibly self-interested components strive to achieve common or individual goals by various forms of interaction (cooperation, negotiation, persuasion, resource exchange, task allocation, \ldots) for most of which argumentation is considered a suitable model in the literature (see for instance \cite{Rahwan05} and the references thereof).

If one considers interactions involving more than two agents and removes some not always realistic assumptions (e.g. that information on all argument exchanges is available to all agents) the interaction process may not reach a solution and continue forever with an infinite production of arguments.

To exemplify, consider a simple negotiation setting where three agents $A_1$, $A_2$ and $A_3$ may exchange resources called $R_a, R_b, R_c$. Each agent possesses some resource and has its own preference ordering on resources. Each agent is only partially informed on the resources owned by other agents and can not know the preferences of other agents.
At a given time instant, an agent $A_x$ builds an argument for proposing to $A_y$ an exchange between resources $R_x$ and $R_y$ if the following conditions holds: (i) $A_x$ owns $R_x$; (ii) $A_x$ knows that $A_y$ owns $R_y$; (iii) $A_x$ prefers $R_y$ to $R_x$.
An agent receiving a proposal may accept or reject it at a later time: the agent who has sent the proposal is free to withdraw it before receiving confirmation of acceptance (typically because the agent has received a more convenient exchange proposal which conflicts with the previous one). 
An agent is free to reiterate a proposal after having withdrawn it (typically because the reason to withdraw does not hold any more) and is obliged to withdraw an offer s/he has made before accepting an incompatible offer s/he has received.
Message exchanges between two agents are not available to other agents but they are collected by an authority supervising the negotiation arena. The authority is informed on all the exchanged messages and on the resources possessed by all agents but has not access to agents' preference rankings.
The authority may therefore build an argumentation framework representing the evolution of the negotiation process and may help agents to overcome critical situations. 
We suppose that the attack relation in the argumentation framework managed by the authority is defined on the basis of the two following rules:
\begin{itemize}

\item a (possibly reiterated) proposal $P_1$ (received or sent by an agent $A_i$) attacks a (possibly reiterated) proposal $P_2$ (received or sent by the same agent $A_i$) if accepting the exchange proposed in $P_1$ makes impossible the exchange proposed in $P_2$;

\item a withdrawal obviously attacks the withdrawn proposal, a reiterated proposal attacks the corresponding previous withdrawal.
\end{itemize}

Suppose now that the initial situation is the one described in Table \ref{tab:negotiation-start}.

\begin{table}[htb]
  \centering
  \begin{tabular}{| c | c | c | c |}
  \hline
  Agent ID & Owns  & Knows & Preference rank \\
  \hline
  $A_1$    & $R_c$ & $A_2$ owns $R_b$  & $R_a > R_b > R_c$ \\
\hline 
  $A_2$    & $R_b$ & $A_3$ owns $R_a$  & $R_c > R_a > R_b$ \\
\hline 
  $A_3$    & $R_a$ & $A_1$ owns $R_c$  & $R_b > R_c > R_a$ \\
  \hline
  \end{tabular}
  \caption{Initial situation of the negotiation example.}
  \label{tab:negotiation-start}
\end{table}

Then each agent builds an offer as follows:
\begin{itemize}

\item $A_1$ sends an offer to $A_2$ proposing an exchange between $R_c$ and $R_b$: $O_1 = Off(t_0, (A_1,A_2, Exch(R_c,R_b)))$

\item $A_2$ sends an offer to $A_3$ proposing an exchange between $R_b$ and $R_a$: $O_2 = Off(t_0,(A_2,A_3, Exch(R_b,R_a)))$

\item $A_3$ sends an offer to $A_1$ proposing an exchange between $R_a$ and $R_c$: $O_3 = Off(t_0, (A_3,A_1, Exch(R_a,R_c)))$
\end{itemize}

Clearly each offer is incompatible with the two others.

It can be seen that each agent prefers the status resulting from the exchange in the offer s/he has received wrt the one resulting from the offer s/he has made.
For instance, agent $A_1$ prefers exchanging $R_c$ with $R_a$ (as proposed by $A_3$) than exchanging $R_c$ with $R_b$ (as s/he has proposed to $A_2$).
Then, let say at time $t_1$, each agent sends a message of withdrawal of the previous offer: $W_1 = Wd(t_1, (A_1,A_2, Exch(R_c,R_b)))$, $W_2 = Wd(t_1,(A_2,A_3, Exch(R_b,R_a)))$, 
$W_3 = Wd(t_1, (A_3,A_1, Exch(R_a,R_c)))$.

As a consequence of the withdrawal and of its local view, each agent is now, let say at time $t_2$, in a position where the only reasonable move is to reiterate the initial offer (let say that these messages are denoted as $O_4$, $O_5$, $O_6$): clearly this reproduces the initial situation, causes three further withdrawals and the process goes on forever\footnote{Note that a similar situation occurs also if we assume that each agent updates its knowledge on who owns what after the first round of offers. In that case the roles of bidder and addressee would be interchanged (e.g. the exchange of $R_b$ with $R_c$ would be proposed by $A_2$ to $A_1$ and so on), but the non-terminating sequence of offers and withdrawals would occur in the same way.}.

\subsubsection{An example in ambient intelligence}\label{subsubsec:intr-ex-ambient}

Consider a system of ambient intelligence consisting of several independently developed interacting components, some of which join and leave dynamically the system, as described in \cite{BikakisAntoniouKDE2010}.

Adapting an example presented in \cite{BikakisAntoniouKDE2010} suppose that the system includes the following components:
\begin{itemize}
\item a people locator;
\item a video surveillance system for each room;
\item a lighting management system for each room;
\item personal smartphones.
\end{itemize}

The components interact as follows:

\begin{itemize}
\item personal smartphones notify their position to the people locator;
\item the video surveillance system notifies the results about people detection to the people locator;
\item the lighting management system has a light sensor and informs the video surveillance system whether each room is dark or not;
\item the people locator informs the lighting system about people's presence in each room.
\end{itemize}

The people locator uses the following rules:

\begin{itemize}
\item[R1:] if a smartphone is in a room the smartphone owner is in the room
\item[R2:] if the video surveillance notifies the presence of a person in a room then there is a person in the room
\item[R3:] if the video surveillance notifies the absence of any person in a room then there is no person in the room
\item[R4:] if a person is present in a room at time t then the person is present in the room at time (t+1)
\item[R5:] if a person is not present in a room at time t then the person is not present in the room at time (t+1)
\end{itemize}

The video surveillance system uses the following rules: 

\begin{itemize}
\item[R6:] if it is not dark and the image processing system recognizes a person in a room then the video surveillance system notifies the presence of a person in the room 
\item[R7:] if it is not dark and the image processing system does not recognize any person in a room then the video surveillance system notifies the absence of any person in the room 
\end{itemize}

The lighting system uses the following rules:

\begin{itemize}
\item[R8:] if it is dark in a room and a person is in the room switch the room lights on
\item[R9:] if the lights are on in a room and no person is in the room switch the room lights off
\end{itemize}

The people locator uses two default persistence rules (R4 and R5) which can be applied in absence of new information and are the weakest ones: in case of conflicting conclusions those derived using R4 and R5 are overruled by those derived using R1, R2, and R3.
Moreover video surveillance is regarded as providing more reliable information than the mere presence of a smartphone, hence R3 is stronger than R1.

To make the presentation compact, let us omit the details concerning message exchanges among the various components and, as a consequence, combine rules together where possible. The set of rules presented above can be represented by the following logic program\footnote{The program, with the restriction to a finite time horizon, has been run in DLV.} $M$, where $\n$ denotes \emph{negation as failure} and $\truen$ denotes \emph{explicit} negation.

{\small
\noindent
$\begin{array}{l l l}
M: & & \\
in(p, r, t) & \leftarrow phone(x), person(p), owner(x, p), room(r), phonein(x, r, t), & \\
& \n videovalid(r, t) & (r1) \\
in(p, r, t) & \leftarrow room(r), person(p), videorecogn(p, r, t), videovalid(r, t) & (r2)\\
\truen in(p, r, t) & \leftarrow  room(r), person(p), \truen videorecogn(p, r, t), videovalid(r, t) & (r3) \\
in(p, r, s(t)) & \leftarrow  person(p), phone(x), room(r), owner(x, p), in(p, r, t), & \\
& \n videovalid(s(t)), \n phlocated(x, s(t)) & (r4)\\
\truen in(p, r, s(t)) &\leftarrow  person(p), phone(x), room(r), owner(x, p), \truen in(p, r, t), & \\
& \n videovalid(r, s(t)), \n phlocated(x, s(t)) & (r5)\\
lighton(r, s(t)) &\leftarrow  room(r), person(p), in(p, r, t), dark(r, t) & (r6)\\
\truen lighton(r, s(t)) & \leftarrow  room(r),  person(p), \truen in(p, r, t) & (r7)\\
phlocated(x, t) & \leftarrow phone(x), room(r), phonein(x, r, t) & (r8)\\
videovalid(r, t) & \leftarrow \n dark(r, t), room(r) & (r9)\\
\end{array}$
}

Most predicates in $M$ have self-explaining names. We assume that variable $t$ refers to time instants, which are discrete and totally ordered, and that $s(t)$ denotes the successor of instant $t$.
We assume that information about the presence of smartphones, darkness, and the outcome of the video recognition system is available for each room in the form of asserted or explicitly negated facts at each time instant, as provided by the relevant devices. In particular, we assume that the image processing component returns either $videorecogn$ or $\truen videorecogn$ for each triple $(p, r, t)$ and, of course, does not recognize any person when it is dark.
We also assume that all devices are properly working so that, in particular, when light is on in a room $r$ at time $t$ the predicate $dark(r, t)$ is false.

R1 is represented by line $(r1)$ of $M$ with the condition $\n videovalid(r, t)$ to ensure priority to the video surveillance system when its output is valid, namely when it is not dark in the room, as specified by $(r9)$.
Line $(r2)$ synthesises rules R2 and R6, and, similarly, $(r3)$ synthesises R3 and R7. 
The persistence rules R4 and R5 are represented by lines $(r4)$ and $(r5)$ with the conditions $\n videovalid(s(t))$ and $\n phlocated(s(t))$ to ensure that other rules prevail when information from devices is available and valid (for phones $(r8)$ applies, where the predicate $phlocated(x, t)$ means that the location of phone $x$, whatever it is, is known at instant $t$).
Rules R8 and R9, concerning the lighting system are represented respectively by lines $(r6)$ and $(r7)$. 

Suppose now that Brian at time instant $0$ (when outside is dark) exits the office, switches the light off and forgets his smartphone. 

It follows that, applying $(r1)$, $in(\Brian, \office, 0)$ is derived and as a consequence, by $(r6)$, light is switched on at instant $1$.
Then the room is no more dark, $videovalid(\office, 1)$ holds by $(r9)$ and since $\truen videorecogn(\Brian, \office, 1)$ holds, by $(r3)$ $\truen in(\Brian, \office, 1)$ is derived and, applying $(r7)$, the light is switched off.
As a consequence at instant $2$ the room is dark and $videovalid(\office, 2)$ can not be derived. $(r1)$ then applies and $in(\Brian, \office, 2)$ is derived, it follows that, by $(r6)$, light is switched on at instant $3$, and so on.

\subsubsection{Non-cooperative dialogues}

To avoid situations of the kind described in Sections \ref{subsubsec:intr-ex-negot} and \ref{subsubsec:intr-ex-ambient}, most argumentation-related dialogue protocols in the literature (see for instance \cite{Vreeswijk93, Vrees-Prakk:2000, dunnebc:2003, Prakken2005, McBurneyParsons09}) concern the two-party case, which implies in particular that all moves are known to all dialogue participants, and assume that both participants accept some rules (in particular some kind of non repetition constraint) in order to guarantee termination\footnote{In \cite{Vrees-Prakk:2000} it is remarked however that the case of infinite proofs is problematic and is left for future developments.}.
While, as already remarked, this guarantee can not be extended to more general contexts, it can also be observed that works encompassing non-cooperative, and hence potentially infinite, two-party dialogues have been considered in the literature (see for instance \cite{GabbayWoods2001a,GabbayWoods2001c}).
In particular, 
in \cite{GabbayWoods2001a} the authors describe several kinds of non-cooperative dialogue games, such as so-called \emph{stone-walling} tactics. 
An example \cite{GabbayWoods2001a}[p.178] is the following game between two agents, the proponent ($P$) and the opponent ($O$):

\begin{example}
\label{ex-gabbay1}
Example of stone-walling:

\begin{enumerate}[${m}_1$:]
\item \textbf{Assert} $\alpha$ ($P$)
\item \textbf{Reject} $\alpha$ ($O$)
\item \textbf{Assert} $\alpha$ ($P$)
\item \textbf{Assert} $\beta$ ($O$)
\item \textbf{Argue since} $\alpha$, \textbf{either} $\lnot \beta$ \textbf{or} $\beta \nvdash \lnot \alpha$ ($P$)
\item \textbf{Assert} $\lambda$ ($O$)
\item \textbf{Argue since} $\alpha$, \textbf{either} $\lnot \lambda$ \textbf{or} $\lambda \nvdash \lnot \alpha$ ($P$)
\item \ldots	
\end{enumerate}
\end{example}

$P$ makes substantially the same move from $m_5$ onwards: this can be interpreted as being convinced that 
$\alpha$ cannot be false and that no case against $\alpha$ can succeed. This kind of non-cooperation structure is 
called \emph{mind closed} and it is easy to imagine such a dialogue continuing forever if $P$'s mind does not change.

In the previous example, stone-walling is done by one party, $P$. In \cite{GabbayWoods2001a}[p.181], however, 
the authors define a quarrel as a reciprocal stone-walling dialogue.
As the authors remark, the quarrel model shows up frequently in actual dialogical practice, and they suggest that this 
is an efficient way of playing the game of \emph{dialectical fatigue}. Dialectical fatigue settles a 
dispute, and declares a win for the party whose opponent just gives up\footnote{Though not explicitly mentioned, the notion of ``dialectical fatigue'' and its exploitation by self-interested agents underpins the examples discussed
in \cite{Dunne-icail-2005}.}:

\begin{example}
\label{ex-gabbay2}
Example from \cite{GabbayWoods2001a}[p.182]:

\begin{enumerate}[${m}_1$:]
\item \textbf{Assert} $\alpha$ ($P$)
\item \textbf{Reject} $\alpha$ ($O$)
\item \textbf{Assert} $\alpha$ ($P$)
\item \textbf{Reject} $\alpha$ ($O$)
\item \ldots
\item[$m_n$:] \textbf{Assert} $\alpha$ ($P$)
\item[$m_{n+1}$:] \textbf{Perhaps you are right} ($O$)
\end{enumerate}
\end{example}

One may wonder what the outcome of the dialogues in Examples \ref{ex-gabbay1} and \ref{ex-gabbay2} ought to
be from a computational perspective. If the traditional \emph{termination rule} is adopted, 
the outcome is that the proponent wins if the opponent concedes the main claim, and the opponent wins
if the proponent retracts the main claim \cite{Prakken2005}.
However, this rule can not be applied to the non-terminating Example \ref{ex-gabbay1}. Moreover termination in Example \ref{ex-gabbay2} is due to fatigue of one of the players, so that this outcome may be regarded as \virg{non-rational}, while a \virg{rational} development of Example \ref{ex-gabbay2} would be non-terminating too.
One may also observe that the termination rule imposes that one position prevails over the other one, while one might consider also the case where the two positions are regarded as equally acceptable.

Besides issues concerning dialogues, multi-agent systems provide a further case for non terminating argumentation in presence of reasoning about mutual beliefs. In fact an agent $ag1$ able to reason about the beliefs of another agent $ag2$, may take into account also the ability of $ag2$ to reason about the beliefs of $ag1$ in turn, then both $ag1$  and $ag2$ may reason about the mutual beliefs about beliefs an so on \emph{ad libitum}.
This kind of problem is exemplified, in a common sense setting, by the novel by the Argentine writer Osvaldo Soriano \virg{The longest penalty ever} \cite{Soriano93}.
Here, in a football game, a goal keeper has to reason about whether to dive to the left or to the right.
The keeper knows that the kicker in the past has always kicked to the right: this would be a reason to dive to the right, but the keeper also knows the kicker knows that the keeper knows his past records (and might therefore decide to kick to the left) and this would be a reason to dive to the left, but in turn the kicker knows that the keeper knows that the kicker knows $\ldots$, and the chain of mutually attacking arguments supporting the decision of diving to the left or to the right grows to infinity. Of course, in a non cooperative context this growth can not be prevented by mutual agreement and, for either agent, stopping the reasoning at a given level represents an arbitrary, and possibly not appropriate, choice.

\subsection{Reasoning with unbounded domains}

Leaving apart non-terminating situations in multi-agent systems, a further example of infinite argumentation concerns reasoning with unbounded domains like time or space.

For instance, reactive systems are characterized by \virg{their perpetual interaction with their environment as well as their nonterminating behaviour}\cite{Gradeletal2002} and as such require models able to encompass infinite objects like automata over infinite words or infinite games.  
While these models are suitable to analyze properties of these systems in a \emph{monotonic} reasoning context, different issues arise and different formalisms are needed in case some kind of nonmonotonic reasoning is carried out.

An example of use of argumentation in a nonmonotonic open-horizon context is provided by Pollock \cite{Pollock98}, who introduced a \emph{temporal projection principle} to address the problem of argumentation-based reasoning on \virg{stable} properties of the world. Quoting Pollock \virg{the built-in epistemic
arsenal of a rational agent must include reason-schemas of the following sort for at least some
choices of P: if $t_0 < t_1$, believing $P-at-t_0$ is a defeasible reason for the agent to believe $P-at-t_1$}. To allow new information to override presumptions based on out-of-date perceptions it is necessary that \virg{the strength of the presumption that a stable property will continue to hold over time decays as the time interval increases}. 
Though not explicitly addressed by \cite{Pollock98}, it is straightforward to consider a spatial version of this projection principle too.
For instance if one has a reason to believe that a certain site is highly dangerous due to pollution or contamination, then this belief can be reasonably projected to all the neighbour locations with strength decaying as the distance from the contaminated site increases.
The set of arguments (with different strength) that can be produced on the basis of this kind of projection principles is, in general, unbounded. To cope with this, in the OSCAR implementation described in \cite{Pollock98}, Pollock restricts the use of the temporal projection principle to a specific form of backward reasoning: the agent is interested in the value of a property at a specific time instant $t$ and checks whether there are reasons to believe that the property had a certain value at an instant $t_0<t$. If this is the case, the reason can be projected, with decreased strength, from $t_0$ to $t$.
However explicit representations of infinite arguments are needed to go beyond this specific form of reasoning. 

Formulating (defeasible) previsions on the basis of (discrete) series of past observations is a further form of reasoning involving similar issues as the set of possible observations is countably infinite and the set of actual observations may, in general, grow indefinitely.
To give an example, consider previsions concerning sport events (e.g. soccer matches) based on previous performances of the teams with defeasible rules of the kind \virg{A team which has won the majority of past matches will win future matches} and \virg{A team which has lost the last three matches will loose next match} (or more complex ones with similar structure).
Here the observation of the outcomes for a team is constantly updated after each match giving rise to new arguments representing new previsions (possibly conflicting each other and with previous ones). The set of generated \virg{previsional} arguments is (at least in principle) infinite at each step as some of the previsions could be projected over the set of all future matches, e.g. for a very strong team with very long tradition one can, reasonably but defeasibly, foresee further wins for many years to come.
It can be observed that open-ended horizons of this kind can be managed in practice by considering a finite temporal window excluding time instants which are \virg{too far} in the past or in the future.
It can be also observed, however, that if such a temporal window is very large it can be anyway more convenient in practice to adopt compact specification techniques for infinite frameworks of the kind we propose in this paper rather than to deal explicitly with all the elements of finite (but very large) sets and that, in any case, an open-ended representation is more appropriate for reasoning concerning long-term trends and scenarios.

As already mentioned, unbounded time horizons have been considered also in game theory, where infinite games are meant to represent open-ended (e.g. life-long) interactions between the players.
Different kinds of infinite games can be considered.
In iterated variable-sum non-cooperative games, like the iterated prisoner dilemma \cite{AxelrodHamilton81}, each player has a payoff at each turn and seeks strategies maximizing the value of the infinite series of the payoffs.
Quite differently, in Gale-Stewart two-player zero-sum games, one of the two player wins depending on the membership of the infinite sequence of the moves played by both players to a predefined \emph{payoff} set. Here a winning strategy for a player is a function for choosing the next move which ensures the membership (or non-membership) of the resulting infinite sequence to the payoff set.
There is a potential rich interplay between infinite games and infinite argumentation. On one hand, infinite argumentation frameworks can be used as an abstract model for some game-theoretic problems and, especially thanks to their rich endowment of alternative semantics, may suggest variants and open new perspectives for these problems, in the same spirit as done by Dung for the stable marriage problem (see \cite[Sect. 3.2]{dung:1995}).
On the other hand existing results on infinite games may provide a formal basis for the open investigation area on infinite argumentation games and the relevant strategies, building on the standpoint that non-termination does not mean necessarily indeterminacy.

\subsection{Motivation summary}

Summing up, it appears that investigation on infinite argumentation structures got somehow stuck in a sort of deadlock situation.
From a theoretical point of view, their fundamental role has been consistently acknowledged and they have been universally encompassed at a definitional level, but actual formalisms to deal with them at an operational level, i.e. compact representations along with computational procedures, have not been developed, possibly due to a \virg{lack of pressure} from the application side.
In turn, the potential emergence of infinite argumentation structures has been evidenced in a variety of application contexts, but, possibly due to the lack of suitable operational approaches from the theoretical side, they have generally been disregarded as problematic or dealt with by adopting specific workarounds.

The present work contributes to overcome this situation by proposing an approach to compact specification of infinite abstract argumentation frameworks endowed with effective computational procedures.

The approach is suitable, in general, to describe infinite argumentation frameworks with some kind of regular structure.
This covers, in fact, the cases of practical interest, since it corresponds to the generation of arguments (and of the attacks between them) by some systematic non terminating mechanism, as it may occur in a multi-agent system or in other automated reasoning contexts, as described above. As it will be better commented later, the approach is also suitable to manage cases where argument generation terminates but the resulting framework is so large to make a compact representation advantageous.
Our proposal can therefore be regarded as a novel enabling technique with respect to the long-term goal of deploying extended argumentation-based reasoners covering also the case of infinite (or very large) frameworks.
In a shorter-term perspective, the results in this paper provide a formal basis for incorporating the management of infinite frameworks in existing implementations of Dung's style argumentation like ASPARTIX \cite{Eglyetal2010} or Dungine \cite{Southetal08}.

%
%

\section{Background notions}\label{section:sec-background}

\subsection{Argumentation frameworks}
In this section we review the elements of Dung's abstract argumentation frameworks \cite{dung:1995} and the relevant semantics notions and basic computational issues.

\begin{defn}\label{defn:standard-afs}
An \emph{argumentation framework} ({\sc af}) is defined as a pair $\tuple{\XC,\AC}$ in which $\XC$
is a set of \emph{arguments} and $\AC\subseteq\XC\times\XC$ describes the \emph{attack relation} between
arguments in $\XC$, so that $\tuple{x,y}\in\AC$ indicates ``the argument $x$ \emph{attacks} the argument $y$'' (or, equivalently,
``the argument $y$ \emph{is attacked by} the argument $x$'').

For $S\subseteq\XC$ we use the notations $S^{-}$ (resp. $S^{+}$) to indicate
\[
\begin{array}{lcl}
S^{-}&\mbox{ $=$ }&\set{~x\in\XC~:~\exists~y\in S\mbox{ for which }\tuple{x,y}\in\AC}\\
S^{+}&\mbox{ $=$ }&\set{~x\in\XC~:~\exists~y\in S\mbox{ for which }\tuple{y,x}\in\AC}
\end{array}
\]
The arguments in $S^{-}$ (resp. $S^{+}$) are said to attack (resp. be attacked by) $S$.

When for any argument $x \in \XC$, the set $\set{x}^{-}$ of its attackers is finite, the argumentation framework is said to be \emph{finitary}. Formally, an {\sc af}, $\tuple{\XC,\AC}$, is \emph{finitary} iff for each argument $x \in \XC$ $|\set{~y~:~\tuple{y,x}\in\AC}|$ is finite.

A subset $S\subseteq\XC$ is \emph{conflict-free} if no argument in $S$ attacks another argument in $S$, i.e.
$(S\times S)\cap\AC$ is empty. An argument $x$ is said to be \emph{acceptable} with respect to $S\subseteq\XC$ if
for any $y\in\XC$ such that $\tuple{y,x}\in\AC$ there is some $z\in S$ for which $\tuple{z,y}\in\AC$, i.e.
$x$ is acceptable wrt to $S$ if any attacker ($y$) of $x$ is counterattacked by an argument ($z$) of $S$.

The \emph{characteristic function} of an {\sc af} is the mapping $\FC~:~2^{\XC}~\rightarrow~2^{\XC}$ where
\[
\FC(S)~~=~~\set{~x\in\XC~:~\mbox{$x$ is acceptable wrt $S$}}
\]
\end{defn}

Much of the development of {\sc af}s has focused on the study of \emph{argumentation semantics} which can be regarded as refining the informal idea of ``collection of justifiable
arguments in an {\sc af}''. Typically this has been achieved by considering predicates that
such collections must satisfy, i.e. mappings $\sigma~:~2^{\XC}\rightarrow\set{\top,\bot}$ so that
$\EC_{\sigma}(\tuple{\XC,\AC})$ describes the set of subsets of $\XC$ that satisfy the criteria
given by $\sigma$ within the {\sc af} $\tuple{\XC,\AC}$. A review of
the many choices that have been considered for $\sigma$ may be found in Baroni and Giacomin~\cite{bg:2009}. 
\begin{defn}\label{defn:sigma-choice}
Let $\tuple{\XC,\AC}$ be an {\sc af} and $S$ a subset of $\XC$.
\begin{enumerate}
\item[a.]
$S$ is \emph{admissible} (denoted as $S \in \EC_{adm}(\tuple{\XC,\AC})$) if $S$ is conflict-free and every argument in $S$ is acceptable wrt $S$, i.e. $S\subseteq\FC(S)$.
\item[b.]
$S$ is a \emph{complete extension}, (denoted as $S\in\EC_{comp}(\tuple{\XC,\AC})$) if $S$ is conflict-free and $x\in S$ if and only if $x$ is acceptable wrt $S$, i.e. $S=\FC(S)$.
\item[c.]
$S$ is a \emph{preferred extension} (denoted as $S \in \EC_{pr}(\tuple{\XC,\AC})$) if $S$ is a \emph{maximal} (wrt $\subseteq$) admissible set. 

\item[d.]
$S$ is a \emph{stable extension} (denoted as $S \in \EC_{stab}(\tuple{\XC,\AC})$) if $S$ is conflict-free and for any $y\not\in S$, there is some
$x\in S$ that attacks $y$, i.e. $S^{+}~=~\XC\setminus S$.

\item[e.]
$S$ is the \emph{grounded extension} of $\tuple{\XC,\AC}$ (denoted as $S \in \EC_{gr}(\tuple{\XC,\AC})$) if it is the (unique) least fixed point of $\FC$, i.e. $S=\FC(S)$ and there is no $S' \subsetneq S$ such that $S'=\FC(S')$. 
\end{enumerate}

\end{defn}

The existence and uniqueness of the grounded extension is
established in \cite{dung:1995} for all {\sc af}s.

Finally, we recall the various ways in which a given argument may relate to these sets in an {\sc af} $\tuple{\XC,\AC}$.
\begin{defn}\label{defn:acc-status}
Let $x\in\XC$ and $\sigma~:~2^{\XC}\rightarrow\set{\top,\bot}$. The argument $x$ is \emph{credulously accepted}
wrt $\sigma$ if there is some $S$ in $\EC_{\sigma}(\tuple{\XC,\AC})$ such that $x\in S$.
It is said to be \emph{sceptically accepted} wrt $\sigma$ if  every $S$ in $\EC_{\sigma}(\tuple{\XC,\AC})$ satisfies $x\in S$.
\end{defn}
The concepts of credulous and sceptical acceptance, together with the various semantics that have been
put forward, naturally motivate a number of computational problems involving {\sc af}s.
\begin{defn}\label{defn:comp-probs}
Let $\sigma~:~2^{\XC}\rightarrow\set{\top,\bot}$.
\begin{enumerate}
\item[a.]
$\mbox{{\sc ca}}_{\sigma}$ is the decision problem whose instances, $\tuple{\tuple{\XC,\AC},x}$, are accepted
if and only if $x$ is credulously accepted wrt $\sigma$ in $\tuple{\XC,\AC}$.
\item[b.]
$\mbox{{\sc sa}}_{\sigma}$ is the decision problem whose instances, $\tuple{\tuple{\XC,\AC},x}$, are accepted
if and only if $x$ is sceptically accepted wrt $\sigma$ in $\tuple{\XC,\AC}$.
\item[c.]
$\mbox{{\sc ver}}_{\sigma}$ is the decision problem whose instances, $\tuple{\tuple{\XC,\AC},S}$, are
accepted if and only if $S\in\EC_{\sigma}(\tuple{\XC,\AC})$.
\item[d.]
$\mbox{{\sc exist}}_{\sigma}$ is the decision problem whose instances, $\tuple{\XC,\AC}$, are
accepted if and only if $\EC_{\sigma}(\tuple{\XC,\AC})\not=\emptyset$.
\item[e.]
$\mbox{{\sc non-empty}}_{\sigma}$ is the decision problem whose instances, $\tuple{\XC,\AC}$, are
accepted if and only if $\EC_{\sigma}(\tuple{\XC,\AC})\not\in\set{\emptyset,\set{\emptyset}}$.
\end{enumerate}
\end{defn}
As well as the \emph{decision problems} described in Defn.~\ref{defn:comp-probs} there are a range
of \emph{function} (or \emph{construction}) problems. We focus on that
of given $\tuple{\XC,\AC}$ (for which $\EC_{\sigma}(\tuple{\XC,\AC})\not=\emptyset$) identifying all sets
$S\subseteq\XC$ for which $S\in\EC_{\sigma}(\tuple{\XC,\AC})$, denoting this problem $\mbox{{\sc cons}}_{\sigma}$.

For each of the semantics $\sigma$ presented in Defn.~\ref{defn:sigma-choice}, these computational problems
have been studied in depth (within \emph{finite} {\sc af}s)
and their general properties are now well understood. We summarise these results in Fact \ref{fact:alg-complexity} and Table \ref{tab:alg-complexity}.

\begin{fact}\label{fact:alg-complexity}
$\ \ \ $
\begin{enumerate}
\item[a.]
The function $\mbox{{\sc cons}}_{gr}$ is polynomial time computable, hence
all of the cases (a) though (e) of Defn.~\ref{defn:comp-probs} are polynomial time decidable
for the grounded extension~\cite{dung:1995}.
\item[b.]
For $\sigma\in\set{\mbox{\textit{adm, pr, comp, gr}}}$, $\mbox{{\sc exist}}_{\sigma}$ is trivial (i.e. it is always verified as a consequence of well-known results).
\item[c.]
$\mbox{{\sc exist}}_{stab}$ is {\sc np}--complete \cite{DimopTorres:1996} (see also \cite{Fraenkel:1981}).
\item[d.]
For $\sigma\in\set{\mbox{\textit{adm, stab, comp}}}$, $\mbox{{\sc ver}}_{\sigma}$ is decidable in polynomial-time,
however $\mbox{{\sc ver}}_{pr}$ is co{\sc np}--complete~\cite{DimopTorres:1996}.
\item[e.]
For $\sigma\in\set{\mbox{\textit{adm, pr, stab, comp}}}$, $\mbox{{\sc ca}}_{\sigma}$ is {\sc np}--complete~\cite{DimopTorres:1996}.
\item[f.]
$\mbox{{\sc sa}}_{\sigma}$ is trivial for $\sigma=\mbox{\textit{adm}}$, polynomial for $\sigma=\mbox{\textit{comp}}$, co{\sc np}--complete 
for $\sigma=stab$\footnote{Note, however, the \emph{caveat} raised in \cite{dw:2009}.}, and 
$\Pi_{2}^{p}$--complete for $\sigma=pr$~\cite{dunnebc:2002}.
\item[g.]
$\mbox{{\sc non-empty}}_{\sigma}$ is {\sc np}--complete for $\sigma\in\set{\mbox{\textit{adm,~pr, comp, stab}}}$~\cite{DimopTorres:1996}.
\end{enumerate}
\end{fact}

\begin{table}[htb]
  \centering
  \begin{tabular}{| l | c | c | c | c | c |}
  \hline
                                    &   $\sigma=$\textit{adm} & $\sigma=$\textit{pr} & $\sigma=$\textit{comp} & $\sigma=$\textit{stab} & $\sigma=$\textit{gr} \\
  \hline
  $\mbox{\sc exist}_{\sigma}$         &  trivial        & trivial      & trivial        & {\sc np}--complete  & trivial     \\
  \hline
  $\mbox{\sc ver}_{\sigma}$           &  polynomial     & co{\sc np}--complete & polynomial & polynomial      & polynomial      \\
  \hline
  $\mbox{\sc ca}_{\sigma}$            & {\sc np}--complete & {\sc np}--complete & {\sc np}--complete & {\sc np}--complete & polynomial      \\
  \hline
  $\mbox{\sc sa}_{\sigma}$            & trivial      & $\Pi_{2}^{p}$  & polynomial        & co{\sc np}--complete & polynomial         \\
  \hline
  $\mbox{\sc non-empty}_{\sigma}$     & {\sc np}--complete & {\sc np}--complete & {\sc np}--complete & {\sc np}--complete & polynomial      \\
  \hline
  \end{tabular}
  \caption{Computational problems in finite {\sc af}s.}
  \label{tab:alg-complexity}
\end{table}

We emphasise that the classifications in Fact~\ref{fact:alg-complexity} are with respect
to finite {\sc af}s. For a more detailed summary of complexity and algorithms within
{\sc af} semantics we refer the reader to the overview of Dunne and Wooldridge~\cite{dw:2009}; complexity-theoretic
treatments of both novel semantics and developments of Dung's original proposals may be found in, among others,
\cite{Baroni-et-al:2011,dw:2010,Dunne:2009b,DHMcBPW:2011,ddw:2011}.

\subsection{Formal languages}\label{subsection:back-flt}

As to the required background on formal languages, which will be heavily used in the paper, we assume that the reader is familiar with the standard concepts and basic results in the field (to make the paper self-contained the necessary ones are provided in \ref{subsection:flt}). We recall only the basic definitions in this section, in order to introduce the reader to the notation used in the sequel of the paper.

\begin{defn}\label{defn:alpha-word}
An \emph{alphabet} is a \emph{finite} set of symbols. For an arbitrary alphabet, the notation
$\Sigma~=~\set{\sigma_1,\sigma_2,\ldots,\sigma_k}$ will be used. A \emph{word}, $w$,
over an alphabet $\Sigma$ is a \emph{finite sequence}, $w=w_{i_1}w_{i_2}\cdots w_{i_r}$ of symbols from $\Sigma$. The
set of all possible words is denoted as $\Sigma^{*}$. The \emph{length}, $|w|$, of $w\in\Sigma^{*}$
is the total number of symbols occuring in its definition. The word of length $0$ in $\Sigma^{*}$
is called the \emph{empty word} and is denoted as $\varepsilon$.

For $u=u_{i_1}\cdots u_{i_r}$ and $v=v_{j_1}\cdots v_{j_s}$ words in $\Sigma^{*}$ the word
$w\in\Sigma^{*}$ formed by \emph{concatenating} $u$ with $v$ (denoted $u\cdot v$) is
the word $u_{i_1}\ldots u_{i_r}~v_{j_1}\ldots v_{j_s}$ whose length is $|u|+|v|=r+s$. For any $u\in\Sigma^{*}$,
$u\cdot\varepsilon=\varepsilon\cdot u=u$, i.e. $\varepsilon$ is an identity element in $\Sigma^{*}$
with respect to the operation $\cdot$ of concatenation.
\end{defn}
\begin{defn}\label{defn:languages-ops}
A \emph{language}, $L$, over an alphabet $\Sigma$, is a subset of $\Sigma^{*}$. For languages $L_1$ and $L_2$
we define languages $L_1\cup L_2$, $L_1\cap L_2$, and $L_1\setminus L_2$ in the obvious way so that the operations $\set{\cup,\cap,\setminus}$ are the standard set-theoretic ones.

In addition, specific to languages,
\[
\begin{array}{lcl}
L_1\cdot L_2&\mbox{ $=$ }&\set{u\cdot v~:~u\in L_1,~v\in L_2}\\
\overline{L}&\mbox{ $=$ }&\set{u~:~u\in\Sigma^{*},~u\not\in L}\\
L^{*}&\mbox{ $=$ }&\cup_{k=0}^{\infty}~\set{~w~:~w=u_1\cdot u_2\cdots u_k,~u_i\in L}\\
L_1/L_2&\mbox{ $=$ }&\set{~u~:~\exists~v\in L_2\mbox{ s.t. }u\cdot v\in L_1}\\
rev(L)&\mbox{ $=$ }&\set{~\sigma_1\sigma_2\cdots\sigma_{m-1}\sigma_m~:~\sigma_m\sigma_{m-1}\cdots\sigma_2\sigma_1\in L}
\end{array}
\]
\end{defn}

The language $L^{*}$ is sometimes referred to as the \emph{Kleene closure} (or $*$-closure) of $L$, while
$L_1/L_2$ is called the \emph{quotient} of $L_1$ wrt $L_2$.\footnote{Some authors distinguish
so-called \emph{left} and \emph{right} quotients of $L_1$ wrt $L_2$, the latter being
$L_1/L_2$ (as given in the definition), the former $\set{~u~:~\exists v\in L_1\mbox{ s.t. }v\cdot u\in L_2}$. We use
only the notion of (right) quotient.}

\begin{defn}\label{defn:rl}
A language $L\subseteq\Sigma^{*}$ is a \emph{regular language} if $L$ satisfies any of the following requirements:
\begin{itemize}
\item[R1.]
$L=\emptyset$ or $L=\set{\varepsilon}$ or $L=\set{\sigma}$ for any $\sigma\in\Sigma$.
\item[R2.]
$L~=~L_1\cup L_2$ where $L_1$ and $L_2$ are regular languages.
\item[R3.]
$L~=~L_1\cdot L_2$ where $L_1$ and $L_2$ are regular languages.
\item[R4.]
$L=(L_1)^{*}$ where $L_1$ is a regular language.
\end{itemize}
\end{defn}

\begin{defn}\label{defn:fg}
A \emph{formal grammar} is defined via a 4-tuple, $\tuple{\Sigma,V,P,S}$ where $\Sigma=\set{\sigma_1,\sigma_2,\ldots,\sigma_n}$ is
a \emph{finite alphabet} of \emph{terminal symbols}; $V=\set{V_1,\ldots,V_m}$ a finite set of \emph{variable symbols}, $P$ is a finite
set of \emph{production rules}, $\set{p_1,p_2,\ldots,p_r}$ of the form $p_i~:~\alpha_i\rightarrow\beta_i$ where
$\alpha_i\in(V\cup\Sigma)^{*}\setminus\Sigma^{*}$ and
$\beta_i\in(V\cup \Sigma)^{*}$ and $S\in V$ is the \emph{start symbol}.
The language generated (see \ref{subsection:flt})) by a grammar $G$ is denoted as $L(G)$.
\end{defn}

\begin{defn}\label{defn:dfa}
A \emph{deterministic finite automaton} ({\sc dfa}) is defined via a $5$-tuple, $M=\tuple{\Sigma,Q,q_0,F,\delta}$ where
$\Sigma=\set{\sigma_1,\ldots,\sigma_k}$ is a finite set of input symbols, $Q=\set{q_0,q_1,\ldots,q_m}$ a finite set of \emph{states};
$q_0\in Q$ the \emph{initial} state; $F\subseteq Q$ the set of \emph{accepting} states; and
$\delta~:~Q\times\Sigma~\rightarrow~Q$ the \emph{state transition} function.
A word $w=w_n w_{n-1}\ldots w_1\in \Sigma^{*}$ is accepted by the {\sc dfa} $\tuple{\Sigma,Q,q_0,F,\delta}$ if
the sequence of states $q_{i_1}q_{i_2}\ldots q_{i_n}$ consistent with the
state transition function $\delta$ which processes every symbol in $w$, i.e. satisfying
$q_{i_1}=\delta(q_0,w_1)$ and $q_{i_j}=\delta(q_{i_{j-1}},w_j)$ for each $2\leq j\leq n$, has $q_{i_n}\in F$.
For a {\sc dfa}, $M=\tuple{\Sigma,Q,q_0,F,\delta}$, $L(M)$ is the subset of $\Sigma^{*}$ accepted by $M$.
\end{defn}


%
%
\section{Formalism requirements and weaknesses of naive representations}\label{section:naive-rep}

Given the goal of investigating novel approaches to deal with argumentation frameworks with a countably infinite set of arguments, we need to establish some basic criteria to evaluate the approaches themselves.

A first basic criterion is \emph{expressiveness}, namely the ability to encompass the description of a sufficiently large variety of infinite argumentation frameworks so as to cover those cases which are meaningful from a theoretical or practical perspective. These include in particular the cases of infinite argumentation frameworks already considered in the literature.

A second criterion is \emph{tractability}: the use of the formalism should not raise intractable computational problems making it impractical.
In particular, we have to notice the arousal of a problem not occurring in the finite case: given that an argumentation framework $\tuple{\XC,\AC}$ involving infinite sets can only be given through a finite encoding $\eta(\tuple{\XC,\AC})$, it must be validated that an encoding $\eta(\tuple{\XC,\AC})$ is indeed a valid description of \emph{some} {\sc af}.

Further, computational requirements related to the basic decision problems listed in Definition \ref{defn:comp-probs} have to be taken into account.
A third criterion is therefore \emph{closure} wrt set-theoretical operations, as they are involved in the definition and/or characterization of the fundamental properties in argumentation semantics and hence in the relevant decision procedures. 
To exemplify, testing whether a set of arguments is conflict-free corresponds to test whether the intersection between this set and the set of its attackers is empty.
Hence, given the specification of two infinite sets of arguments in a formalism, the specification of their intersection should be captured (and, hopefully, be easily constructable) within the same formalism.

In the view of satisfying the above requirements, a standard approach to the problem of representing an infinite collection of objects via a finite specification is to exploit formal grammars and their associated machine models\footnote{Considering other possible choices of formal tools for the specification of infinite structures is beyond the scope of the present paper and is left for future work.}.

In this context, we now consider and criticize a rather straightforward approach one might adopt in describing $\tuple{\XC,\AC}$ where
the supporting set of arguments is an infinite, but enumerable, set.
The idea, introduced in Definition \ref{defn:naive} consists in describing the attack relation with a language referring to indexes in the argument enumeration.

\begin{defn}\label{defn:naive}
Let $\XC=\set{x_1,x_2,\ldots,x_n,\ldots}$ be a countably infinite set of atomic arguments. A subset $\AC\subseteq\XC\times\XC$
is \emph{naively encoded} if described as the language $L_{\AC}$ over the two symbol alphabet $\set{0,1}$ for which
\[
L_{\AC}~~=~~\set{0^i\cdot1\cdot0^j~:~\tuple{x_i,x_j}\in\AC}
\]
\end{defn}
Thus the naive encoding of a set of attacks uses a unary\footnote{We could, of course, use an
arbitrary number base, however, to do so adds nothing in the way of
expressive power and can, in fact, reduce this considerably.} form to describe the (indices) of the source and destination
arguments involved in the attack with the symbol $1$ used to separate these two components (in the absence of any attack $L_{\AC}~=~\emptyset$).

For naive encodings one can consider formal grammars and their associated languages as a means of presenting a finite
specification of $\tuple{\XC,\AC}$, i.e. as a grammar $G$ over alphabet $\set{0,1}$ for which $L(G)=L_{\AC}$.
We show (proofs are given in \ref{proofs-naive}) that different choices for the family of grammars $G$ belongs to lead invariably to the violation of (at least) one of the three criteria above, making this approach unsuitable in spite of its apparent simplicity.

We start with an unsurprising property of naive encodings.
\begin{restatable}{propn}{noformalgrammar}
\label{prop:unrecognisable}
For $\XC$ as introduced in Definition \ref{defn:naive}, there are choices of $\AC\subseteq\XC\times\XC$ such that there is no formal grammar, $G$
with $L(G)=L_{\AC}\subseteq\set{0^i\cdot1\cdot0^j~:~i,~j\geq1}$.
\end{restatable}

The issue of infinite attack structures which cannot be described within naive encodings may, justifiably,
be seen as a purely technical limitation as far as the families of {\sc af}s so affected are unlikely to feature in applications (see the proof of Proposition \ref{prop:unrecognisable}).
It turns out however that, for unrestricted grammars, the criterion of tractability is not satisfied since the problem of determining if a naive encoding does indeed describe \emph{some} {\sc af} is not semi-decidable.

\begin{restatable}{propn}{unrestrictednondecidable}
\label{propn:valid-tm}
Given an arbitrary (i.e. unrestricted) grammar $G$ over the alphabet $\set{0,1}$ the problem of determining
if $L(G)\subseteq\set{0^i\cdot1\cdot0^j~:~i,~j\geq 1}$ is not semi-decidable, i.e. there is no TM program
which given (a description of) $G$ as input halts and accepts precisely those $G$ for which 
$L(G)\subseteq\set{0^i\cdot1\cdot0^j~:~i,~j\geq 1}$.
\end{restatable}

In fact, results analogous to Proposition~\ref{propn:valid-tm} continue to hold even if we use the less expressive class of context-sensitive grammars. 
\begin{restatable}{propn}{contextsensitivenonsemidecidable}
\label{propn:valid-csl}
Given an arbitrary \emph{context-sensitive} grammar, $G$, over the alphabet $\set{0,1}$ the problem of determining
if $L(G)\subseteq\set{0^i\cdot1\cdot0^j~:~i,~j\geq 1}$ is not semi-decidable.
\end{restatable}

The problem evidenced in Propositions \ref{propn:valid-tm} and \ref{propn:valid-csl} does not hold when considering context-free languages.

\begin{restatable}{propn}{contextfreedecidable}
\label{propn:valid-cfg}
Given an arbitrary \emph{context-free} grammar ({\sc cfg}), $G$, over the alphabet $\set{0,1}$ the problem of determining
if $L(G)\subseteq\set{0^i\cdot1\cdot0^j~:~i,~j\geq 1}$ is decidable.
\end{restatable}

Context-free languages, however, do not satisfy the property of closure:
it is well-known that they are not closed under intersection, complement and set difference.

Turning to regular languages, as with the context-free case one can decide
if a given {\sc dfa} accepts the naive encoding of some $\tuple{\XC,\AC}$.

\begin{restatable}{propn}{dfadecidable}
\label{prop-dfa-decidable}
Given $M=\tuple{Q,\set{0,1},q_0,F,\delta }$ a {\sc dfa} over the alphabet $\set{0,1}$ there is a polynomial (in $|Q|$) algorithm
that decides $L(M)\subseteq\set{0^i\cdot1\cdot0^j~:~i,~j\geq 1}$.
\end{restatable}

Moreover, regular languages are fully satisfactory as far as closure properties are concerned (see \ref{subsection:flt} Fact \ref{fact:properties}).
Unfortunately, however, they feature very limited expressive power in terms of describing naive encodings.

\begin{restatable}{propn}{limiteddfa}
\label{prop-limited-dfa}
Let $L$ be any subset of $\set{0^i\cdot1\cdot0^j~:~i,~j\geq 1}$ with the following property:
there are infinitely many values of $k$ such that $\set{0^k\cdot1\cdot0^m~:~m\geq1}\cap L\not=\emptyset$
\emph{and} for all $0^n\cdot1\cdot0^m\in L$, $n\leq m$.
Then $L$ is \emph{not} a regular language.
\end{restatable}

Notice that one consequence of Proposition~\ref{prop-limited-dfa} is that the naive encoding of the infinite {\sc af} whose only
attacks are \emph{self-attacks}, i.e. $\tuple{p,p}$ for all $p\in\XC$ fails to be a regular language.\footnote{This language,
$\set{0^m\cdot1\cdot0^m~:~m\geq1}$ is, however, context-free.}

In summary, although naive representations have an appealing structural simplicity, if adopted one has to contend with
issues of undecidability (for the most expressive grammar classes), lack of closure (for context-free languages) or limited expressiveness
(for regular languages). 

%
%
\section{A Generic Regular Expression Formalism and its Properties}\label{section:sec-generic-representation}
The issues identified with so-called naive representations in the preceding section largely stem from
the following fact: given that encodings of \emph{arguments}, $p_i\in\XC$ are effectively achieved for free -- that is,
for all natural numbers $k$, $p_k\in\XC$ and no further analysis is needed -- the task of describing
$\tuple{\XC,\AC}$ comes down to describing the (infinite) set $\AC$. In assuming that $p_k\in\XC$ for any $k$, however,
this severely limits the extent to which $\AC$ can be described in a computationally useful manner.

In this section we present an alternative method for describing infinite {\sc af}s, $\tuple{\XC,\AC}$. The basic idea is
that, rather than assuming $\XC$ is understood simply as $\set{p_1,p_2,\ldots,p_n,\ldots}$, we consider
arguments in $\XC$ to be specified so that there is some structural aspect linking them. In this way
we can then present very general specifications of the attack structure that are conditioned solely in
terms of the \emph{specific} arguments in $\XC$. 

In a nutshell, the proposal consists of two basic elements: a description of the (infinite) set of arguments through an appropriate \emph{argument encoding} relying on some finite automaton, and a description of the attack relations linking arguments together through an \emph{attack expression}.
More precisely, the attack expression specifies a mapping between regular expressions describing sets of arguments, with the intended meaning that the set $S$ is attacked by the elements of the set obtained from $S$ through the mapping.
The combination of the automaton describing the set of arguments and of the attack expression will be called the \emph{{\sc af} specification}.

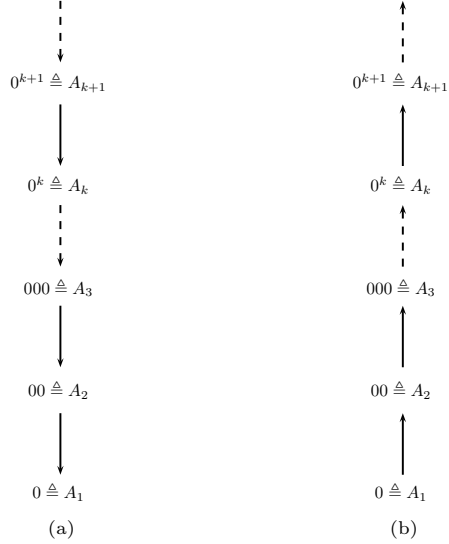
\begin{figure}[thb]
\begin{center}

\subfigure[]{\label{fig:AFL}
\scalebox{0.65} 
{
\begin{pspicture}(0,-5.156875)(5.1471877,5.136875)
\usefont{T1}{ptm}{m}{n}
\rput(2.5154688,3.446875){$0^{k+1} \triangleq A_{k+1}$}
\usefont{T1}{ptm}{m}{n}
\rput(2.5154688,1.351875){$0^k \triangleq A_k$}
\usefont{T1}{ptm}{m}{n}
\rput(2.5154688,-0.743125){$000 \triangleq A_3$}
\usefont{T1}{ptm}{m}{n}
\rput(2.5154688,-2.838125){$00 \triangleq A_2$}
\usefont{T1}{ptm}{m}{n}
\rput(2.5154688,-4.933125){$0 \triangleq A_1$}
\psline[linewidth=0.04cm,arrowsize=0.05291667cm 2.0,arrowlength=1.4,arrowinset=0.4]{->}(2.5535936,2.996875)(2.5535936,1.736875)
\psline[linewidth=0.04cm,linestyle=dashed,dash=0.16cm 0.16cm,arrowsize=0.05291667cm 2.0,arrowlength=1.4,arrowinset=0.4]{->}(2.5535936,0.936875)(2.5535936,-0.323125)
\psline[linewidth=0.04cm,arrowsize=0.05291667cm 2.0,arrowlength=1.4,arrowinset=0.4]{->}(2.5535936,-1.123125)(2.5535936,-2.383125)
\psline[linewidth=0.04cm,arrowsize=0.05291667cm 2.0,arrowlength=1.4,arrowinset=0.4]{->}(2.5535936,-3.323125)(2.5535936,-4.583125)
\psline[linewidth=0.04cm,linestyle=dashed,dash=0.16cm 0.16cm,arrowsize=0.05291667cm 2.0,arrowlength=1.4,arrowinset=0.4]{->}(2.5535936,5.116875)(2.5535936,3.856875)
\end{pspicture} 
}}
\hspace{2em}
\subfigure[]{\label{fig:AFR}
\scalebox{0.65} 
{
\begin{pspicture}(0,-5.156875)(5.1471877,5.136875)
\usefont{T1}{ppl}{m}{n}
\rput(2.5154688,3.446875){$0^{k+1} \triangleq A_{k+1}$}
\usefont{T1}{ppl}{m}{n}
\rput(2.5154688,1.351875){$0^k \triangleq A_k$}
\usefont{T1}{ppl}{m}{n}
\rput(2.5154688,-0.743125){$000 \triangleq A_3$}
\usefont{T1}{ppl}{m}{n}
\rput(2.5154688,-2.838125){$00 \triangleq A_2$}
\usefont{T1}{ppl}{m}{n}
\rput(2.5154688,-4.933125){$0 \triangleq A_1$}
\psline[linewidth=0.04cm,arrowsize=0.05291667cm 2.0,arrowlength=1.4,arrowinset=0.4]{<-}(2.5535936,2.996875)(2.5535936,1.736875)
\psline[linewidth=0.04cm,linestyle=dashed,dash=0.16cm 0.16cm,arrowsize=0.05291667cm 2.0,arrowlength=1.4,arrowinset=0.4]{<-}(2.5535936,0.936875)(2.5535936,-0.323125)
\psline[linewidth=0.04cm,arrowsize=0.05291667cm 2.0,arrowlength=1.4,arrowinset=0.4]{<-}(2.5535936,-1.123125)(2.5535936,-2.383125)
\psline[linewidth=0.04cm,arrowsize=0.05291667cm 2.0,arrowlength=1.4,arrowinset=0.4]{<-}(2.5535936,-3.323125)(2.5535936,-4.583125)
\psline[linewidth=0.04cm,linestyle=dashed,dash=0.16cm 0.16cm,arrowsize=0.05291667cm 2.0,arrowlength=1.4,arrowinset=0.4]{<-}(2.5535936,5.116875)(2.5535936,3.856875)
\end{pspicture} 
}}

\caption{Two infinite attack chains: $AF_D$ \subref{fig:AFL}, and $AF_U$ \subref{fig:AFR}.}
\label{fig:AFD-AFU}

\end{center}
\end{figure}

\begin{example}
\label{ex:chains}
To illustrate our approach we will use some \virg{simple} infinite structures, which can be regarded as basic patterns, possibly to be reused in the context of more articulated infinite argumentation frameworks.
A first example of basic structure (related to endless debates, or \virg{chiken and egg} style dilemmas, or the kicker and goalkeeper example) is an infinite sequence of arguments linked by the attack relation. We note that this kind of structure admits two different instantiations: one where the sequence \virg{starts} with an attacked argument (which corresponds to the examples mentioned above) and a dual one where the sequence \virg{starts} with an attacking argument.
They correspond respectively to the argumentation frameworks $AF_D$ and $AF_U$ (see Figure \ref{fig:AFD-AFU}) defined as follows: $AF_D = \tuple{\XC,\AC_D}$ with $\XC=\set{A_1, A_2, \ldots, A_n , \ldots}$ and $\AC_D=\set{\tuple{A_{i+1}, A_i}~:~i \geq 1}$ and $AF_U = \tuple{\XC,\AC_U}$ with $\XC$ as above and $\AC_U=\set{\tuple{A_i,A_{i+1}}~:~i \geq 1}$.
Another example of basic structure, related to temporal projection, is a couple of \virg{parallel} sequences of arguments where corresponding arguments in the sequences mutually attack each other. This kind of structure may correspond to conflicting information acquired at the same time for the same entity (e.g. a physical quantity) from two different and equally reliable sources $A$ and $B$ (e.g. two different experiments or measurements).
The conflict between the initial arguments corresponding to the readings from the two sources is then projected over time.
This can be represented by the argumentation framework $AF_M = \tuple{\XC_M,\AC_M}$ with $\XC_M=\set{A_1, A_2, \ldots, A_n , \ldots} \cup \set{B_1, B_2, \ldots, B_n , \ldots}$ and $\AC_M=\set{\tuple{A_i, B_i}~:~i \geq 1} \cup \set{\tuple{B_i, A_i}~:~i \geq 1}$ (see Figure \ref{fig:AF-M}).
To provide a slightly more complicated structure, we will consider also a variant of this temporal projection situation, with a third source $C$ which is in agreement with the source $A$ and is considered more reliable than $A$ and $B$.
This can be represented by the argumentation framework $AF_R = \tuple{\XC_R,\AC_R}$ with $\XC_R=\set{A_1, A_2, \ldots, A_n , \ldots} \cup \set{B_1, B_2, \ldots, B_n , \ldots} \cup \set{C_1, C_2, \ldots, C_n , \ldots}$ and $\AC_R=\set{\tuple{A_i, B_i}~:~i \geq 1} \cup \set{\tuple{B_i, A_i}~:~i \geq 1} \cup \set{\tuple{C_i, B_i}~:~i \geq 1}$ (see Figure \ref{fig:AF-R}).
\end{example}

\begin{figure}[!htb]
\begin{center}

\subfigure[]{\label{fig:AF-M}
\scalebox{0.65} 
{
\begin{pspicture}(0,-3.87)(7.0487266,3.87)
\usefont{T1}{ppl}{m}{n}
\rput(2.3454688,2.799742){$0 \triangleq A_1$}
\usefont{T1}{ppl}{m}{n}
\rput(2.2754688,0.62364745){$000 \triangleq A_2$}
\usefont{T1}{ppl}{m}{n}
\rput(4.7494073,0.62364745){$0000 \triangleq B_2$}
\usefont{T1}{ppl}{m}{n}
\rput(2.2254686,-1.9563526){$0\cdots0 \triangleq A_k$}
\usefont{T1}{ppl}{m}{n}
\rput(4.6170077,-1.9563526){$00\cdots0 \triangleq B_k$}
\psbezier[linewidth=0.04,arrowsize=0.05291667cm 2.0,arrowlength=1.4,arrowinset=0.4]{->}(2.584454,3.0726087)(2.6012225,3.85)(4.4709167,3.737121)(4.504454,3.1275759)
\psbezier[linewidth=0.04,arrowsize=0.05291667cm 2.0,arrowlength=1.4,arrowinset=0.4]{<-}(2.554454,2.4225857)(2.5803928,1.73)(4.5085154,1.7397369)(4.534454,2.47)
\psline[linewidth=0.04cm,linestyle=dashed,dash=0.16cm 0.16cm](3.454454,-0.45)(3.454454,-0.87)
\psline[linewidth=0.04cm,linestyle=dashed,dash=0.16cm 0.16cm](3.454454,-2.97)(3.454454,-3.85)
\psbezier[linewidth=0.04,arrowsize=0.05291667cm 2.0,arrowlength=1.4,arrowinset=0.4]{->}(2.624454,0.9326087)(2.6412225,1.71)(4.5109167,1.5971212)(4.544454,0.98757577)
\psbezier[linewidth=0.04,arrowsize=0.05291667cm 2.0,arrowlength=1.4,arrowinset=0.4]{<-}(2.594454,0.2825858)(2.6203928,-0.41)(4.5485153,-0.40026313)(4.574454,0.33)
\psbezier[linewidth=0.04,arrowsize=0.05291667cm 2.0,arrowlength=1.4,arrowinset=0.4]{->}(2.604454,-1.6273913)(2.6212225,-0.85)(4.4909167,-0.96287876)(4.524454,-1.5724243)
\psbezier[linewidth=0.04,arrowsize=0.05291667cm 2.0,arrowlength=1.4,arrowinset=0.4]{<-}(2.574454,-2.2774143)(2.6003928,-2.97)(4.5285153,-2.960263)(4.554454,-2.23)
\usefont{T1}{ppl}{m}{n}
\rput(4.489407,2.7636476){$00 \triangleq B_1$}
\end{pspicture} 
}}
\hspace{0em}
\subfigure[]{\label{fig:AF-R}
\scalebox{0.65} 
{
\begin{pspicture}(0,-3.88)(10.428726,3.88)
\usefont{T1}{ppl}{m}{n}
\rput(2.3454688,2.809742){$0 \triangleq A_1$}
\usefont{T1}{ppl}{m}{n}
\rput(4.901483,2.829742){$00 \triangleq B_1$}
\usefont{T1}{ppl}{m}{n}
\rput(7.691483,2.769742){$000 \triangleq C_1$}
\usefont{T1}{ppl}{m}{n}
\rput(2.3854687,0.63364744){$0000 \triangleq A_2$}
\usefont{T1}{ppl}{m}{n}
\rput(4.7994075,0.63364744){$00000 \triangleq B_2$}
\usefont{T1}{ppl}{m}{n}
\rput(7.7294073,0.57364744){$000000 \triangleq C_2$}
\usefont{T1}{ppl}{m}{n}
\rput(2.2254686,-1.9463526){$0\cdots0 \triangleq A_k$}
\usefont{T1}{ppl}{m}{n}
\rput(4.6170077,-1.9463526){$00\cdots0 \triangleq B_k$}
\usefont{T1}{ppl}{m}{n}
\rput(7.8670077,-1.9863526){$000\cdots0 \triangleq C_k$}
\psbezier[linewidth=0.04,arrowsize=0.05291667cm 2.0,arrowlength=1.4,arrowinset=0.4]{->}(2.584454,3.0826087)(2.6012225,3.86)(4.4709167,3.747121)(4.504454,3.1375759)
\psbezier[linewidth=0.04,arrowsize=0.05291667cm 2.0,arrowlength=1.4,arrowinset=0.4]{<-}(2.554454,2.4325857)(2.5803928,1.74)(4.5085154,1.7497369)(4.534454,2.48)
\psline[linewidth=0.04cm,linestyle=dashed,dash=0.16cm 0.16cm](5.674454,0.04)(5.674454,-1.44)
\psline[linewidth=0.04cm,linestyle=dashed,dash=0.16cm 0.16cm](5.674454,-2.38)(5.674454,-3.86)
\psline[linewidth=0.04cm,arrowsize=0.05291667cm 2.0,arrowlength=1.4,arrowinset=0.4]{->}(6.584454,2.78)(5.774454,2.78)
\psline[linewidth=0.04cm,arrowsize=0.05291667cm 2.0,arrowlength=1.4,arrowinset=0.4]{->}(6.624454,0.60770994)(5.814454,0.60770994)
\psline[linewidth=0.04cm,arrowsize=0.05291667cm 2.0,arrowlength=1.4,arrowinset=0.4]{->}(6.604454,-1.9722902)(5.794454,-1.9722902)
\psbezier[linewidth=0.04,arrowsize=0.05291667cm 2.0,arrowlength=1.4,arrowinset=0.4]{->}(2.624454,0.9426087)(2.6412225,1.72)(4.5109167,1.6071212)(4.544454,0.99757576)
\psbezier[linewidth=0.04,arrowsize=0.05291667cm 2.0,arrowlength=1.4,arrowinset=0.4]{<-}(2.594454,0.2925858)(2.6203928,-0.4)(4.5485153,-0.39026314)(4.574454,0.34)
\psbezier[linewidth=0.04,arrowsize=0.05291667cm 2.0,arrowlength=1.4,arrowinset=0.4]{->}(2.604454,-1.6173913)(2.6212225,-0.84)(4.4909167,-0.9528788)(4.524454,-1.5624242)
\psbezier[linewidth=0.04,arrowsize=0.05291667cm 2.0,arrowlength=1.4,arrowinset=0.4]{<-}(2.574454,-2.267414)(2.6003928,-2.96)(4.5285153,-2.950263)(4.554454,-2.22)
\end{pspicture} 
}}

\caption{Two infinite argumentation frameworks: $AF_M$ \subref{fig:AF-M}, and $AF_R$ \subref{fig:AF-R}.}

\end{center}
\end{figure}
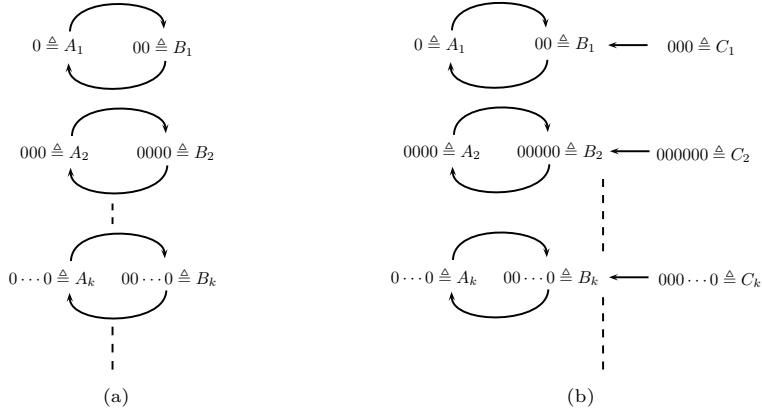

\subsection{Argument encoding}

Central to our formalism is the notion of ``argument encoding'' as a given set of words over some base set of symbols.
\begin{defn}\label{defn:arg-encoding}
Let $\Sigma~=~\set{\sigma_1,\sigma_2,\ldots,\sigma_k}$ be an alphabet. An \emph{argument encoding}
over $\Sigma$ is any regular language $\XC\subseteq\Sigma^{*}\setminus\set{\varepsilon}$.
\end{defn}
Thus a (possibly infinite) set of arguments, $\XC$, corresponds to some regular language over a finite alphabet $\Sigma$.

\begin{examplecont}[\ref{ex:chains}]
The notion of argument encoding is very general and, in fact, leaves completely open the choice of the alphabet and of the regular language to be adopted to represent a given infinite structure.
Considering the frameworks $AF_D$ and $AF_U$, featuring a simple \virg{linear} structure, the straightforward choice we will follow uses an alphabet consisting of a unique symbol, actually $\Sigma=\set{0}$, and 
the regular language $0 \cdot 0^{*}$, by adopting the correspondence $0^{i} \triangleq A_i$.
For frameworks with more articulated structures like $AF_M$ and $AF_R$, one might consider using an alphabet with a symbol for each class of arguments, e.g. 
$\Sigma_M=\set{0,1}$ and $\Sigma_R=\set{0,1,2}$.
Accordingly the regular languages to encode $\XC_M$ and $\XC_R$ would be $0 \cdot 0^{*} \cup 1 \cdot 1^{*}$ and $0 \cdot 0^{*} \cup 1 \cdot 1^{*} \cup 2 \cdot 2^{*}$ respectively, with the correspondences $0^{i} \triangleq A_i$, $1^{i} \triangleq B_i$, $2^{i} \triangleq C_i$.
As we will see later, however, a different approach will be adopted since it is more suitable for the purpose of representation of the attack relation.
In fact, we will use again the alphabet $\Sigma=\set{0}$, and 
the regular language $0 \cdot 0^{*}$ for all arguments, putting the different classes of arguments in correspondence with distinct sublanguages of $0 \cdot 0^{*}$.
In fact for $AF_M$ we will adopt the correspondence  $0 \cdot (00)^i \triangleq A_{i+1}$ and $00 \cdot (00)^i \triangleq B_{i+1}$ for $i \geq 0$, i.e. $0 \triangleq A_1$, $00 \triangleq B_1$, $000 \triangleq A_2$, $0000 \triangleq B_2$ and so on.
Similarly for $AF_R$ we will adopt the correspondence  $0 \cdot (000)^i \triangleq A_{i+1}$, $00 \cdot (000)^i \triangleq B_{i+1}$ and $000 \cdot (000)^i \triangleq C_{i+1}$for $i \geq 0$.
\end{examplecont}

\subsection{Attack expression}

The increased expressivity and computational effectiveness of our approach (in comparison to naive encodings), derives from the
mechanism used to describe sets of attacks. We seek to develop formalisms by which the set
of arguments (in $\XC$) that ``attack'' in the standard sense of \cite{dung:1995} some specified subset $S\subseteq\XC$
may be presented, i.e. for defining ``suitable'' functions $\mu~:~2^{\XC}\rightarrow2^{\XC}$, such that, for a given set of arguments $S$, $\mu(S)$ specifies the set of arguments attacking $S$, i.e. $\mu(S)= S^{-}$. By ``suitable''
we recognise that there are certain natural conditions that such functions ought to respect. 

\begin{defn}\label{defn:reasonable-attack}
A mapping $\mu~:~2^{\Sigma^{*}}\rightarrow2^{\Sigma^{*}}$ is \emph{reasonable} if
\begin{enumerate}
\item[R1.]
$\forall S\subseteq\Sigma^{*}$, $\mu(S)~=~\bigcup_{u\in S}~\mu(\set{u})$. (Additivity)
\item[R2.]
$\forall S\subseteq\Sigma^{*}$ for which $S$ is a regular language, $\mu(S)$ is a regular language. (Closure)
\end{enumerate}
In addition, we say that a reasonable mapping is \emph{invertible} if the function
$\nu~:~2^{\Sigma^{*}}\rightarrow2^{\Sigma^{*}}$  defined via 
\[
\nu(S)~~=~~\set{u\in\Sigma^{*}~:~\exists~v\in S~\mbox{ s.t. }v\in\mu(\set{u})}
\]
is in turn a reasonable mapping.
It is easy to see that $\nu$ satisfies property R1 by definition.
\end{defn}

We observe that all of these conditions hold for the finite instantiations of $\AC$ in Dung's {\sc af}s. For the
development we consider to infinite $\XC$, the additivity restriction states that 
attacks on $S$ must be associated with \emph{individual} arguments in $S$ and not with $S$ as a whole, i.e. we can construct the
set of attacks on $S$ simply by considering attacks on the members of $S$ in turn.\footnote{We note, however,
that settings with attack relations not respecting additivity have been examined in the context
of finite frameworks, e.g. Nielsen and Parsons~\cite{NielsenP06a} and Bochmann~\cite{Bochmann:2003}.} It may be
noted that additivity implies that the mapping is also monotonic, i.e. if $R\subseteq S$ then $\mu(R)\subseteq\mu(S)$. Thus
additive mappings reflect the natural condition that attacks on a set of arguments, $S$,
cannot be eliminated simply by adding more arguments to $S$: notice that by disallowing \emph{explicit removal} of an attack
on a set (in the definition of $\mu$) we require \emph{defences} to attacks to be made by
direct counterattacks. In this way we preserve the concept of ``$u$ is acceptable to $S$'' by identifying any
$v\in S$ such that $v\in\mu(\set{x})$ for each $x\in\mu(\set{u})$.

While the additivity may be justified through semantic considerations, our reason for imposing
``closure wrt regular languages'' is motivated by computational concerns: given that $S\subseteq\XC$ (if regular)
has a simple computational representation (e.g. as a {\sc dfa} accepting exactly the arguments in $S$) it is desirable that
subsets of $\XC$ ``related to'' $S$, e.g. through the property of attacking its members, can also be so described. In principle
this allows for the outcome of $\mu(S)$ to be described as a {\sc dfa}.

Finally the concept of invertibility addresses issues arising
with the ``inverse'' mapping: just as $\mu~:~2^{\XC}\rightarrow2^{\XC}$ describes the subset of $\XC$ that \emph{attacks}
a given $S\subseteq\XC$, so $\nu~:~2^{\XC}\rightarrow2^{\XC}$ describes the subset of $\XC$ that is \emph{attacked by}
$S$. The justification of additivity and closure properties is on similar grounds to those
used with $\mu$.

Given the basic desiderata that mappings defining attacks ought
to satisfy stated in Definition~\ref{defn:reasonable-attack}, we now turn to the issue of how general functions satisfying these
desiderata can be constructed. For this purpose we introduce the concept of attack expressions over a finite alphabet $\Sigma$.
\begin{defn}\label{defn:attack-expr}
A well-formed \emph{attack expression} ({\sc ae}) over $\Sigma$  is a sentence constructed by the following rules.
\begin{enumerate}
\item[1.]
For all $\sigma_i\in\Sigma$, $\sigma_i$ is an attack expression over $\Sigma$.
\item[2.]
The symbol $I$ (for \emph{identity}) is an attack expression
over $\Sigma$.

\item[3.]
If $p$ and $q$ are two attack expressions over $\Sigma$ then $p~\cup~q$ is also
an attack expression over $\Sigma$.
\item[4.]
If $p$ is an attack expression and $K_{\Sigma}$ is a regular expression (using the operations $\set{+,\cdot,*}$) over
\emph{only} symbols from $\Sigma$ (i.e. the identity symbol $I$ does \emph{not} occur in $K_{\Sigma}$)
then all of
$K_{\Sigma}\cdot p$, $p\cdot K_{\Sigma}$, $p/K_{\Sigma}$, $K_{\Sigma}/p$ and $p\cap K_{\Sigma}$ are attack expressions.
\item[5.]
If $p$ is an attack expression over $\Sigma$ then $(p)$ and $\gamma(p)$ for
$\gamma\in\set{hd,~tl}$ are attack expressions over $\Sigma$.
\item[6.]
The only attack expressions over $\Sigma$ are those formed by a finite number of applications
of (1) through (5).
\end{enumerate}
\end{defn}

Let $\AC\EC(\Sigma)$ denote the set of all well-formed attack expressions over $\Sigma$ and for $a\in\AC\EC(\Sigma)$
let $size(a)$ denote the number of \emph{operations} (i.e. applications of rules 3--5) used to
define $a$. The key motivation for this formalism is in describing the \emph{attack} structure relating a set of arguments.
Hence each $a\in\AC\EC(\Sigma)$, defines a mapping $\underline{a}~:~2^{\Sigma^{*}}\rightarrow 2^{\Sigma^{*}}$
as follows
\begin{defn}\label{defn:attackers-of}
Let $p\in\AC\EC(\Sigma)$ and $S\subseteq\Sigma^{*}$, 
the set $\underline{p}(S)$ is given by the rules below:
\[
\underline{p}(S)~~=~~\left\{
{
\begin{array}{ll}
\set{\sigma_i}&\mbox{ if }p=\sigma_i\\
S&\mbox{ if }p=I\\
(\underline{b}(S)~\cup~\underline{c}(S))&\mbox{ if }p=(b~\cup~c)\\
K_{\Sigma}\cdot\underline{b}(S)&\mbox{ if }p=K_{\Sigma}\cdot b\\
\underline{b}(S)\cdot K_{\Sigma}&\mbox{ if }p=b\cdot K_{\Sigma}\\
\underline{b}(S)/K_{\Sigma}&\mbox{ if }p=b/K_{\Sigma}\\
K_{\Sigma}/\underline{b}(S)&\mbox{ if }p=K_{\Sigma}/b\\
\underline{b}(S)\cap K_{\Sigma}&\mbox{ if }p=b\cap K_{\Sigma}\\
\gamma(\underline{b}(S))&\mbox{ if }p=\gamma(b)
\end{array}
}
\right.
\]
The unary operations -- $\set{hd,~tl}$ -- are defined as follows for $S\subseteq\Sigma^{*}$:
\[
hd(S)~~=~~\set{~\sigma_i~\in \Sigma:~\exists w \mbox{ s.t. } \sigma_i\cdot w\in S}
\]
\[
tl(S)~~=~~\set{~w~:~\exists ~\sigma_i~\in \Sigma \mbox{ s.t. } \sigma_i\cdot w\in S}
\]
Note that $tl(\set{\sigma})=\set{\varepsilon}$ and $tl(\emptyset)=hd(\set{\varepsilon})=hd(\emptyset)=\emptyset$.
\end{defn}

Before analysing the properties of the proposed attack expressions, we comment on the set
of operations -- $\set{\cup,~\cdot K_{\Sigma},~K_{\Sigma}\cdot,~/K_{\Sigma},~K_{\Sigma}/,~\cap K_{\Sigma},~hd,~tl}$ that are provided. In particular we note the limited
way in which $\cdot$ and $\cap$ may be used and the absence of complement and Kleene $*$ operators, despite the
fact that regular languages are closed under the last two of these. The immediate problem with allowing
arbitrary usage of $\cdot$ and $\cap$ concerns the fact that, for expressions such as $p\cdot q$ or $p\cap q$
we cannot, in general, guarantee that the mappings $\underline{p}\cdot\underline{q}$ or $\underline{p}\cap\underline{q}$
are \emph{additive}. For example, suppose that $p=q=I$, $\Sigma=\set{0,1}$ and $S=\set{0^k~:~k\geq1}$.
Then $(\underline{p}\cdot\underline{q})(S)=S\setminus\set{0}$, however
\[
\bigcup_{w\in S}~(\underline{p}\cdot\underline{q})(\set{w})~=~\set{0^{2k}~:~k\geq1}~~\not=S\setminus\set{0}
\]
Similarly if $p=I$, $q=tl(I)$, $S=\set{1^{k}~:~k\geq1}\cup\set{0\cdot1^{k}~:~k\geq1}$ we get
\[
(\underline{p}\cap\underline{q})(S)~=~S\cap tl(S)~=\set{1^k~:~k\geq 1}
\]

but
\[
\bigcup_{w\in S}~(\underline{p}\cap\underline{q})(\set{w})~=~\bigcup_{w\in S}\underline{p}(\set{w})\cap\underline{q}(\set{w})~=~
\bigcup_{w\in S}~(\set{w}\cap\set{tl(w)})~~=~~\emptyset
\]
We could, of course, avoid this by directly defining $\mu(S)$ to be $\bigcup_{w\in S}~\mu(\set{w})$, i.e. restricting
the domain of $\mu$ to $\Sigma^{*}$. In this case, however, we cannot always ensure that $\underline{p}\cdot\underline{q}$
preserves regularity. For example, using $p=I$, $q=\set{1}\cdot I$ with
$S=\set{0^k~:~k\geq1}$ (which is a regular language),
\[
\bigcup_{w\in S}~\underline{p}(w)\cdot\underline{q}(w)~~=~~\set{0^k\cdot1\cdot0^k~:~k\geq1}
\]
is not regular. This behaviour arises since regular languages are not closed under \emph{unbounded} union
(only with respect to \emph{finite} union). The issues underpinning the absence of $\cdot$
from the class of allowed operations are easily seen also to arise were $*$ to be added. Finally allowing
complementation would lead to mappings which were not monotonic (and thus could not be additive).

Theorem \ref{thm:reasonableness} provides a first confirmation of the soundness of the proposed approach by showing\footnote{The proofs relevant to this section are given in \ref{sec:proof-sect-representation}.} that the mappings arising via attack expressions have appropriate properties.

\begin{restatable}{thm}{thmreasonablemapping}
\label{thm:reasonableness}
Let $p\in\AC\EC(\Sigma)$ be \emph{any} attack expression over $\Sigma$. The mapping $\underline{p}~:~2^{\Sigma^{*}}\rightarrow2^{\Sigma^{*}}$
is reasonable.

\end{restatable}

\subsection{{\sc af} specification}

We now have the basic elements of our formal descriptive mechanism for infinite frameworks, the idea being that the set of arguments is specified as a regular language $\XC\subseteq\Sigma^{*}$ and the attack relation is specified through an attack expression $a$. 
In fact, given an element $v$ of $\XC$ the set $T$ of attackers of $v$ might be defined as $T = \underline{a}(\set{v})$.
Considering now a set $S \subseteq \XC$, the set of attackers of $S$ is given by 
$\set{v\in\Sigma^{*}~:~\exists~u\in S\mbox{ s.t. } v \in \underline{a}(\set{u})}$ which is equal to $\underline{a}(S)$ by additivity of $\underline{a}$.
However, it should be noted that while the attack relation must be a subset of $\XC \times \XC$, it might be the case that for some $S \subseteq \XC$  $\underline{a}(S) \nsubseteq \XC$. To fix this (actually minor) problem we need to introduce some further notation.

\begin{defn}
Let $\mu~:~2^{\Sigma^{*}}\rightarrow 2^{\Sigma^{*}}$ and $\XC\subseteq\Sigma^{*}$ a regular language, we define $\mu_{\XC}:~2^{\XC}\rightarrow 2^{\XC}$ as
\[
\mu_{\XC}(S)~~=~~\mu(S)\cap\XC 
\]
\end{defn}

Proposition \ref{prop:res-restrict} shows that considering $\mu_{\XC}$ instead of $\mu$ does not affect the property of being a reasonable mapping.

\begin{restatable}{propn}{propreasonablemapping}
\label{prop:res-restrict}
If $\mu$ is a reasonable mapping over the domain $2^{\Sigma^{*}}$ then $\mu_{\XC}$ is a reasonable mapping over the domain $2^{\XC}$.
\end{restatable}

We can now formally introduce the notion of {\sc af} specification.

\begin{defn}\label{defn:generic-inf-dfa}
Let $\Sigma$ be a finite alphabet of symbols and $\XC\subseteq\Sigma^{*}$ be a regular language. An {\sc af} \emph{specification} ({\sc afs})
is a pair $\tuple{\MC,a}$ where $\MC=\tuple{Q,\Sigma,q_0,\delta,F}$ is some finite automaton\footnote{We do not
require that $\MC$ be limited to a \emph{specific} class of automata since there is no
expressive gain in imposing such a restriction (see Fact~\ref{fact:reg-fa} in \ref{subsection:flt}).}
for which $L(\MC)=\XC$ and $a\in\AC\EC(\Sigma)$ is a well-formed attack expression over $\Sigma$.  Given
an {\sc afs} $\tuple{\MC,a}$, the relation $\rightarrow_{a}$
over $\XC\times\XC$ is defined by $u\rightarrow_{a}v$ (read as ``$v$ is attacked by $u$'') if $u\in \underline{a}_{\XC}(\set{v})$.
We call $\tuple{\XC, \rightarrow_{a}}$ the argumentation framework induced by $\tuple{\MC,a}$.
\end{defn}

Note that, given an {\sc af} \emph{specification} and a set $S \subseteq \XC$, the set of attackers of $S$ denoted as $\pi_{a}^{-}(S) \triangleq \set{~v\in\Sigma^{*}~:~\exists~u\in S \mbox{ s.t. }v\rightarrow_{a}u}$ is equal to $\underline{a}_{\XC}(S)$, by additivity.

\begin{examplecont}[\ref{ex:chains}]
Continuing with the examples, we can now complete the specification of $AF_U$, $AF_D$, $AF_M$ and $AF_R$ by identifying the relevant attack expressions.

As to $AF_U$ we note that, for a generic $i \geq 2$ the argument (corresponding to) $0^{i}$ is attacked by the argument $0^{i-1}$, leading to the attack expression $a=tl(I)$. Note that $\underline{a}(\set{0})=\set{\varepsilon}$ but $\underline{a}_{\XC}(\set{0})=\emptyset$.
On the other hand, in $AF_D$, for a generic $i \geq 1$, the argument $0^{i}$ is attacked by the argument $0^{i+1}$, leading to the attack expression $I\cdot 0$.\footnote{Of course, we could equally write $0\cdot I$ in this case.}

$AF_M$ requires a more articulated attack expression.
Here each argument $A_i$ corresponding to $0 \cdot (00)^{i-1}$ is attacked by the argument $B_i$ corresponding to $00 \cdot (00)^{i-1}$ and viceversa each argument $B_i$ is attacked by the argument $A_i$.
As to the attacks from a generic $B_i$ to a generic $A_i$ we note that the attacker can simply be obtained by adding a trailing $0$. We can not however use the simple expression $I\cdot 0$ as above since this would entail not only that any $A_i$ is attacked by $B_i$ but also that any $B_i$ is attacked by $A_{i+1}$ which is not the case.
The attack expression has therefore to specify that the trailing $0$ applies only to the elements of the sublanguage $0 \cdot (00)^{*}$, giving rise to the expression $(I \cap (0 \cdot (00)^{*})) \cdot 0$.
Similarly, the attacks from a generic $A_i$ to a generic $B_i$ can be obtained using the $tl$ operator and properly restricting its application, giving rise to $tl(I \cap ((00) \cdot (00)^{*}))$.
The complete attack expression of $AF_M$ is obtained by the union of the two expressions above: $a_M=((I \cap (0 \cdot (00)^{*})) \cdot 0) \cup tl(I \cap ((00) \cdot (00)^{*}))$.

As to $AF_R$, we note first that the attacks between arguments $A_i$ and $B_i$ are analogous to the case of $AF_M$ with the difference that $A_i$ corresponds to $0 \cdot (000)^{i-1}$ and $B_i$ corresponds to $00 \cdot (000)^{i-1}$. Hence, similarly to above, we obtain the expressions $(I \cap (0 \cdot (000)^{*})) \cdot 0$ and $tl(I \cap ((00) \cdot (000)^{*}))$.
As to the attacks from an argument $C_i$, corresponding to $000 \cdot (000)^{i-1}$, to an argument $B_i$ corresponding to $00 \cdot (000)^{i-1}$ we note that they can again be represented through the addition of a trailing $0$ yielding $(I \cap (00 \cdot (000)^{*})) \cdot 0$.
The complete attack expression of $AF_R$ turns out to be $a_R=((I \cap (0 \cdot (000)^{*})) \cdot 0) \cup (tl(I \cap ((00) \cdot (000)^{*}))) \cup ((I \cap (00 \cdot (000)^{*})) \cdot 0)$.
\end{examplecont}
 
As a further remark, we note that the attack expression $I$ captures the case of an argumentation framework where each argument attacks itself (and only itself) which has been shown not to be representable in the naive approach using regular languages in Section \ref{section:naive-rep}.

\subsection{Inverting the attack expression}

Definition \ref{defn:generic-inf-dfa} provides a formal specification of the attackers of an element (subset) of $\Sigma^{*}$. However it is useful to consider also the specification of the arguments attacked by a given element (subset) of $\Sigma^{*}$. This is possible since the function $\underline{a}$ is invertible as shown below by Theorem \ref{thm:attack-comp}. In particular, the proof of the theorem (see \ref{sec:proof-sect-representation}) allows one to construct from a given $a\in\AC\EC(\Sigma)$ a related expression $a^{+}$ with the property that for all $u,~v\in\Sigma^{*}$, $u\in\underline{a}(\set{v})$ if and only if $v\in\inv{a}(\set{u})$.
The expressions $a^{+}$ use the same basic elements as $\AC\EC(\Sigma)$ plus the 
$rev()$ operator\footnote{With a little abuse of notation, in order to simplify the presentation, we will apply the $rev()$ operator also to single words, i.e. for $w \in \Sigma^*$ and $\set{w'}=rev(\set{w})$, $rev(w)=w'$.}, which does not affect any of the desired properties.

The following properties, shown to be valid in the proof of Theorem \ref{thm:attack-comp}, provide the basic elements to derive the expressions $a^{+}$ from $\underline{a}$.

\begin{fact}\mbox{}
\label{conjectures-inverse}

  \begin{enumerate}
  \item If $a = \sigma_{i}$ then $\inv{a}(S) = tl(I \cap \sigma_i)\cdot \Sigma^{*}$; If $a = I$ then $\inv{a}(S) = S$. \label{Prop0}
  \item If $a=b\cup c$, then $\inv{a}(S)~=~\inv{b}(S)\cup\inv{c}(S)$. \label{Prop1}
  \item If $a=b\cdot K_{\Sigma}$, then $\inv{a}(S)~=~\inv{b}(S/K_{\Sigma})$.\label{Prop2}
  \item If $a=K_{\Sigma}\cdot b$, then $\inv{a}(S)~=~\inv{b}(rev(rev(S)/rev(K_{\Sigma})))$. \label{Prop3}
  \item If $a=b/K_{\Sigma}$, then $\inv{a}(S)~=~\inv{b}(S\cdot K_{\Sigma})$. \label{Prop4}
  \item If $a=K_{\Sigma}/b$, then $\inv{a} (S)~=~ \inv{b}(rev(rev(K_{\Sigma})/rev(S)))$. \label{Prop5}
  \item If $a=b\cap K_{\Sigma}$, then $\inv{a}(S)~=~\inv{b}(S\cap K_{\Sigma})$. \label{Prop6}
  \item If $a=hd(b)$, then $\inv{a}(S)~=~\inv{b}(((S\cap\Sigma)\cdot\Sigma^{*}))$. \label{Prop7}
  \item If $a=tl(b)$, then $\inv{a}(S)~=~\inv{b}(\Sigma\cdot S)$. \label{Prop8}
  \end{enumerate}
\end{fact}

\begin{restatable}{thm}{thmclosureregular}
\label{thm:attack-comp}
Let $\tuple{\MC,a}$ be an {\sc afs} with $L(\MC)=\XC$ and $a\in\AC\EC(\Sigma)$. The mapping $\inv{a}:2^{\Sigma^{*}}\rightarrow2^{\Sigma^{*}}$, defined as  $\inv{a}(S) = \set{~v\in\Sigma^{*}~:~\exists~u\in S\mbox{ s.t. }u \in \underline{a}(v)}$ is closed wrt regular languages.
\end{restatable}

It remains now to show that the inverse of an attack expression actually provides the set of the arguments attacked by a set. 

\begin{restatable}{propn}{propnmapping}
Let $\tuple{\MC,a}$ be an {\sc afs} with $L(\MC)=\XC$ and $a\in\AC\EC(\Sigma)$. Define the mapping, 
$\pi_{a}^{+}~:~2^{\XC}\rightarrow2^{\XC}$ by
\[
\pi_{a}^{+}(S)\mbox{ $=$ } \set{~v\in\XC~:~\exists~u\in S\mbox{ s.t. }u\rightarrow_{a}v}
\]
It holds that $\pi_{a}^{+}(S) = \inv{a}_{\XC}(S)$.
\end{restatable}

By Theorem \ref{thm:attack-comp} and the regularity of $\XC$ it is easy to see that $\pi_{a}^{+}$ is additive and preserves regularity.

\begin{examplecont}[\ref{ex:chains}]
Completing our example, we can now derive the mapping $\pi_{a}^{+}$ for $AF_U$, $AF_D$, $AF_M$ and $AF_R$.

As to $AF_U$ from the attack expression $tl(I)$ applying Fact \ref{conjectures-inverse}.\ref{Prop8} (and \ref{conjectures-inverse}.\ref{Prop0} for $I$) we get $\inv{a}(S)= I(\Sigma \cdot S) = \Sigma \cdot S = 0 \cdot S$.
As to $AF_D$, from the attack expression $I \cdot 0$ applying Fact \ref{conjectures-inverse}.\ref{Prop2} (and \ref{conjectures-inverse}.\ref{Prop0} for $I$) we get $\inv{a}(S) = I(S / 0) = S/0$ which (in view of $S \subseteq 0\cdot0^{*}$) is equivalent to $tl(S)$.

As to $AF_M$, given the attack expression  $a_M =((I \cap (0 \cdot (00)^{*})) \cdot 0) \cup tl(I \cap ((00) \cdot (00)^{*}))$ by Fact \ref{conjectures-inverse}.\ref{Prop1} we can examine separately the two terms $b = ((I \cap (0 \cdot (00)^{*})) \cdot 0)$ and $c = tl(I \cap ((00) \cdot (00)^{*}))$.
As to $\inv{b}$, from  \ref{conjectures-inverse}.\ref{Prop2} we get $\inv{d}(S/0)$ with $d = (I \cap (0 \cdot (00)^{*}))$. Applying then  \ref{conjectures-inverse}.\ref{Prop6} (and \ref{conjectures-inverse}.\ref{Prop0} for $I$) we get $\inv{b}(S) = (S / 0) \cap (0 \cdot (00)^{*})$.
As to $\inv{c}$, from Fact \ref{conjectures-inverse}.\ref{Prop8} we get $\inv{e}(\Sigma \cdot S)$ with $e= I \cap ((00) \cdot (00)^{*})$.
Applying again \ref{conjectures-inverse}.\ref{Prop6} and \ref{conjectures-inverse}.\ref{Prop0} and taking into account $\Sigma = \set{0}$ we get $\inv{c}(S)= (0 \cdot S) \cap ((00) \cdot (00)^{*})$.
Summing up $\inv{a_M}(S)= ((S / 0) \cap (0 \cdot (00)^{*})) \cup ((0 \cdot S) \cap ((00) \cdot (00)^{*}))$.

The case of $AF_R$ is analogous, yielding $\inv{a_R}(S) = ((S / 0) \cap (0 \cdot (000)^{*})) \cup ((0 \cdot S) \cap ((00) \cdot (000)^{*})) \cup ((S / 0) \cap (00 \cdot (000)^{*}))$.
\end{examplecont}

\subsection{Representation of finite {\sc af}s and combination of {\sc af} specifications}

The previous sections show how the proposed {\sc af} specification mechanism can deal with (up to now simple) infinite frameworks. One may then wonder whether this mechanism is suitable to describe \emph{finite} {\sc af}s as well or its structure is somehow bounded to the infinite case.
In fact this problem does not arise: any finite {\sc af} $\tuple{\XC,\AC}$ can be easily described via the mechanisms proposed in this paper.
Noting that $\XC$ is finite simply choose $\Sigma=\XC$ as the underlying alphabet, and let $\MC$ be the trivial associated automaton.
The set $\AC$ is a finite subset of $\Sigma \times \Sigma$ and treated directly as a regular language $L_A~=~\set{x\cdot y~:~\tuple{x,y}\in\AC}$. We then have $a\in\AC\EC(\Sigma)$ specified by
$a = hd((\Sigma\cdot I)\cap L_A)$, giving $\pi_{a}^{-}(S)~=~hd((\Sigma\cdot S)\cap L_A)$ and (after some manipulation)
$\pi_{a}^{+}(S)~=~tl((S \cdot\Sigma)\cap L_A)$.

For example if $\XC~=~\set{w,x,y,z}$ and 
$\AC~=~\set{\tuple{w,x},\tuple{x,y},\tuple{y,z},\tuple{z,w}}$ then
$L_A~=~\set{wx,xy,yz,zw}$ with, for example,
\[
\begin{array}{l}
\begin{array}{lcl}
\pi_{a}^{-}(\set{w,y})&\mbox{ $=$ }&hd((\set{w,x,y,z}\cdot \set{w,y}) \cap\set{wx,xy,yz,zw})\\
&\mbox{ $=$ }&hd(\set{ww,wy,xw,xy,yw,yy,zw,zy}\cap\set{wx,xy,yz,zw})\\
&\mbox{ $=$ }&hd(\set{xy,zw})~~=~~\set{x,z}
\end{array}\\
\\
\begin{array}{lcl}
\pi_{a}^{+}(\set{w,y})&\mbox{ $=$ }&tl((\set{w,y}\cdot\set{w,x,y,z})\cap\set{wx,xy,yz,zw})\\
&\mbox{ $=$ }&tl(\set{ww,yw,wx,yx,wy,yy,wz,yz}\cap\set{wx,xy,yz,zw})\\
&\mbox{ $=$ }&tl(\set{wx,yz})~~=~~\set{x,z}
\end{array}
\end{array}
\]

Having shown that finite {\sc af}s do not raise, \emph{per se} any expressiveness concern, a further important question has to be addressed: one may wonder whether it is possible to give the specification of an {\sc af} resulting from the combination of a finite subframework with one or more infinite subframeworks, with the different subframeworks linked together by finite attack relations.
This kind of combined specification is particularly relevant in practice.
To have an example, consider again the frameworks $AF_M$ and $AF_R$, concerning cases of temporal projection with initial information acquired from different sources at the same time.
Clearly, one has also to cover the case where information is acquired from different sources at different times. 
As a very simple example, consider a slight modification of the situation represented by $AF_R$, so that information from the third more reliable source, $C$, is acquired with some delay (to keep things simple, let say one time instant later) wrt the information from sources $A$ and $B$. 
This situation could be represented with a framework composed by two subframeworks, a finite one, consisting of two mutually attacking arguments corresponding to the information initially acquired from $A$ and $B$, and an infinite one, with the same structure as $AF_R$.
More generally, frameworks with this kind of structure correspond to cases where a reasoning (or dialogue) process enters a non terminating iterative behavior after some initial non iterative steps, which is clearly a more general (and possibly more common) situation wrt the cases of \virg{iterative behavior from the beginning} we have considered in our simple illustrative examples.
We will now show how this kind of structure can be captured in our formalism.

Let $AF_0=\tuple{\XC_0,\AC_0}$ a finite {\sc af} and $AF_1=\tuple{\XC_1,\AC_1}, \ldots, AF_n=\tuple{\XC_n,\AC_n}$ a finite sequence of infinite frameworks with specifications $\tuple{\MC_1, a_1}, \ldots, \tuple{\MC_n, a_n}$ such that for each $1 \leq i \leq n$ $L(\MC_i) = \XC_i \subseteq \Sigma_{i}^{*}$ and $a_i \in \AC\EC(\Sigma_i)$.

Letting $\Sigma_0=\XC_0$ we assume without loss of generality that the alphabets used for the different frameworks are pairwise disjoint namely, for $0 \leq i \leq n$,  $0 \leq j \leq n$, $i \neq j$, $\Sigma_i \cap \Sigma_j = \emptyset$.

Consider now the problem of specifying a framework $AF_{\cup}=\tuple{\XC_{\cup},\AC_{\cup}}$ with the following structure:
\begin{itemize}
\item $\XC_{\cup} = \bigcup_{i \in \set{0, \ldots, n}}\XC_i$;
\item $\AC_{\cup} = \bigcup_{i \in \set{0, \ldots, n}}\AC_i \cup \bigcup_{i,j \in \set{0, \ldots, n}, i \neq j} \AC_{i,j}$, where $\AC_{i,j}$ is an arbitrary \emph{finite} subset of $\XC_i \times \XC_j$.
\end{itemize}

In words, $AF_{\cup}$ includes all the subframeworks $AF_0, \ldots AF_n$ with their \virg{internal} attack relations $\AC_0, \ldots \AC_n$ plus new arbitrary  \emph{finite}  attack relations $\AC_{i,j}$ linking each pair of subframeworks and representing the additional attacks from elements of $\XC_i$ to elements of $\XC_j$.

The question is now how to derive the specification of $AF_{\cup}$ from the specifications of the subframeworks $AF_0, \ldots AF_n$ and from the new attacks $\AC_{i,j}$.
As to the reference alphabet, clearly $\Sigma_{\cup} = \bigcup_{i \in \set{0, \ldots, n}} \Sigma_i$.
Thanks to the hypothesis of disjointness of the alphabets $\Sigma_i$, we know that also the sets of arguments $\XC_i$ are disjoint and we can safely define $\XC_{\cup} = \bigcup_{i \in \set{0, \ldots, n}} \XC_i$.
As $\XC_{\cup}$ is the union of a set of regular languages it is a regular language too (Fact \ref{fact:closure-props}, \ref{subsection:flt}), whose automaton $\MC_{\cup}$ can be effectively derived from the automata $\MC_i$ (Fact \ref{fact:properties}.b, \ref{subsection:flt}).

As to the attack expression, it has to preserve, in the new framework, the attack relations $\AC_0, \ldots \AC_n$ of all the subframeworks and include the new attacks $\AC_{i,j}$.

As to $\AC_0$, the corresponding attack expression in $AF_{\cup}$ is exactly the same as for $\AC_0$ in isolation.
Letting $L_0~=~\set{x\cdot y~:~\tuple{x,y}\in\AC_0}$, the attack expression\footnote{We remark that using $\Sigma_{\cup}$ instead of $\Sigma_0$ would give the same result.} for the finite subframework is $\hat{a}_{0} = hd((\Sigma_0 \cdot I) \cap L_0)$.

As to the attack relations $\AC_1, \ldots \AC_n$ of the infinite subframeworks $AF_1, \ldots AF_n$, each relation $\AC_i$ is described by an attack expression $a_i$ and we need to devise a corresponding attack expression $\hat{a}_{i}$ which preserves exactly the same attacks between the arguments of $AF_i$ in the context of $AF_{\cup}$, i.e. for any set $S \subseteq \Sigma_{\cup}$ it must hold that $\underline{\hat{a}}_{i}(S) = \underline{\hat{a}}_{i}(S \cap \XC_i) = \underline{a}_{i}(S \cap \XC_i)$.
It can be seen that such an expression $\hat{a}_{i}$ can be obtained from $a_i$ by applying two simple replacement operations concerning the basic elements $\sigma$ and $I$ (Rules 1 and 2 of Definition \ref{defn:attack-expr}):
\begin{itemize}
\item each occurrence of $\sigma$ (with $\sigma$ an element of $\Sigma_i$) in $a_i$ is replaced by $hd(\sigma \cdot(I \cap K_i))$ within $\hat{a}_{i}$, where $K_i$ is the regular expression specifying $\XC_i$;
\item each occurrence of $I$ in $a_i$ is replaced by $(I \cap K_i)$ within $\hat{a}_{i}$, where $K_i$ is as above.
\end{itemize}

It is immediate to see that if the attack expression $a_i$ consists exactly of $I$ or $\sigma$, the above replacements ensure that $\hat{a}_{i}$ satisfies the desired property.
By inspection of the rules 3-5 of Definition \ref{defn:attack-expr} it is also easy to see that the desired property is preserved in more articulated expressions, constructed by repeated application of these rules starting from the basic elements, without requiring any further modification.

Let us turn now to the specification of the additional finite attack relations $\AC_{i,j}$ ($i \neq j$) between subframeworks. 
First, observe that each $\AC_{i,j}$ can be specified directly as a regular language $L_{i,j} = \set{ x \cdot y : \tuple{x,y} \in \AC_{i,j}}$ and let again $K_i$ be the regular expression corresponding to $\XC_i$.
Then, note that the expression $(K_i \cdot I) \cap L_{i,j}$ is useful to select words in $L_{i,j}$ when appropriate and that the quotient operator wrt $K_j$ can then be applied to extract the subword referring to the attacker. Finally, we have to ensure that the extracted subword belongs to $K_i$ which can be obtained by an intersection operation.
In summary, the attack expression specifying an attack relations $\AC_{i,j}$ ($i \neq j$) is $\hat{a}_{i,j} = (((K_i \cdot I) \cap L_{i,j}) / K_j) \cap K_i$.

Putting together the various subexpressions we have devised above, the complete attack expression $a_{\cup}$ for $AF_{\cup}$ is given by $a_{\cup} = \bigcup_{i \in \set{0,\ldots,n}} \hat{a}_{i}  \cup \bigcup_{i,j \in \set{0,\ldots,n}, i\neq j} \hat{a}_{i,j}$.

\section{Computing with {\sc af} Specifications}\label{sec:computing}

In this section we show\footnote{The proofs relevant to this section are given in \ref{sec:proofs-sect-computing}.} that a number of problems that are well-known to be efficiently, i.e. polynomial time,
decidable in finite {\sc af}s may be effectively handled within the context of {\sc af} specifications, i.e. there exist procedures which are certain to produce the answer in a finite number of steps.

Theorem \ref{thm:afs-ops} 
provides the main result of this section.

\begin{restatable}{thm}{effectivealghs}
\label{thm:afs-ops}
Let $\tuple{\MC,a}$ be an {\sc afs}, 
with induced argumentation framework $\tuple{\XC, \AC}$ and
$S\subseteq \XC$. 
The following problems are decidable.
\begin{enumerate}
\item[a.]
Deciding if the set $S$ is conflict free 
\item[b.]
For $x\in\XC$, deciding if $x\in\FC(S)$, i.e. whether $x$ is acceptable to $S$ 
\item[c.]
Deciding if $S\in\EC_{adm}(\tuple{\XC,\AC})$, i.e. whether $S$ is admissible 
\item[d.]

Deciding if $S\in\EC_{stab}(\tuple{\XC,\AC})$, i.e. whether $S$ is a stable extension 
\item[e.]
Constructing a {\sc dfa} accepting $\FC(S) = \XC \setminus \pi_{a}^{+}(\XC \setminus \pi_{a}^{+}(S))$, i.e. the set of arguments acceptable to $S$.
\item[f.]
Deciding if $S\in\EC_{comp}(\tuple{\XC,\AC})$, i.e. whether $S$ is a complete extension 
\end{enumerate}
\end{restatable}

The algorithms described in the proof of Theorem~\ref{thm:afs-ops} provide full solutions to the decision problems listed above and are applicable irrespective of whether $\tuple{\MC,a}$ gives rise to finitary frameworks or not. 
For the construction of the grounded extension we obtain a result applicable to finitary argumentation frameworks only. 
In fact, in \cite{dung:1995} it is shown that, letting $\FC^{1}(S)=\FC(S)$, and for $i > 1$ $\FC^{i}(S)=\FC(\FC^{i-1}(S))$, for a finitary argumentation framework the grounded extension is given by $\bigcup_{i \geq 1} \FC^{i}(\emptyset)$.

Using this result we can in some cases obtain (a representation of) the grounded extension through Algorithm \ref{algorithm:grounded}.

\begin{algorithm}[h]
\caption{Computing $\EC_{gr}(\tuple{\XC,\AC})$ from (finitary) {\sc afs} $\tuple{\MC,a}$ }\label{algorithm:grounded}
\begin{algorithmic}[1]
\STATE $X_0~:=L(\MC)$;
\STATE $Y_{0}~:=~X_0\setminus\pi_{a}^{+}(X_0)$;
\STATE $G~:=~Y_0$;
\STATE $k~:=~0$;
\WHILE{$Y_k\not=\emptyset$}
\STATE $X_{k+1}~:=~X_{k}\setminus(Y_{k}\cup\pi_{a}^{+}(Y_{k}))$;
\STATE $Y_{k+1}~:=~X_{k+1}\setminus\pi_{a}^{+}(X_{k+1})$;
\STATE $G~:=~G\cup Y_{k+1}$;
\STATE $k~:=~k+1$;
\ENDWHILE
\STATE \textbf{report} $\MC_G$ where $L(\MC_G)=G$.
\end{algorithmic}
\end{algorithm}

Algorithm~\ref{algorithm:grounded} effectively reproduces the sequence of iterations involving:
\begin{itemize}
\item identifying
unattacked arguments in $\XC$ (in l.2 and l.7), where we note that $Y_0$ computes $\FC(\emptyset)$, 
i.e. $\XC\setminus\pi_{a}^{+}(\XC)$ since
$\pi^{+}_{a}(\emptyset)=\emptyset$;

\item adding these to the extension being accumulated (in l.3 and l.8);

\item repeating this process on the {\sc af} induced by the arguments remaining after removing these and those
they attack (in l.6). 
\end{itemize}

The process terminates when the set of unattacked arguments is empty.
Let us exemplify the application of Algorithm~\ref{algorithm:grounded} to $AF_R$.
Starting with the set of unattacked arguments, we get $Y_{0} = 000 \cdot (000)^{*}$ at l.2.
Then, given that $\pi_{a}^{+}(Y_{0}) = 00 \cdot (000)^{*}$  in the first iteration of the {\bf while} loop we get $X_{1} = 0 \cdot (000)^{*}$ at l.6 and, given that $\pi_{a}^{+}(X_{1}) = 00 \cdot (000)^{*}$, we get $Y_{1}=X_{1}$ at l.7 and $G = (000 \cdot (000)^{*}) \cup (0 \cdot (000)^{*})$ then in the next iteration $X_{2}=Y_{2}=\emptyset$ and the algorithm terminates.

Algorithm~\ref{algorithm:grounded}, however, does not guarantee termination, since it terminates only in the cases where there is $k$ such that $\FC^{k+1}(\emptyset)= \FC^{k}(\emptyset)$. For instance in the framework $AF_U$ of Example \ref{ex:chains}, in which for $i\geq1$, $\FC_i~=\set{~0^{2k-1}~:~1\leq k\leq i}$, no such $k$ exists. 

In such cases we can use properties of the operations involved in defining attack expressions together with known identities
for regular expressions to derive the form taken by arguments in the grounded extension directly. In the case of $AF_U$ we recall that,
\[
\begin{array}{lcl}
\XC&\mbox{ $=$ }&\set{~0^i~:~i\geq1}\\
a&\mbox{ $=$ }&tl(I)\\
\pi_{a}^{-}(S)&\mbox{ $=$ }&tl(S)\\
\pi_{a}^{+}(S)&\mbox{ $=$ }&0\cdot S
\end{array}
\]
For notational ease we write $\overline{~T~}$ for $\XC\setminus T$ so that $\FC_0=\emptyset$, $\FC_i=\FC(\FC_{i-1})$
and
\[
\FC(T)~~=~~\overline{~\pi_{a}^{+}(~\overline{\pi_{a}^{+}(T)}~)~}
\]
For $AF_U$ we get
\[
\begin{array}{lcl}
\FC_1&\mbox{ $=$ }&\overline{~0\cdot\overline{0\cdot\emptyset}~}\\
&\mbox{ $=$ }&\overline{~0\cdot\XC~}\\
&\mbox{ $=$ }&\set{0}\\
\FC_2&\mbox{ $=$ }&\overline{~0\cdot\overline{0\cdot0}~}\\
&\mbox{ $=$ }&\overline{~\set{0\cdot0}\cup\set{0\cdot0\cdot0\cdot0\cdot0^k~:~k\geq0}}\\

&\mbox{ $=$ }&\set{0,0\cdot0\cdot0}\\
&\mbox{ $\ldots$ }&\\
\FC_{i+1}&\mbox{ $=$ }&\overline{~0\cdot\overline{\set{0^{2k}~:~1\leq k\leq i}}~}\\
&\mbox{ $=$ }&\overline{~0\cdot\set{\set{0^{2k-1}~:~1\leq k\leq i}\cup\set{0^{2i+1}\cdot0^j~:~j\geq0}}}\\
&\mbox{ $=$ }&\overline{~\set{0^{2k}~:~1\leq k\leq i}\cup\set{0^{2i+2}\cdot0^{j}~:~j\geq0}}\\
&\mbox{ $=$ }&\set{0^{2k-1}~:~1\leq k\leq i+1}
\end{array}
\]
So that the least fixed-point is the infinite set $\set{0^{2i-1}~:~i\geq1}$: note that, by the analysis given, 
this \emph{is} a fixed-point and it is straightforward to show that no strict subset defines a fixed-point.

The analysis in the infinite case of those problems which are computationally intractable within finite {\sc af}s (cf. the summary presented in Fact~\ref{fact:alg-complexity}) is left to future work.
As a first step in this direction, we can show that, restricting again to finitary frameworks, limited decision procedures are possible for two of these problems. Note that we currently have no result on whether these problems are decidable: Theorem \ref{thm:semi-dec} ensures that, if this is not the case, they are at least semi-decidable.

\begin{restatable}{thm}{thmsemidec}
\label{thm:semi-dec}
Let $\tuple{\MC,a}$ be an {\sc afs} in which the induced argumentation framework $\tuple{\XC,\AC}$ is finitary. The following problems
are all \emph{semi-decidable}.
\begin{enumerate}
\item[a.]
Determining if $\EC_{stab}(\tuple{\XC,\AC})=\emptyset$.
\item[b.]
Given a \emph{finite} $R\subset\XC$, determining if $\forall~w\in R~\neg\mbox{{\sc ca}}_{adm}(\tuple{\XC,\AC},w)$.
\end{enumerate}
\end{restatable}

We note that the restriction to \emph{finitary} frameworks (and finite subsets of $\XC$ in the second part) is needed: 
without this the method used in proving Thm.~\ref{thm:semi-dec} could not be applied.

To conclude this section, as some of the results we provided rely on the condition that an {\sc af} specification gives rise to a finitary argumentation framework, one is interested in conditions ensuring that this holds.
We provide an easy sufficient condition to this purpose, namely the absence of the $*$ operator in the attack expression, leaving further investigations on this specific question for future work.

\begin{restatable}{propn}{propnfinitary}
\label{propn:finitary}
Let $\tuple{\MC,a}$ be an {\sc afs} with induced argumentation framework $\tuple{\XC, \rightarrow_{a}}$.
Let $\KC~=~\set{K_{\Sigma}^{1},~K_{\Sigma}^{2},\ldots,K_{\Sigma}^{r}}$ be the set of regular expressions used in defining $a$, i.e. with operations $b\cdot K_{\Sigma}$, $K_{\Sigma}\cdot b$, $K_{\Sigma}/b$, 
$b/K_{\Sigma}$ and $b\cap K_{\Sigma}$. 
If \emph{no} $K\in\KC$ uses the $*$ operator then $\tuple{\XC, \rightarrow_{a}}$ is finitary.
\end{restatable}

It is easy to see that the condition of Proposition~\ref{propn:finitary}  is not a \emph{necessary} one. Consider, for example
$a=hd(I\cdot(\sigma_1+\sigma_2)^{*})$. This is finitary, however, fails to satisfy the conditions of Propn.~\ref{propn:finitary}.
Less trivially, $a=(\sigma_1\cdot(\sigma_1)^{*})\cdot I\cdot(\sigma_2\cdot(\sigma_{2})^{*})$ will be
finitary when $\XC\cap\set{\sigma_1\cdot(\sigma_1)^{*}\cdot x\cdot\sigma_2\cdot(\sigma_{2})^{*}}$ is bounded
for all $x\in\XC$.

%
%
\section{{\sc af} Specifications at Work}\label{section:example-cases}

In this section we illustrate the suitability of our approach by showing how it can be used to provide a formal representation of four examples which altogether combine different features.
Two of the examples (presented in sections \ref{subsec:camver} and \ref{subsec:dungexamp}) are abstract in nature (they have been previously introduced in the literature mainly for the sake of theoretical analysis), while the examples of sections \ref{examp:negotiation} and \ref{examp:ambient} are inspired to realistic application domains (also taken from the literature) and have been introduced in Section \ref{subsec:multiag}.
The two abstract examples of sections \ref{subsec:camver} and \ref{subsec:dungexamp} concern infinite non-finitary {\sc af}s, while the \virg{application-oriented} examples of sections \ref{examp:negotiation} and \ref{examp:ambient} give rise to infinite finitary {\sc af}s.

Moreover, we remark that the examples of sections \ref{subsec:dungexamp} and \ref{examp:ambient} are based on a formalization in terms of a logic program featuring an infinite Herbrand base.
In fact, the formalization in argumentation terms of logic programs with infinite Herbrand base is, at a general level, one of the \virg{natural} applications of the proposed framework, given that a direct correspondence between logic programs with negation as failure and abstract argumentation frameworks has been established in Dung's paper itself. 

We will present the abstract examples before, as they are more suitable to illustrate in detail the technical use of the formalism as a specification tool in articulated frameworks, and later the \virg{application-oriented} examples, to give an account of some potential practical uses of the formalism without cluttering the description of the more realistic examples with too much technical details, derivable by analogy from the first examples. 

\subsection{The {\sc af} from Caminada and Verheij~\cite{camver:2010}}\label{subsec:camver}

In \cite{camver:2010} Caminada and Verheij describe a (non-finitary) {\sc af}, $\tuple{\XC,\AC}$ (see Figure \ref{fig:caminada}) with the
property that $\tuple{\XC,\AC}$ has no semi-stable extension\footnote{The existence of semi-stable extensions in finitary argumentation frameworks is analyzed in \cite{Weydert:2011}.}, i.e. admissible set $S$ for which $S\cup S^{+}$ is maximal
wrt $\subseteq$. The construction uses arguments
\[
\XC~~=~~\set{~A_i~:~i\geq 1~}~~\cup~~\set{~B_i~:~i\geq1~}~\cup~\set{~C_i~:~i\geq1~}
\]
linked by the attack relation, $\AC$, containing
\[
\begin{array}{l}
\set{~\tuple{A_i,A_i}~:~i\geq1~}~~\cup~~\set{~\tuple{B_j,A_i}~:~j\geq i\geq 1~}~~\cup\\
\set{~\tuple{B_i,C_i}~:~i\geq1~}~~\cup~~\set{~\tuple{B_j,B_i}~:~j>i\geq1~}~~\cup\\
\set{~\tuple{C_i,B_i}~:~i\geq1~}
\end{array}
\]

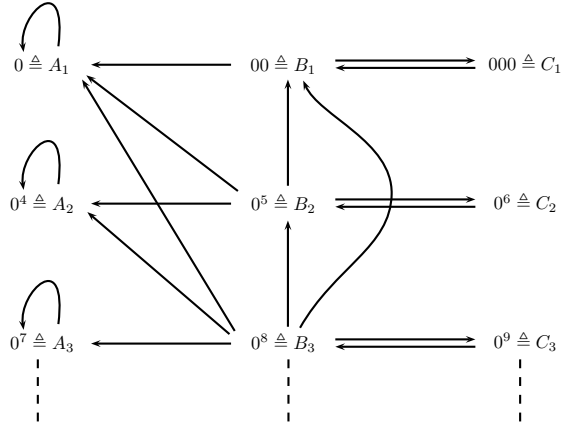
\begin{figure}[!htb]
\begin{center}
\scalebox{0.7} 
{
\begin{pspicture}(0,-4.33)(12.847187,4.33)
\usefont{T1}{ptm}{m}{n}
\rput(1.7754687,2.48){$0 \triangleq A_1$}
\usefont{T1}{ptm}{m}{n}
\rput(6.3554688,2.48){$00 \triangleq B_1$}
\usefont{T1}{ptm}{m}{n}
\rput(10.955469,2.5){$000 \triangleq C_1$}
\usefont{T1}{ptm}{m}{n}
\rput(1.7754687,-0.16){$0^4 \triangleq A_2$}
\usefont{T1}{ptm}{m}{n}
\rput(6.3554688,-0.16){$0^5 \triangleq B_2$}
\usefont{T1}{ptm}{m}{n}
\rput(10.955469,-0.14){$0^6 \triangleq C_2$}
\usefont{T1}{ptm}{m}{n}
\rput(1.7754687,-2.8){$0^7 \triangleq A_3$}
\usefont{T1}{ptm}{m}{n}
\rput(6.3554688,-2.8){$0^8 \triangleq B_3$}
\usefont{T1}{ptm}{m}{n}
\rput(10.955469,-2.78){$0^9 \triangleq C_3$}
\psbezier[linewidth=0.04,arrowsize=0.05291667cm 2.0,arrowlength=1.4,arrowinset=0.4]{->}(2.081875,2.8204918)(2.261875,4.31)(1.321875,3.43)(1.421875,2.71)
\psbezier[linewidth=0.04,arrowsize=0.05291667cm 2.0,arrowlength=1.4,arrowinset=0.4]{->}(2.081875,0.2004918)(2.261875,1.69)(1.321875,0.81)(1.421875,0.09)
\psbezier[linewidth=0.04,arrowsize=0.05291667cm 2.0,arrowlength=1.4,arrowinset=0.4]{->}(2.081875,-2.4595082)(2.261875,-0.97)(1.321875,-1.85)(1.421875,-2.57)
\psline[linewidth=0.04cm,arrowsize=0.05291667cm 2.0,arrowlength=1.4,arrowinset=0.4]{->}(5.361875,2.47)(2.701875,2.47)
\psline[linewidth=0.04cm,arrowsize=0.05291667cm 2.0,arrowlength=1.4,arrowinset=0.4]{->}(5.361875,-0.17)(2.701875,-0.17)
\psline[linewidth=0.04cm,arrowsize=0.05291667cm 2.0,arrowlength=1.4,arrowinset=0.4]{->}(5.361875,-2.83)(2.701875,-2.83)
\psline[linewidth=0.04cm,arrowsize=0.05291667cm 2.0,arrowlength=1.4,arrowinset=0.4]{->}(5.491875,0.07)(2.611875,2.27)
\psline[linewidth=0.04cm,arrowsize=0.05291667cm 2.0,arrowlength=1.4,arrowinset=0.4]{->}(5.431875,-2.59)(2.531875,2.19)
\psline[linewidth=0.04cm,arrowsize=0.05291667cm 2.0,arrowlength=1.4,arrowinset=0.4]{->}(5.331875,-2.65)(2.631875,-0.33)
\psline[linewidth=0.04cm,arrowsize=0.05291667cm 2.0,arrowlength=1.4,arrowinset=0.4]{->}(6.441875,0.17)(6.441875,2.19)
\psline[linewidth=0.04cm,arrowsize=0.05291667cm 2.0,arrowlength=1.4,arrowinset=0.4]{->}(6.441875,-2.51)(6.441875,-0.49)
\psbezier[linewidth=0.04,arrowsize=0.05291667cm 2.0,arrowlength=1.4,arrowinset=0.4]{->}(6.671875,-2.53)(7.331875,-1.37)(8.392074,-0.998178)(8.411875,0.01)(8.431676,1.018178)(7.031875,1.49)(6.731875,2.17)
\psline[linewidth=0.04cm,arrowsize=0.05291667cm 2.0,arrowlength=1.4,arrowinset=0.4]{->}(10.001875,2.41)(7.341875,2.41)
\psline[linewidth=0.04cm,arrowsize=0.05291667cm 2.0,arrowlength=1.4,arrowinset=0.4]{->}(10.001875,-0.23)(7.341875,-0.23)
\psline[linewidth=0.04cm,arrowsize=0.05291667cm 2.0,arrowlength=1.4,arrowinset=0.4]{->}(10.001875,-2.89)(7.341875,-2.89)
\psline[linewidth=0.04cm,arrowsize=0.05291667cm 2.0,arrowlength=1.4,arrowinset=0.4]{<-}(10.001875,2.55)(7.341875,2.55)
\psline[linewidth=0.04cm,arrowsize=0.05291667cm 2.0,arrowlength=1.4,arrowinset=0.4]{<-}(10.001875,-0.09)(7.341875,-0.09)
\psline[linewidth=0.04cm,arrowsize=0.05291667cm 2.0,arrowlength=1.4,arrowinset=0.4]{<-}(10.001875,-2.75)(7.341875,-2.75)
\psline[linewidth=0.04cm,linestyle=dashed,dash=0.16cm 0.16cm](1.691875,-3.13)(1.691875,-4.31)
\psline[linewidth=0.04cm,linestyle=dashed,dash=0.16cm 0.16cm](6.451875,-3.13)(6.451875,-4.31)
\psline[linewidth=0.04cm,linestyle=dashed,dash=0.16cm 0.16cm](10.851875,-3.13)(10.851875,-4.31)
\end{pspicture} 
}
\caption{An infinite {\sc af} with no semi-stable extensions.}
\label{fig:caminada}
\end{center}
\end{figure}

One obvious choice to describe this scheme would be the set $\Sigma~=~\set{A,B,C}$ so that
\[
\XC~=~\set{A}\cdot\set{A}^{*}~\cup~\set{B}\cdot\set{B}^{*}~\cup\set{C}\cdot\set{C}^{*}
\]
It is not too hard to see, however, that it is impossible to describe the required set of attacks via some $a\in\AC\EC(\set{A,B,C})$: for example
suppose $S$ is an infinite regular subset of $\set{A^{i}~:~i\geq1}$. Then
\[
\pi_{a}^{-}(S)~~=~\set{~B^{j}~:~j\geq \min\set{i~:~A^{i}\in S}}~\cup~S
\]
Now while this is a regular language for any \emph{fixed} subset of $\set{A^{i}~:~i\geq1}$ given that
it requires determining $\min\set{~i~:~A^{i}\in S}$ it is not possible to construct a \emph{general} expression
allowing this minimum to be computed.

Nevertheless this scheme can be described within our formalism. Let $\Sigma=\set{0}$ and $\XC~=~\set{0^i~:~i\geq1}~=~\set{0}\cdot\set{0}^{*}$.
We can partition $\XC$ into three sets, $L_A$, $L_B$ and $L_C$, as follows:
\[
\begin{array}{lclcl}
L_A&\mbox{ $=$ }&\set{0^{3i+1}~:~i\geq0}&\mbox{ $=$ }&\set{0}\cdot\set{000}^{*}\\
L_B&\mbox{ $=$ }&\set{0^{3i+2}~:~i\geq0}&\mbox{ $=$ }&\set{00}\cdot\set{000}^{*}\\
L_C&\mbox{ $=$ }&\set{0^{3i}~:~i\geq1}&\mbox{ $=$ }&\set{000}\cdot\set{000}^{*}
\end{array}
\]

We can now identify the attack expression: 
\begin{itemize}
\item each element of $L_A$ is attacked by itself, giving rise to the sub-expression $I\cap L_A$, and by all elements of $L_B$ with greater or equal index, giving rise to the sub-expression $(I\cap L_A)\cdot0\cdot(000)^{*}$;
\item each element of $L_B$ is attacked by the element of $L_C$ with the same index, giving rise to the sub-expression $(I\cap L_B)\cdot 0$, and by all elements of $L_B$ with greater index, giving rise to the sub-expression $(I\cap L_B)\cdot(000)\cdot(000)^{*}$;
\item each element of $L_C$ is attacked by the element of $L_B$ with the same index, giving rise to the sub-expression $tl(I\cap L_C)$.
\end{itemize}

Summing up, we get
\[
\begin{array}{lcl}
a & = & (I\cap L_A)~\cup~ (I\cap L_A)\cdot0\cdot(000)^{*}~~\cup\\
&&(I\cap L_B)\cdot(000)\cdot(000)^{*}~\cup~(I\cap L_B)\cdot0~~\cup\\
&& tl(I\cap L_C)
\end{array}
\]

\noindent 
giving directly

\[
\begin{array}{lcl}
\pi_{a}^{-}(S)&\mbox{ $=$ }&S\cap L_A~\cup~(S\cap L_A)\cdot0\cdot(000)^{*}~\cup\\
&&(S\cap L_B)\cdot(000)\cdot(000)^{*}~\cup~(S\cap L_B)\cdot0~\cup\\
&&tl(S\cap L_C)
\end{array}
\]

In order to compute $\pi_{a}^{+}(S)$, we need to take into account Fact \ref{conjectures-inverse}. From \ref{conjectures-inverse}.\ref{Prop1} the inverted mapping will be the union of the inverted sub-mappings corresponding to the various sub-expressions of $a$.
\begin{itemize}
\item The first term: $I\cap L_A$, according to \ref{conjectures-inverse}.\ref{Prop6} with $b=I$, gives rise to $\inv{I}(S \cap L_A)$, and thus, from \ref{conjectures-inverse}.\ref{Prop0}, to $S \cap L_A$.

\item As to the second term $(I\cap L_A)\cdot 0 \cdot (000)^{*}$, from \ref{conjectures-inverse}.\ref{Prop2} with $b =  (I\cap L_A)$ we obtain $\inv{b}(S / \set{0\cdot(000)^{*}})$. In turn, to obtain $\inv{b}$, \ref{conjectures-inverse}.\ref{Prop6} applies, which letting $c=I$ yields $\inv{c}(L_A \cap (S / \set{0 \cdot(000)^{*}}))$, which applying \ref{conjectures-inverse}.\ref{Prop0} yields $L_A \cap (S / \set{0 \cdot(000)^{*}})$.

\item Following the same steps, from $(I \cap L_B) \cdot (000) \cdot (000)^{*}$ we obtain $L_B \cap (S / \set{000 \cdot (000)^{*}})$, and from $(I \cap L_B) \cdot 0$ we derive $L_B \cap (S/0)$.

\item Finally from $tl(I \cap L_C)$, applying \ref{conjectures-inverse}.\ref{Prop8} with $b= I \cap L_C$ we obtain $\inv{b}(\Sigma \cdot S)$. Applying then \ref{conjectures-inverse}.\ref{Prop6} with $c=I$ we get $\inv{c}(L_C \cap \Sigma \cdot S)$, which applying \ref{conjectures-inverse}.\ref{Prop0} yields $L_C \cap \Sigma \cdot S$.

\end{itemize}

Putting things together (according to Fact \ref{conjectures-inverse}.\ref{Prop1}), we obtain

\[
\begin{array}{lcl}
\pi_{a}^{+}(S)&\mbox{ $=$ }&(S\cap L_A)~\cup~(L_A \cap( S / \set{0 \cdot (000)^{*}}))~\cup\\
&& (L_B \cap( S / \set{000 \cdot (000)^{*}})) \cup (L_B \cap( S / 0))~\cup\\ 
&& (L_C \cap (\Sigma \cdot S))
\end{array}
\]

We can now exemplify the use of the computational procedures\footnote{The reader is referred to the proof of Theorem \ref{thm:afs-ops} for the underlying details.} of Section \ref{sec:computing} in this case.
Let us start with the check of conflict-freeness, which for a set $S$ involves verifying whether $\pi_{a}^{-}(S)\cap S~=~\emptyset$.
From the formulation of $\pi_{a}^{+}(S)$ given above it is easily verifiable that:
\begin{itemize}
\item any set $S$ such that $S \cap L_A \neq \emptyset$ is not conflict-free;
\item any set $S$ such that $\card{S \cap L_B} \geq 2$ is not conflict-free;
\item any set $S \subseteq L_C$ is conflict-free.
\end{itemize}
As to conflict-freeness, leaving apart the empty set, we have therefore to consider only singletons of the form $S=\set{00 \cdot(000)^i}$ with a fixed $i \geq 0$, any set $S \subseteq L_C$, and the union of any of the singletons $\set{00 \cdot(000)^i}$ with any subset of $L_C \setminus \set{000 \cdot(000)^i }$. 
For admissibility of a set $S$, one has to verify whether $\pi_{a}^{+}(S) \supseteq \pi_{a}^{-}(S)$ in addition to conflict-freeness.
For the generic singleton $S=\set{00 \cdot(000)^i}$, we have $\pi_{a}^{+}(S) = \set{0 \cdot(000)^j} \cup \set{00 \cdot(000)^k} \cup \set{000 \cdot(000)^i } $ for $0 \leq j \leq i$ and $0 \leq k  < i$, while $\pi_{a}^{-}(S)= \set{00 \cdot(000)^{i+1} \cdot (000)^{*}} \cup \set{000 \cdot(000)^i }$.
Since  $\set{00 \cdot(000)^{i+1} \cdot (000)^{*}} \cap \pi_{a}^{+}(S) = \emptyset$, it turns out that $\pi_{a}^{+}(S) \nsupseteq \pi_{a}^{-}(S)$ hence no $S=\set{00 \cdot(000)^i}$ is admissible.
Considering instead any set $S \subseteq L_C$ it can be seen that $\pi_{a}^{+}(S) = L_B \cap (S / 0) = tl(S \cap L_C) = \pi_{a}^{-}(S)$.
Hence any subset of $L_C$ is admissible.
Considering now the union $S$ of a singleton $\set{00 \cdot(000)^i}$ and a subset of $L_C$, since $\pi_{a}^{-}(\set{00 \cdot(000)^i}) \supseteq \set{00 \cdot(000)^{i+1} \cdot (000)^{*}}$ then $\set{(000)^{i+2} \cdot (000)^{*}}$ must be contained in $S$.
Taking now into account the facts above, it turns out that $S= \set{00 \cdot(000)^i, (000)^{i+2} \cdot (000)^{*}} \cup L'$ for any $L' \subseteq \set{(000)^j | 1 \leq j \leq i}$.
To determine whether an admissible set $S\subseteq L_C$ is a complete extension we have to check whether $S = \FC(S) = \XC \setminus \pi_{a}^{+}(\XC \setminus \pi_{a}^{+}(S))$.
Considering any such $S \subseteq L_C$, we have already seen that $\pi_{a}^{+}(S) = L_B \cap (S / 0)$.
Hence $\XC \setminus \pi_{a}^{+}(S) = L_A \cup (L_B \setminus (S / 0)) \cup L_C $.
It follows that $\pi_{a}^{+}(\XC \setminus \pi_{a}^{+}(S)) = L_A \cup L_B \cup (L_C \setminus S)$, and hence $\FC(S) = \XC \setminus (L_A \cup L_B \cup (L_C \setminus S)) = S$.
As to the sets of the form $S= \set{00 \cdot(000)^i, (000)^{i+2} \cdot (000)^{*}} \cup L'$, it turns out (using the same reasoning line) that $\FC(S) = \set{00 \cdot (000)^i} \cup (L_C \setminus \set{000 \cdot (000)^i})$, which is also the unique complete extension including $\set{00 \cdot (000)^i}$.
Summing up, all (either finite or infinite) subsets of $L_C$ plus the sets $\set{00 \cdot (000)^i} \cup (L_C \setminus \set{000 \cdot (000)^i})$ with $i \geq 0$ form the set of all the complete extensions of this framework.

\subsection{Dung's example}\label{subsec:dungexamp}

\newcommand{\afdung}{\overline{AF}}

We will now consider the infinite argumentation framework introduced in \cite[p. 331, 352]{dung:1995}. 
The framework is derived from the following logic program\footnote{In the logic program $Q$, $x$ is any natural number, $s(x)$ denotes the successor of $x$ and $q(x)$ can be regarded as any property of all natural numbers.} $Q$. 

$\begin{array}{l l l}
Q:      & r \leftarrow \n p  & (r1)\\
        & p \leftarrow \n q(x)  & (r2)\\
        & q(x) \leftarrow even(x)  & (r3)\\
        & q(x) \leftarrow \n even(x)  & (r4)\\
        & even(s(x)) \leftarrow \n even(x)  & (r5)\\
        & even(0) \leftarrow & (r6)
\end{array}$

The rules for transforming a logic program into an {\sc AF} are defined in \cite[p. 343]{dung:1995} as follows.

First of all, for a logic program $P$, \grounded{P} denotes the \emph{set of all ground instances of clauses}
in $P$. For each literal $h$, the \emph{complement} of $h$ is denoted by \logicalcomplement{h}.
Let $K = \set{\n b_1, \ldots, \n b_m}$ be a set of ground negative literals. A ground atom $k$ is said to be
a \emph{defeasible consequence} of $P, K$ if there is a sequence of ground atoms $(e_0, e_1, \ldots, e_n)$ with $e_n = k$
such that for each $e_i$, either $e_i\leftarrow~ \in \grounded{P}$ or $e_i$ is the head of a
clause $e_i \leftarrow a_1, \ldots, a_t, \n a_{t+1}, \ldots, \n a_{t+r}$ in \grounded{P} such that the positive
literals $a_1, \ldots, a_t$ belong to the preceding members in the sequence and the negative literals $\n a_{t+1}, \ldots, \n a_{t+3}$
belong to $K$. $K$ is said to be a \emph{support for $k$ with respect to $P$}.

A logic program $P$ is transformed into an {\sc af} $AF(P) = \tuple{\XC_P, \AC_P}$ as follows:
\[
\begin{array}{rcl}
\XC_P&=&\mbox{ }~\set{(K, k) | K \textrm{ is a support for $k$ with respect to $P$}}\\
&&\cup\set{\darg{\n k}{\n k}|k \textrm{ is a ground atom}};\\
(K, h)\mbox{ attacks }(K', h')&\Leftrightarrow&\logicalcomplement{h} \in K'
\end{array}
\]

Arguments of the form $\darg{\n k}{\n k}$ capture the idea that $k$ would be concluded false if
there is no acceptable argument supporting $k$. An argument $a$ attacks an argument $b$ when the consequence of $a$ contradicts one of the members of the support of $b$.

The framework $AF(Q) = \tuple{\XC_{Q}, \AC_{Q}}$, derived from the logic program $Q$, turns out to be non finitary and is depicted in Figure \ref{fig:af-non-finitary}. Note that to keep the notation simple each predicate $s(\varx)$ has been replaced by the result of the expression $\varx + 1$.

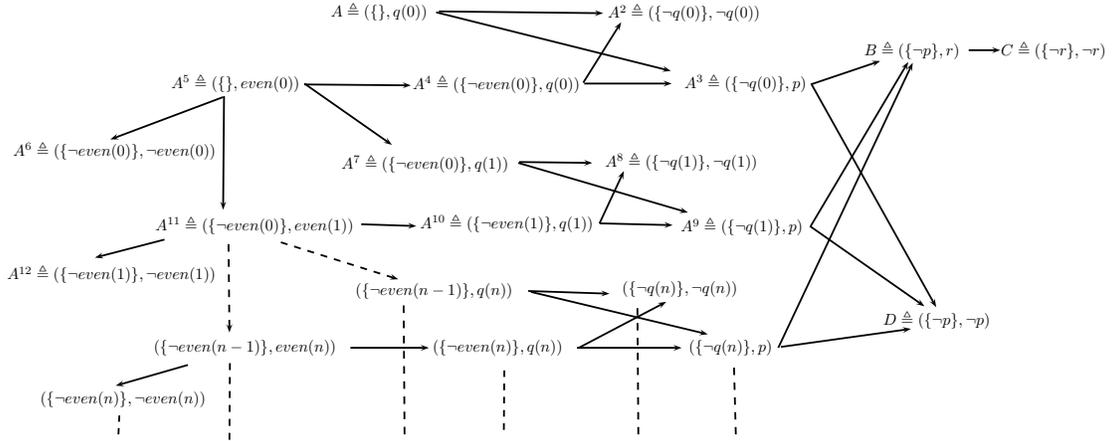
\begin{figure}[!htb]
\begin{center}
\noindent\makebox[\textwidth]{%
\scalebox{0.6} 
{
\begin{pspicture}(0,-4.8509374)(29.047188,4.8709373)
\usefont{T1}{ppl}{m}{n}
\rput(7.4954686,3.0790625){$\cone^5 \name \darg{}{even(0)}$}
\usefont{T1}{ppl}{m}{n}
\rput(10.675468,4.6590624){$\cone \name \darg{}{q(0)}$}
\usefont{T1}{ppl}{m}{n}
\rput(7.9154687,-0.0609375){$\cone^{11} \name \darg{\lnot even(0)}{even(1)}$}
\usefont{T1}{ppl}{m}{n}
\rput(7.7054687,-2.8009374){$\darg{\lnot even(n-1)}{even(n)}$}
\usefont{T1}{ppl}{m}{n}
\rput(4.815469,1.5790625){$\cone^6 \name \darg{\lnot even(0)}{\lnot even(0)}$}
\usefont{T1}{ppl}{m}{n}
\rput(4.7454686,-1.1609375){$\cone^{12} \name \darg{\lnot even(1)}{\lnot even(1)}$}
\usefont{T1}{ppl}{m}{n}
\rput(5.005469,-3.9409375){$\darg{\lnot even(n)}{\lnot even(n)}$}
\usefont{T1}{ppl}{m}{n}
\rput(13.275469,3.0590625){$\cone^4 \name \darg{\lnot even(0)}{q(0)}$}
\usefont{T1}{ppl}{m}{n}
\rput(17.435469,4.6390624){$\cone^2 \name \darg{\lnot q(0)}{\lnot q(0)}$}
\usefont{T1}{ppl}{m}{n}
\rput(18.79547,3.0790625){$\cone^3 \name \darg{\lnot q(0)}{p}$}
\usefont{T1}{ppl}{m}{n}
\rput(11.695469,1.2990625){$\cone^7 \name \darg{\lnot even(0)}{q(1)}$}
\usefont{T1}{ppl}{m}{n}
\rput(13.505468,-0.0209375){$\cone^{10} \name \darg{\lnot even(1)}{q(1)}$}
\usefont{T1}{ppl}{m}{n}
\rput(17.37547,1.3390625){$\cone^8 \name \darg{\lnot q(1)}{\lnot q(1)}$}
\usefont{T1}{ppl}{m}{n}
\rput(18.71547,-0.1009375){$\cone^9 \name \darg{\lnot q(1)}{p}$}
\usefont{T1}{ppl}{m}{n}
\rput(22.485468,3.7990625){$\ctwo \name \darg{\lnot p}{r}$}
\usefont{T1}{ppl}{m}{n}
\rput(25.62547,3.7990625){$\cthree \name \darg{\lnot r}{\lnot r}$}
\usefont{T1}{ppl}{m}{n}
\rput(23.03547,-2.1609375){$\cfour \name \darg{\lnot p}{\lnot p}$}
\usefont{T1}{ppl}{m}{n}
\rput(11.895469,-1.5409375){$\darg{\lnot even(n-1)}{q(n)}$}
\usefont{T1}{ppl}{m}{n}
\rput(13.305469,-2.7809374){$\darg{\lnot even(n)}{q(n)}$}
\usefont{T1}{ppl}{m}{n}
\rput(17.345469,-1.5009375){$\darg{\lnot q(n)}{\lnot q(n)}$}
\usefont{T1}{ppl}{m}{n}
\rput(18.46547,-2.8009374){$\darg{\lnot q(n)}{p}$}
\psline[linewidth=0.04cm,arrowsize=0.05291667cm 2.0,arrowlength=1.4,arrowinset=0.4]{->}(7.241875,2.7890625)(7.221875,0.2690625)
\psline[linewidth=0.04cm,linestyle=dashed,dash=0.16cm 0.16cm,arrowsize=0.05291667cm 2.0,arrowlength=1.4,arrowinset=0.4]{->}(7.341875,-0.4909375)(7.361875,-2.4509375)
\psline[linewidth=0.04cm,arrowsize=0.05291667cm 2.0,arrowlength=1.4,arrowinset=0.4]{->}(7.221875,2.7690625)(4.721875,1.8290625)
\psline[linewidth=0.04cm,arrowsize=0.05291667cm 2.0,arrowlength=1.4,arrowinset=0.4]{->}(9.021875,3.0690625)(10.961875,1.7090625)
\psline[linewidth=0.04cm,arrowsize=0.05291667cm 2.0,arrowlength=1.4,arrowinset=0.4]{->}(5.921875,-0.3909375)(4.381875,-0.7909375)
\psline[linewidth=0.04cm,arrowsize=0.05291667cm 2.0,arrowlength=1.4,arrowinset=0.4]{->}(6.441875,-3.1709375)(4.841875,-3.6309376)
\psline[linewidth=0.04cm,arrowsize=0.05291667cm 2.0,arrowlength=1.4,arrowinset=0.4]{->}(23.741875,3.8090625)(24.441875,3.8090625)
\psline[linewidth=0.04cm,linestyle=dashed,dash=0.16cm 0.16cm](7.3723435,-3.1309376)(7.361875,-4.8309374)
\psline[linewidth=0.04cm,linestyle=dashed,dash=0.16cm 0.16cm](18.541876,-3.2309375)(18.581875,-4.7109375)
\psline[linewidth=0.04cm,linestyle=dashed,dash=0.16cm 0.16cm](13.452344,-3.2909374)(13.441875,-4.6109376)
\psline[linewidth=0.04cm,linestyle=dashed,dash=0.16cm 0.16cm,arrowsize=0.05291667cm 2.0,arrowlength=1.4,arrowinset=0.4]{->}(8.501875,-0.4509375)(11.101875,-1.2709374)
\psline[linewidth=0.04cm,arrowsize=0.05291667cm 2.0,arrowlength=1.4,arrowinset=0.4]{->}(9.021875,3.0490625)(11.381875,3.0290625)
\psline[linewidth=0.04cm,arrowsize=0.05291667cm 2.0,arrowlength=1.4,arrowinset=0.4]{->}(10.281875,-0.0509375)(11.501875,-0.0909375)
\psline[linewidth=0.04cm,arrowsize=0.05291667cm 2.0,arrowlength=1.4,arrowinset=0.4]{->}(10.041875,-2.7909374)(11.781875,-2.7909374)
\psline[linewidth=0.04cm,arrowsize=0.05291667cm 2.0,arrowlength=1.4,arrowinset=0.4]{->}(12.001875,4.6690626)(15.641875,4.6290627)
\psline[linewidth=0.04cm,arrowsize=0.05291667cm 2.0,arrowlength=1.4,arrowinset=0.4]{->}(11.941875,4.6490626)(17.141874,3.3490624)
\psline[linewidth=0.04cm,arrowsize=0.05291667cm 2.0,arrowlength=1.4,arrowinset=0.4]{->}(15.221875,3.0690625)(17.161875,3.0690625)
\psline[linewidth=0.04cm,arrowsize=0.05291667cm 2.0,arrowlength=1.4,arrowinset=0.4]{->}(15.201875,3.0890625)(16.041876,4.4290624)
\psline[linewidth=0.04cm,arrowsize=0.05291667cm 2.0,arrowlength=1.4,arrowinset=0.4]{->}(13.781875,1.3290625)(15.401875,1.3090625)
\psline[linewidth=0.04cm,arrowsize=0.05291667cm 2.0,arrowlength=1.4,arrowinset=0.4]{->}(13.761875,1.3090625)(17.521875,0.2090625)
\psline[linewidth=0.04cm,arrowsize=0.05291667cm 2.0,arrowlength=1.4,arrowinset=0.4]{->}(15.561875,-0.0309375)(17.181875,-0.0709375)
\psline[linewidth=0.04cm,arrowsize=0.05291667cm 2.0,arrowlength=1.4,arrowinset=0.4]{->}(15.561875,-0.0309375)(16.101875,1.1290625)
\psline[linewidth=0.04cm,arrowsize=0.05291667cm 2.0,arrowlength=1.4,arrowinset=0.4]{->}(13.981875,-1.5309376)(15.781875,-1.5909375)
\psline[linewidth=0.04cm,arrowsize=0.05291667cm 2.0,arrowlength=1.4,arrowinset=0.4]{->}(14.041875,-1.5509375)(17.961876,-2.4909375)
\psline[linewidth=0.04cm,arrowsize=0.05291667cm 2.0,arrowlength=1.4,arrowinset=0.4]{->}(15.061875,-2.7909374)(17.381874,-2.7909374)
\psline[linewidth=0.04cm,arrowsize=0.05291667cm 2.0,arrowlength=1.4,arrowinset=0.4]{->}(15.101875,-2.7909374)(17.041876,-1.7709374)
\psline[linewidth=0.04cm,arrowsize=0.05291667cm 2.0,arrowlength=1.4,arrowinset=0.4]{->}(20.241875,3.0490625)(21.781876,3.5690625)
\psline[linewidth=0.04cm,arrowsize=0.05291667cm 2.0,arrowlength=1.4,arrowinset=0.4]{->}(20.261875,3.0490625)(23.021875,-1.8909374)
\psline[linewidth=0.04cm,arrowsize=0.05291667cm 2.0,arrowlength=1.4,arrowinset=0.4]{->}(20.241875,-0.0909375)(22.401875,3.5490625)
\psline[linewidth=0.04cm,arrowsize=0.05291667cm 2.0,arrowlength=1.4,arrowinset=0.4]{->}(20.221874,-0.0909375)(22.761875,-1.8709375)
\psline[linewidth=0.04cm,arrowsize=0.05291667cm 2.0,arrowlength=1.4,arrowinset=0.4]{->}(19.521875,-2.7909374)(22.501875,3.5290625)
\psline[linewidth=0.04cm,arrowsize=0.05291667cm 2.0,arrowlength=1.4,arrowinset=0.4]{->}(19.581875,-2.7709374)(22.461876,-2.3709376)
\psline[linewidth=0.04cm,linestyle=dashed,dash=0.16cm 0.16cm](16.401875,-1.9109375)(16.421875,-4.7109375)
\psline[linewidth=0.04cm,linestyle=dashed,dash=0.16cm 0.16cm](4.921875,-4.2909374)(4.901875,-4.7109375)
\psline[linewidth=0.04cm,linestyle=dashed,dash=0.16cm 0.16cm](11.221875,-1.8509375)(11.241875,-4.7109375)
\end{pspicture} 
}}
\caption{Graphical representation of $AF(Q)$.}
\label{fig:af-non-finitary}
\end{center}
\end{figure}

Let us examine the elements of $\XC_{Q}$.
First, there are arguments of the form $\darg{\n k}{\n k}$ for each ground atom $k$, namely:
\begin{itemize}
\item two \virg{single} arguments for the atoms $p$ and $r$ (at the right of the figure, respectively in the lower and higher part)
\item an infinite sequence of arguments for the atoms with form $q(\varx)$, corresponding to the fifth (from the left) \virg{column} in Figure \ref{fig:af-non-finitary}.
\item an infinite sequence of arguments for the atoms with form $even(\varx)$, corresponding to the first (from the left) \virg{column} in Figure \ref{fig:af-non-finitary}
\end{itemize}

Then, there are arguments of the form $(K, k)$ derived by applying rules (r1)-(r6), where $K$ is a support for $k$.
In particular we have:
\begin{itemize}
\item two arguments with empty support: from (r6) we get $\darg{}{even(0)}$ (top of the second \virg{column}) and from (r6) and (r3) we get $\darg{}{q(0)}$ (above third \virg{column} in Figure \ref{fig:af-non-finitary})
\item a \virg{single} argument $\darg{\n p}{r}$ from (r1) (at the right of the figure, in the higher part)
\item an infinite sequence of arguments of the form $\darg{\n q(\varx)}{p} | \varx \ge 0$ from (r2) (sixth \virg{column})
\item an infinite sequence of arguments of the form $\darg{\n even(\varx)}{q(\varx)} | \varx \ge 0$ from (r4) (fourth \virg{column})
\item an infinite sequence of arguments of the form $\darg{\n even(\varx)}{even(\varx + 1)} | \varx \ge 0$ from (r5) (second \virg{column})
\item an infinite sequence of arguments of the form $\darg{\n even(\varx)}{q(\varx + 1)} | \varx \ge 0$ from (r5) and (r3) (third \virg{column})

\end{itemize}

Turning to the attack relation $\AC_{Q}$, we observe that:
\begin{itemize}

\item each argument in the second \virg{column}, (with consequence $even(\varx)$) attacks the corresponding arguments in the first, third and fourth \virg{columns} (having support $\n even(\varx)$)
\item each argument in the second \virg{column} also attacks its \virg{successor} in the column
\item each argument in the third \virg{column}, (with consequence $q(\varx)$) attacks the corresponding arguments in the fifth and sixth \virg{columns} (having support $\n q(\varx)$)
\item similarly, each argument in the fourth \virg{column}, (with consequence $q(\varx)$) attacks the corresponding arguments in the fifth and sixth \virg{columns} (having support $\n q(\varx)$)
\item each argument in the sixth \virg{column} (with consequence $p$) attacks both arguments $\darg{\n p}{\n p}$ and $\darg{\n p}{r}$
\item argument $\darg{\n p}{r}$ attacks $\darg{\n r}{\n r}$
\end{itemize}

In order to provide an {\sc af} specification for $AF(Q)$ we need first a {\sc dfa} representing the infinite set of arguments $\XC_{Q}$ and then a proper attack expression representing the relation $\AC_{Q}$.

As to the {\sc dfa} representation of $\XC_{Q}$, it is handy to consider separately the infinite sequences corresponding to the six \virg{columns} in Figure \ref{fig:af-non-finitary} and the three \virg{single} arguments $\darg{\n p}{r}$, $\darg{\n r}{\n r}$, and $\darg{\n p}{\n p}$.

As to the arguments included in the \virg{columns}, we define a correspondence with sequences of a unique symbol, namely $\cone$, by exploiting the \virg{regular} structure of the sequences of arguments.
In fact, let us associate the string $\cone$ with argument $\darg{}{q(0)}$,
$\cone\cone$ with argument $\darg{\n q(0)}{\n q(0)}$, $\cone\cone\cone$ with argument $\darg{\n q(0)}{p}$, $\cone\cone\cone\cone$ with argument $\darg{\n even(0)}{q(0)}$, $\cone^5$ with argument $\darg{}{even(0)}$, 
and $\cone^6$ with argument $\darg{\n even(0)}{\n even(0)}$. The association may then continue periodically over the \virg{columns} by putting $\cone^7$ in correspondence with $\darg{\n even(0)}{q(1)}$, $\cone^8$ with $\darg{\n q(1)}{\n q(1)}$ and so on.
More formally, we use the elements of $\cone \cdot \cone^{*}$ to represent arguments as follows:
\[
\begin{array}{l}
\cone \name \darg{}{q(0)} \\
\forall \varx \ge 0, \cone^{2 + 6\varx} \name \darg{\n q(\varx)}{\n q(\varx)}\\
\forall \varx \ge 0, \cone^{3 + 6\varx} \name \darg{\n q(\varx)}{p}\\
\forall \varx \ge 0, \cone^{4 + 6\varx} \name \darg{\n even(\varx)}{q(\varx)}\\
\cone^5 \name \darg{}{even(0)}\\
\forall \varx \ge 0, \cone^{6 + 6\varx} \name \darg{\n even(\varx)}{\n even(\varx)}\\
\forall \varx \ge 0, \cone^{7 + 6\varx} \name \darg{\n even(\varx)}{q(\varx + 1)}\\
\forall \varx \ge 0, \cone^{11 + 6\varx} \name \darg{\n even(\varx)}{even(\varx + 1)}
\end{array}
\]

To simplify the subsequent description it is useful to denote as six distinct (regular) languages the subsets of $\cone^{*}$ corresponding to the different sequences of arguments in $\XC_{Q}$:
\[
\begin{array}{l}
L_{2} \name \set{\cone^{2 + 6\varx} | \varx \ge 0}\\
L_{3} \name \set{\cone^{3 + 6\varx} | \varx \ge 0}\\
L_{4} \name \set{\cone^{4 + 6 \varx} | \varx \ge 0}\\
L_{6} \name \set{\cone^{6 + 6\varx} | \varx \ge 0}\\
L_{7} \name \set{\cone^{7 + 6 \varx} | \varx \ge 0}\\

L_{11} \name \set{\cone^{11+ 6 \varx} | \varx \ge 0}
\end{array}
\]

To complete the representation of $\XC_{Q}$, the remaining three \virg{single} arguments are assigned three distinct alphabet elements as follows:
\begin{itemize}
\item $\ctwo \name \darg{\n p}{r}$
\item $\cthree \name \darg{\n r}{\n r}$
\item $\cfour \name \darg{\n p}{\n p}$
\end{itemize}

In summary, the representation of the arguments in $AF(Q)$ is based on the set of symbols $\Sigma_Q = \set{\cone, \ctwo, \cthree, \cfour}$ and the encoding consists of a {\sc DFA} $\MC_Q$ accepting the language $L(\MC_Q) = \set{\ctwo, \cthree, \cfour} \cup \set{\cone \cdot  \cone^{*}}$.
Clearly $L(\MC_Q)$ is a regular language included in $\Sigma_Q^{*} \setminus \set{\varepsilon}$.

The simple minimal {\sc DFA} accepting $L(\MC_Q)$ is depicted in Figure \ref{fig:af-non-finitary-dfa-minimal}.

\begin{figure}[htb]
\begin{center}
\scalebox{0.5} 
{
\begin{pspicture}[linewidth=1bp](0bp,0bp)(216.31bp,217.08bp)

  \pstVerb{2 setlinejoin} 
\psset{linecolor=black}
  \psbezier[arrows=->](130.78bp,108.96bp)(134.88bp,109.59bp)(139.2bp,110.26bp)(153.54bp,112.47bp)
  \psset{linecolor=[rgb]{0.0,0.0,0.0}}
  \rput(137.15bp,105.94bp){$\cfour$}
  \psset{linecolor=black}
  \psbezier[arrows=->](108.35bp,131.56bp)(109.01bp,135.66bp)(109.71bp,139.97bp)(112.01bp,154.31bp)
  \psset{linecolor=[rgb]{0.0,0.0,0.0}}
  \rput(113.38bp,137.93bp){$\cthree$}
  \psset{linecolor=black}
  \psbezier[arrows=->](62.625bp,149.25bp)(72.833bp,149.25bp)(81bp,147bp)(81bp,142.5bp)(81bp,139.69bp)(77.81bp,137.76bp)(62.625bp,135.75bp)
  \psset{linecolor=[rgb]{0.0,0.0,0.0}}
  \rput(84bp,142.5bp){$\cone$}
  \psset{linecolor=black}
  \psbezier[arrows=->](64.428bp,66.559bp)(68.432bp,70.427bp)(72.736bp,74.584bp)(84.474bp,85.924bp)
  \psbezier[arrows=->](80.062bp,117.38bp)(76.378bp,119.31bp)(72.504bp,121.33bp)(59.627bp,128.06bp)
  \psset{linecolor=[rgb]{0.0,0.0,0.0}}
  \rput(72.341bp,116.37bp){$\cone$}
  \psset{linecolor=black}
  \psbezier[arrows=->](116bp,80.585bp)(117.83bp,76.861bp)(119.76bp,72.944bp)(126.17bp,59.927bp)
  \psset{linecolor=[rgb]{0.0,0.0,0.0}}
  \rput(114.85bp,72.802bp){$\ctwo$}
{%
  \psset{linecolor=[rgb]{0.0,0.0,0.0}}
  \psellipse[](140bp,32bp)(27bp,27bp)
  \psellipse[](140bp,32bp)(31bp,31bp)
  \rput(139.92bp,32bp){$q_1$}
}%
{%
  \psset{linecolor=[rgb]{0.0,0.0,0.0}}
  \psellipse[](104bp,105bp)(27bp,27bp)
  \rput(104.06bp,104.84bp){$q_0$}
}%
{%
  \psset{linecolor=[rgb]{0.0,0.0,0.0}}
  \psellipse[](184bp,117bp)(27bp,27bp)
  \psellipse[](184bp,117bp)(31bp,31bp)
  \rput(184.31bp,117.21bp){$q_3$}
}%
{%
  \psset{linecolor=[rgb]{0.0,0.0,0.0}}
  \psellipse[](117bp,185bp)(27bp,27bp)
  \psellipse[](117bp,185bp)(31bp,31bp)
  \rput(116.96bp,185.08bp){$q_2$}
}%
{%
  \psset{linecolor=[rgb]{0.0,0.0,0.0}}
  \psellipse[](32bp,143bp)(27bp,27bp)
  \psellipse[](32bp,143bp)(31bp,31bp)
  \rput(32bp,142.5bp){$q_4$}
}%
\end{pspicture}
}
\caption{Graphical representation of the minimal DFA describing the argument encoding of $AF(Q)$.}
\label{fig:af-non-finitary-dfa-minimal}
\end{center}
\end{figure}
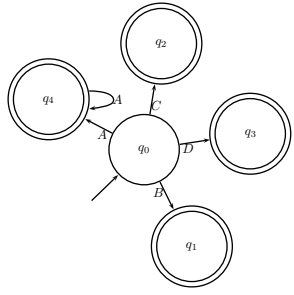

We need now to identify a suitable attack expression $a^Q$ for $AF(Q)$.
To this purpose let us first observe that, for a set of arguments $S$, the function $\pi_{a^Q}^{-}(S)$ must return as result a non-empty set if and only if $S$ has a non empty intersection with the set of arguments which receive an attack in $AF(Q)$, namely $\set{\ctwo, \cthree, \cfour} \cup L_2 \cup L_3 \cup L_4 \cup L_6 \cup L_7 \cup L_{11}$.
As a consequence, the global attack expression may be built as the union of various sub-expressions, each associated with a class of attacked arguments. Each sub-expression:
\begin{itemize}
\item has to select the range of sets for which a non-empty result is returned: this can be achieved specifying the intersection between $I$ and a given sub-language of $\XC_Q$ corresponding to the class of attacked arguments;
\item has to define a set of attackers through a proper symbol manipulation.
\end{itemize}

To exemplify, consider a sub-expression to specify the attackers of argument $\ctwo$, namely $\set{\cone^{3+6\varx} | \varx \ge 0}$ (corresponding to $\set{\darg{\n q(\varx)}{p}}$, $\varx \geq 0$). This can be obtained as 
$tl(I \cap \ctwo) \cdot \cone\cone\cone \cdot (\wseq{\cone}{6})^{*} = tl(I \cap \ctwo) \cdot L_3$.
In fact, $I \cap \ctwo$ yields either $\set{\ctwo}$ or the empty set. In the former case the $tl$ operator yields $\varepsilon$ which, combined with $L_3$, gives the desired result, while in the latter case the combination produces the empty set. 
In a similar (and simpler) way, since $\cthree$ is attacked only by $\ctwo$ we obtain the subexpression $tl(I \cap \cthree) \cdot \ctwo$, while $\cfour$ has exactly the same attackers as $\ctwo$, yielding $tl(I \cap \cfour) \cdot L_3$.

Consider now the specification of the attackers of the elements of sub-languages of $\XC_Q$. 
Starting from $L_2$ we observe that each element $\cone^{2 + 6\varx}$ (corresponding to $\darg{\n q(\varx)}{\n q(\varx)}$, $\varx \geq 0$) has two attackers namely $\cone^{2 + 6\varx - 1}$ (corresponding to $\darg{}{q(0)}$ for $\varx=0$ and to $\darg{\n even(\varx -1)}{q(\varx)}$ for $\varx > 0$) and $\cone^{2 + 6\varx + 2}$ (corresponding to $\darg{\n even(\varx)}{q(\varx)}$, $\varx \geq 0$).
The elements of the first family of attackers can be obtained by applying the $tl$ operator to $L_2$, those of the second family of attackers by concatenating $\cone \cdot \cone$ to $L_2$.
This reasoning gives rise to the sub-expressions $tl(I \cap L_2)$ and $(I \cap L_2) \cdot \cone \cdot \cone$.

Similarly, each element of $L_3$ $\cone^{3 + 6\varx}$ (corresponding to $\darg{\n q(\varx)}{p}$, $\varx \geq 0$) has two attackers namely $\cone^{3 + 6\varx + 1}$ (corresponding to $\darg{\n even(\varx)}{q(\varx)}$, $\varx \geq 0$) and $\cone^{3 + 6\varx - 2}$ (corresponding to $\darg{}{q(0)}$ for $\varx=0$ and to $\darg{\n even(\varx - 1)}{q(\varx)}$ for $\varx \geq 1$).
This reasoning gives rise to the sub-expressions\footnote{In the following, in order to simplify notation, we denote as $tl^n(w)$
the $n^{th}$ application of $tl(\cdot)$ to the word $w$, namely $tl(tl(\ldots(tl(w))))$.}  $(I \cap L_3) \cdot \cone$ and $tl^2(I \cap L_3)$.

As to $L_4$, each element $\cone^{4 + 6\varx}$ (corresponding to $\darg{\n even(\varx)}{q(\varx)}$, $\varx \geq 0$) has one attacker $\cone^{4 + 6\varx + 1}$ (corresponding to $\darg{}{even(0)}$ for $\varx = 0$ and to $\darg{\n even(\varx - 1)}{even(\varx)}$, $\varx \geq 1$).
This gives rise to the sub-expression $(I \cap L_4) \cdot \cone$.

Turning to $L_6$, each element $\cone^{6 + 6\varx}$ (corresponding to $\darg{\n even(\varx)}{\n even(\varx)}$, $\varx \geq 0$) has one attacker $\cone^{6 + 6\varx - 1}$ (corresponding to $\darg{}{even(0)}$ for $\varx = 0$ and to $\darg{\n even(\varx - 1)}{even(\varx)}$, $\varx \geq 1$).
This gives rise to the sub-expression $tl(I \cap L_6) $.

In $L_7$, each element $\cone^{7 + 6\varx}$ (corresponding to $\darg{\n even(\varx)}{q(\varx + 1)}$, $\varx \geq 0$) has one attacker $\cone^{7 + 6\varx - 2}$ (corresponding to $\darg{}{even(0)}$ for $\varx = 0$ and to $\darg{\n even(\varx - 1)}{even(\varx)}$, $\varx \geq 1$).
This gives rise to the sub-expression $tl^2(I \cap L_7)$.

Finally, each element of $L_{11}$ $\cone^{11 + 6\varx}$,(corresponding to $\darg{\n even(\varx)}{even(\varx + 1)}$, $\varx \geq 0$) has one attacker $\cone^{11 + 6\varx - 6}$ (corresponding to $\darg{}{even(0)}$ for $\varx = 0$ and to $\darg{\n even(\varx - 1)}{even(\varx)}$, $\varx \geq 1$).
This gives rise to the sub-expression $tl^6(I \cap L_{11})$.

In summary, we obtain the following attack expression:

{\allowdisplaybreaks
\noindent
\begin{eqnarray*}
a^Q & \name & tl(I \cap \ctwo) \cdot L_3~\cup \\
    &       & tl(I \cap \cthree) \cdot \ctwo~\cup \\
    &       & tl(I \cap \cfour) \cdot L_3~\cup \\
    &       & tl(I \cap L_2)~\cup \\
    &       & (I \cap L_2) \cdot \cone \cdot \cone~\cup \\
    &       & tl^2(I \cap L_3)~\cup \\
    &       & (I \cap L_3) \cdot \cone~\cup \\
    &       & (I \cap L_4) \cdot \cone~\cup \\
    &       & tl(I \cap L_6)~\cup \\
    &       & tl^2(I \cap L_7)~\cup \\
    &       & tl^6(I \cap L_{11})\\
\end{eqnarray*}
}

The relevant mapping $\underline{a^Q}: 2^{\Sigma_Q^{*}} \rightarrow 2^{\Sigma_Q^{*}}$ follows directly:

{\allowdisplaybreaks
\noindent
\begin{eqnarray*}
\underline{a^Q}(S) & =  & tl(S \cap \ctwo) \cdot L_3~\cup \\
                   &        & tl(S \cap \cthree) \cdot \ctwo~\cup \\
                   &        & tl(S \cap \cfour) \cdot L_3~\cup \\
                   &        & tl(S \cap L_2)~\cup \\
                   &        & (S \cap L_2) \cdot \cone \cdot \cone~\cup \\
                   &        & tl^2(S \cap L_3)~\cup \\
                   &        & (S \cap L_3) \cdot \cone~\cup \\
                   &        & (S \cap L_4) \cdot \cone~\cup \\
                   &        & tl(S \cap L_6)~\cup \\
                   &        & tl^2(S \cap L_7)~\cup \\
                   &        & tl^6(S \cap L_{11})\\
\end{eqnarray*}
}

We can now apply 
Fact \ref{conjectures-inverse} to obtain the inverted mapping $\underline{a^{Q}}^{+}$.
First, we observe that on the basis of \ref{conjectures-inverse}.\ref{Prop1} 
the inverted mapping will be the union of the inverted sub-mappings corresponding to the various sub-expressions of $a^{Q}$.

Consider first, for the sake of illustration, the sub-expression $tl(I \cap \ctwo) \cdot L_3$, which has the form $b \cdot K_{\Sigma}$ with $b = tl(I \cap \ctwo)$ and $K_{\Sigma} = L_3$. 
Accordingly, Fact \ref{conjectures-inverse}.\ref{Prop2} 
applies, yielding $\underline{b}^{+}(S/L_3)$. In turn, to obtain $\underline{b}^{+}$, Fact \ref{conjectures-inverse}.\ref{Prop8} applies which, letting $c= I \cap \ctwo$, gives $\underline{c}^{+}(\Sigma_Q \cdot (S/L_3))$. Applying Fact \ref{conjectures-inverse}.\ref{Prop6} (and the base case for $I$) to $c$ we obtain $\ctwo \cap (\Sigma_Q \cdot (S/L_3))$.

The sub-expressions $tl(I \cap \cthree) \cdot \ctwo$ and $tl(I \cap \cfour) \cdot L_3$ are analogous, yielding $\cthree \cap (\Sigma_Q \cdot (S / \ctwo))$ and $\cfour \cap (\Sigma_Q \cdot (S/L_3))$.

From the sub-expression $tl(I \cap L_2)$ orderly applying Fact \ref{conjectures-inverse}.\ref{Prop8} and \ref{conjectures-inverse}.\ref{Prop6} we obtain $L_2 \cap (\Sigma_Q \cdot S)$, while from $(I \cap L_2) \cdot \cone \cdot \cone$ applying \ref{conjectures-inverse}.\ref{Prop2} and \ref{conjectures-inverse}.\ref{Prop6} we have $L_2 \cap (S / (\cone \cdot \cone))$. 

For the sub-expression $tl^2(I \cap L_3)$ we apply \ref{conjectures-inverse}.\ref{Prop8} twice and \ref{conjectures-inverse}.\ref{Prop6}, yielding $L_3 \cap (\Sigma_Q \cdot \Sigma_Q \cdot S)$.

The treatment of each of the remaining sub-expressions is similar to one of the previous cases, yielding the following result.
 
{\allowdisplaybreaks
\noindent
\begin{eqnarray*}
\underline{a^{Q}}^{+}(S) & = & \ctwo \cap (\Sigma_Q \cdot (S/L_3))~\cup \\
 &  & \cthree \cap (\Sigma_Q \cdot (S / \ctwo))~\cup \\
 &  & \cfour \cap (\Sigma_Q \cdot (S/L_3))~\cup \\
 &  & L_2 \cap (\Sigma_Q \cdot S)~\cup \\
 &  & L_2 \cap (S / (\cone \cdot \cone))~\cup \\
 &  & L_3 \cap (\Sigma_Q \cdot \Sigma_Q \cdot S)~\cup \\
 &  & L_3 \cap (S / \cone)~\cup \\
 &  & L_4 \cap (S / \cone)~\cup \\
 &  & L_6 \cap (\Sigma_Q \cdot S)~\cup \\
 &  & L_7 \cap (\Sigma_Q \cdot \Sigma_Q \cdot S)~\cup \\
 &  & L_{11} \cap (\Sigma_Q \cdot \Sigma_Q \cdot \Sigma_Q \cdot \Sigma_Q \cdot \Sigma_Q \cdot \Sigma_Q \cdot S)
\end{eqnarray*}
}

It can be easily observed that both $\underline{a^Q}$ and $\underline{a^{Q}}^{+}$ can not produce results outside $\XC_Q$ hence, for any $S \subseteq \XC_Q$,  $\pi_{a^{Q}}^{-}(S)=\underline{a^Q}(S) \cap \XC_Q = \underline{a^Q}(S)$ and $\pi_{a^{Q}}^{+}(S)= \underline{a^{Q}}^{+}(S) \cap \XC_Q = \underline{a^{Q}}^{+}(S)$.

We can now exemplify the analysis of semantics properties in $AF(Q)$.

Letting $S = \set{\ctwo} \cup \set{\cfour} \cup \set{\cone^{1 + 12\varx} | \varx \ge 0} \cup \set{\cone^{5 + 12\varx} | \varx \ge 0} \cup \set{\cone^{10 + 12\varx} | \varx \ge 0} \cup \set{\cone^{12 + 12\varx} | \varx \ge 0}$ consider the problem of checking whether $S$ is conflict-free (the structure of the infinite set $S$ is evidenced in Figure \ref{fig:af-non-finitary-ev}).

\begin{figure}[!htb]
\begin{center}
\noindent\makebox[\textwidth]{%
\scalebox{0.6} 
{
\begin{pspicture}(0,-5.07)(29.047188,5.05)
\definecolor{color1322b}{rgb}{0.8274509803921568,0.8274509803921568,0.8274509803921568}
\definecolor{color1322}{rgb}{0.996078431372549,0.996078431372549,0.996078431372549}
\definecolor{color883b}{rgb}{0.7411764705882353,0.7411764705882353,0.7411764705882353}
\definecolor{color883}{rgb}{0.803921568627451,0.803921568627451,0.803921568627451}
\psellipse[linewidth=0.04,linecolor=color1322,dimen=outer,fillstyle=solid,fillcolor=color1322b](11.831875,-1.76)(2.11,0.59)
\psellipse[linewidth=0.04,linecolor=color1322,dimen=outer,fillstyle=solid,fillcolor=color1322b](7.581875,-3.04)(2.48,0.59)
\psellipse[linewidth=0.04,linecolor=color883,dimen=outer,shadow=true,fillstyle=solid,fillcolor=color883b](4.661875,-1.4)(2.48,0.59)
\psellipse[linewidth=0.04,linecolor=color883,dimen=outer,shadow=true,fillstyle=solid,fillcolor=color883b](13.481875,-0.26)(2.02,0.59)
\psellipse[linewidth=0.04,linecolor=color883,dimen=outer,shadow=true,fillstyle=solid,fillcolor=color883b](7.391875,2.86)(1.59,0.59)
\psellipse[linewidth=0.04,linecolor=color883,dimen=outer,shadow=true,fillstyle=solid,fillcolor=color883b](10.631875,4.46)(1.25,0.59)
\psellipse[linewidth=0.04,linecolor=color883,dimen=outer,shadow=true,fillstyle=solid,fillcolor=color883b](23.011875,-2.44)(1.41,0.59)
\psellipse[linewidth=0.04,linecolor=color883,dimen=outer,shadow=true,fillstyle=solid,fillcolor=color883b](22.451876,3.58)(1.31,0.59)
\usefont{T1}{ppl}{m}{n}
\rput(7.4954686,2.86){$\cone^5 \name \darg{}{even(0)}$}
\usefont{T1}{ppl}{m}{n}
\rput(10.675468,4.44){$\cone \name \darg{}{q(0)}$}
\usefont{T1}{ppl}{m}{n}
\rput(7.9154687,-0.28){$\cone^{11} \name \darg{\lnot even(0)}{even(1)}$}
\usefont{T1}{ppl}{m}{n}
\rput(7.7054687,-3.02){$\darg{\lnot even(n-1)}{even(n)}$}
\usefont{T1}{ppl}{m}{n}
\rput(4.815469,1.36){$\cone^6 \name \darg{\lnot even(0)}{\lnot even(0)}$}
\usefont{T1}{ppl}{m}{n}
\rput(4.7454686,-1.38){$\cone^{12} \name \darg{\lnot even(1)}{\lnot even(1)}$}
\usefont{T1}{ppl}{m}{n}
\rput(5.005469,-4.16){$\darg{\lnot even(n)}{\lnot even(n)}$}
\usefont{T1}{ppl}{m}{n}
\rput(13.275469,2.84){$\cone^4 \name \darg{\lnot even(0)}{q(0)}$}
\usefont{T1}{ppl}{m}{n}
\rput(17.435469,4.42){$\cone^2 \name \darg{\lnot q(0)}{\lnot q(0)}$}
\usefont{T1}{ppl}{m}{n}
\rput(18.79547,2.86){$\cone^3 \name \darg{\lnot q(0)}{p}$}
\usefont{T1}{ppl}{m}{n}
\rput(11.695469,1.08){$\cone^7 \name \darg{\lnot even(0)}{q(1)}$}
\usefont{T1}{ppl}{m}{n}
\rput(13.505468,-0.24){$\cone^{10} \name \darg{\lnot even(1)}{q(1)}$}
\usefont{T1}{ppl}{m}{n}
\rput(17.37547,1.12){$\cone^8 \name \darg{\lnot q(1)}{\lnot q(1)}$}
\usefont{T1}{ppl}{m}{n}
\rput(18.71547,-0.32){$\cone^9 \name \darg{\lnot q(1)}{p}$}
\usefont{T1}{ppl}{m}{n}
\rput(22.485468,3.58){$\ctwo \name \darg{\lnot p}{r}$}
\usefont{T1}{ppl}{m}{n}
\rput(25.62547,3.58){$\cthree \name \darg{\lnot r}{\lnot r}$}
\usefont{T1}{ppl}{m}{n}
\rput(23.03547,-2.38){$\cfour \name \darg{\lnot p}{\lnot p}$}
\usefont{T1}{ppl}{m}{n}
\rput(11.895469,-1.76){$\darg{\lnot even(n-1)}{q(n)}$}
\usefont{T1}{ppl}{m}{n}
\rput(13.305469,-3.0){$\darg{\lnot even(n)}{q(n)}$}
\usefont{T1}{ppl}{m}{n}
\rput(17.345469,-1.72){$\darg{\lnot q(n)}{\lnot q(n)}$}
\usefont{T1}{ppl}{m}{n}
\rput(18.46547,-3.02){$\darg{\lnot q(n)}{p}$}
\psline[linewidth=0.04cm,arrowsize=0.05291667cm 2.0,arrowlength=1.4,arrowinset=0.4]{->}(7.241875,2.57)(7.221875,0.05)
\psline[linewidth=0.04cm,linestyle=dashed,dash=0.16cm 0.16cm,arrowsize=0.05291667cm 2.0,arrowlength=1.4,arrowinset=0.4]{->}(7.341875,-0.71)(7.361875,-2.67)
\psline[linewidth=0.04cm,arrowsize=0.05291667cm 2.0,arrowlength=1.4,arrowinset=0.4]{->}(7.221875,2.55)(4.721875,1.61)
\psline[linewidth=0.04cm,arrowsize=0.05291667cm 2.0,arrowlength=1.4,arrowinset=0.4]{->}(9.021875,2.85)(10.961875,1.49)
\psline[linewidth=0.04cm,arrowsize=0.05291667cm 2.0,arrowlength=1.4,arrowinset=0.4]{->}(5.921875,-0.61)(4.381875,-1.01)
\psline[linewidth=0.04cm,arrowsize=0.05291667cm 2.0,arrowlength=1.4,arrowinset=0.4]{->}(6.441875,-3.39)(4.841875,-3.85)
\psline[linewidth=0.04cm,arrowsize=0.05291667cm 2.0,arrowlength=1.4,arrowinset=0.4]{->}(23.741875,3.59)(24.441875,3.59)
\psline[linewidth=0.04cm,linestyle=dashed,dash=0.16cm 0.16cm](7.3723435,-3.35)(7.361875,-5.05)
\psline[linewidth=0.04cm,linestyle=dashed,dash=0.16cm 0.16cm](18.541876,-3.45)(18.581875,-4.93)
\psline[linewidth=0.04cm,linestyle=dashed,dash=0.16cm 0.16cm](13.452344,-3.51)(13.441875,-4.83)
\psline[linewidth=0.04cm,linestyle=dashed,dash=0.16cm 0.16cm,arrowsize=0.05291667cm 2.0,arrowlength=1.4,arrowinset=0.4]{->}(8.501875,-0.67)(11.101875,-1.49)
\psline[linewidth=0.04cm,arrowsize=0.05291667cm 2.0,arrowlength=1.4,arrowinset=0.4]{->}(9.021875,2.83)(11.381875,2.81)
\psline[linewidth=0.04cm,arrowsize=0.05291667cm 2.0,arrowlength=1.4,arrowinset=0.4]{->}(10.281875,-0.27)(11.501875,-0.31)
\psline[linewidth=0.04cm,arrowsize=0.05291667cm 2.0,arrowlength=1.4,arrowinset=0.4]{->}(10.041875,-3.01)(11.781875,-3.01)
\psline[linewidth=0.04cm,arrowsize=0.05291667cm 2.0,arrowlength=1.4,arrowinset=0.4]{->}(12.001875,4.45)(15.641875,4.41)
\psline[linewidth=0.04cm,arrowsize=0.05291667cm 2.0,arrowlength=1.4,arrowinset=0.4]{->}(11.941875,4.43)(17.141874,3.13)
\psline[linewidth=0.04cm,arrowsize=0.05291667cm 2.0,arrowlength=1.4,arrowinset=0.4]{->}(15.221875,2.85)(17.161875,2.85)
\psline[linewidth=0.04cm,arrowsize=0.05291667cm 2.0,arrowlength=1.4,arrowinset=0.4]{->}(15.201875,2.87)(16.041876,4.21)
\psline[linewidth=0.04cm,arrowsize=0.05291667cm 2.0,arrowlength=1.4,arrowinset=0.4]{->}(13.781875,1.11)(15.401875,1.09)
\psline[linewidth=0.04cm,arrowsize=0.05291667cm 2.0,arrowlength=1.4,arrowinset=0.4]{->}(13.761875,1.09)(17.521875,-0.01)
\psline[linewidth=0.04cm,arrowsize=0.05291667cm 2.0,arrowlength=1.4,arrowinset=0.4]{->}(15.561875,-0.25)(17.181875,-0.29)
\psline[linewidth=0.04cm,arrowsize=0.05291667cm 2.0,arrowlength=1.4,arrowinset=0.4]{->}(15.561875,-0.25)(16.101875,0.91)
\psline[linewidth=0.04cm,arrowsize=0.05291667cm 2.0,arrowlength=1.4,arrowinset=0.4]{->}(13.981875,-1.75)(15.781875,-1.81)
\psline[linewidth=0.04cm,arrowsize=0.05291667cm 2.0,arrowlength=1.4,arrowinset=0.4]{->}(14.041875,-1.77)(17.961876,-2.71)
\psline[linewidth=0.04cm,arrowsize=0.05291667cm 2.0,arrowlength=1.4,arrowinset=0.4]{->}(15.061875,-3.01)(17.381874,-3.01)
\psline[linewidth=0.04cm,arrowsize=0.05291667cm 2.0,arrowlength=1.4,arrowinset=0.4]{->}(15.101875,-3.01)(17.041876,-1.99)
\psline[linewidth=0.04cm,arrowsize=0.05291667cm 2.0,arrowlength=1.4,arrowinset=0.4]{->}(20.241875,2.83)(21.781876,3.35)
\psline[linewidth=0.04cm,arrowsize=0.05291667cm 2.0,arrowlength=1.4,arrowinset=0.4]{->}(20.261875,2.83)(23.021875,-2.11)
\psline[linewidth=0.04cm,arrowsize=0.05291667cm 2.0,arrowlength=1.4,arrowinset=0.4]{->}(20.241875,-0.31)(22.401875,3.33)
\psline[linewidth=0.04cm,arrowsize=0.05291667cm 2.0,arrowlength=1.4,arrowinset=0.4]{->}(20.221874,-0.31)(22.761875,-2.09)
\psline[linewidth=0.04cm,arrowsize=0.05291667cm 2.0,arrowlength=1.4,arrowinset=0.4]{->}(19.521875,-3.01)(22.501875,3.31)
\psline[linewidth=0.04cm,arrowsize=0.05291667cm 2.0,arrowlength=1.4,arrowinset=0.4]{->}(19.581875,-2.99)(22.461876,-2.59)
\psline[linewidth=0.04cm,linestyle=dashed,dash=0.16cm 0.16cm](16.401875,-2.13)(16.421875,-4.93)
\psline[linewidth=0.04cm,linestyle=dashed,dash=0.16cm 0.16cm](4.921875,-4.51)(4.901875,-4.93)
\psline[linewidth=0.04cm,linestyle=dashed,dash=0.16cm 0.16cm](11.221875,-2.07)(11.241875,-4.93)
\end{pspicture} 
}}
\caption{$AF(Q)$ with an infinite subset evidenced.}
\label{fig:af-non-finitary-ev}
\end{center}
\end{figure}
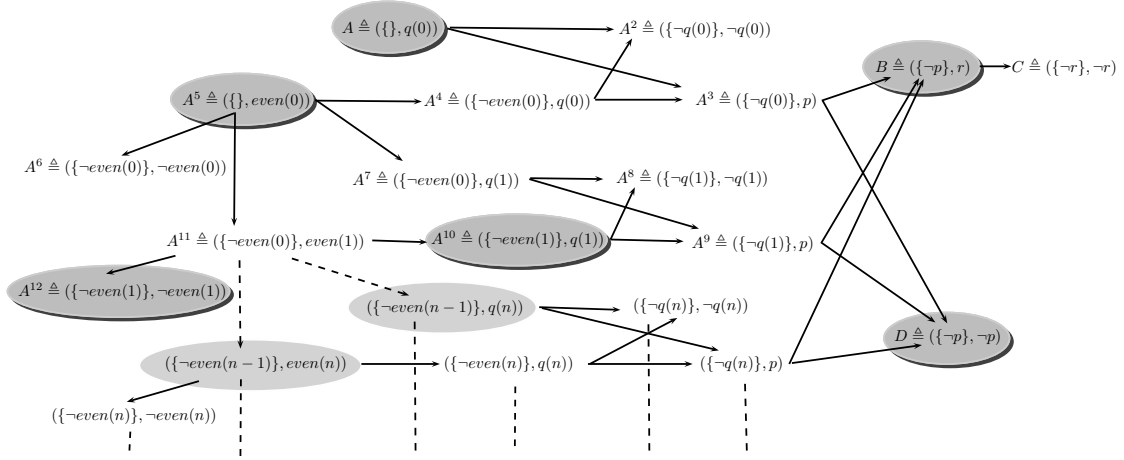

We have to prove that $\pi_{a^{Q}}^{-}(S) \cap S = \emptyset$.

We can now apply $\underline{a^Q}$ to the subsets forming the definition of $S$.
Noting in particular that 
\begin{itemize}
\item $\set{\cone^{1 + 12\varx} | \varx \ge 0}$ intersects $L_7$ \virg{starting} from $\wseq{\cone}{13}$
\item $\set{\cone^{5 + 12\varx} | \varx \ge 0}$ intersects $L_{11}$ \virg{starting} from $\wseq{\cone}{17}$
\item $\set{\cone^{10 + 12\varx} | \varx \ge 0}$ intersects $L_{4}$ \virg{starting} from $\wseq{\cone}{10}$
\item  $\set{\cone^{12 + 12\varx} | \varx \ge 0}$ intersects $L_6$ \virg{starting} from $\wseq{\cone}{12}$
\end{itemize}

we obtain:
{\allowdisplaybreaks
\noindent
\begin{eqnarray*}
\pi_{a^{Q}}^{-}(S) & = & L_3 \cup L_3 \cup \\
 & & tl^2(\wseq{\cone}{13} \cdot (\wseq{\cone}{12})^{*}) \cup \\
 & & tl^6(\wseq{\cone}{17} \cdot (\wseq{\cone}{12})^{*}) \cup \\
 & & (\wseq{\cone}{10} \cdot (\wseq{\cone}{12})^{*}) \cdot \cone \cup  \\
 & & tl(\wseq{\cone}{12} \cdot (\wseq{\cone}{12})^{*}).
\end{eqnarray*}
}

Noting that the last four elements of the above expression coincide, we have
$\pi_{a^{Q}}^{-}(S)= L_3 \cup \wseq{\cone}{11} \cdot (\wseq{\cone}{12})^{*}$.
As to conflict-freeness, it is easily seen that $\pi_{a^{Q}}^{-}(S) \cap S = \emptyset$.

Let us now turn to the problem of acceptability checking, by verifying whether the argument $\ctwo$ is acceptable wrt $S$, i.e. $\ctwo \in \FC_Q(S)$.
This requires to check whether $\pi_{a^{Q}}^{-}(\set{\ctwo}) \setminus \pi_{a^{Q}}^{+}(S) = \emptyset$ (see the proof of part b of Theorem \ref{thm:afs-ops}).

To identify $\pi_{a^{Q}}^{+}(S)$ we can apply $\underline{a^{Q}}^{+}$ to the subsets evidenced in the above definition of $S$.
In particular:
\begin{itemize}
\item the second item in the definition of $\underline{a^{Q}}^{+}$ is effective (i.e. gives a non-empty result) on $\set{\ctwo}$ yielding $\set{\cthree}$
\item no item is effective on $\set{\cfour}$
\item the fourth and sixth items are effective on $\set{\cone^{1 + 12\varx} | \varx \ge 0}$ yielding $\cone \cdot \cone \cdot (\wseq{\cone}{12})^{*}$ and $\cone \cdot \cone \cdot \cone \cdot (\wseq{\cone}{12})^{*}$
\item the eighth, ninth, tenth, and eleventh items are effective on $\set{\cone^{5 + 12\varx} | \varx \ge 0}$ yielding $\wseq{\cone}{4} \cdot (\wseq{\cone}{12})^{*}$, $\wseq{\cone}{6} \cdot (\wseq{\cone}{12})^{*}$, $\wseq{\cone}{7} \cdot (\wseq{\cone}{12})^{*}$, $\wseq{\cone}{11} \cdot (\wseq{\cone}{12})^{*}$
\item the fifth and seventh items are effective on $\set{\cone^{10 + 12\varx} | \varx \ge 0}$ yielding $\wseq{\cone}{8} \cdot (\wseq{\cone}{12})^{*}$ and $\wseq{\cone}{9} \cdot (\wseq{\cone}{12})^{*}$
\item no item is effective on $\set{\cone^{12 + 12\varx} | \varx \ge 0}$

\end{itemize}

Summing up, 

{\allowdisplaybreaks
\noindent
\begin{eqnarray*}
\pi_{a^{Q}}^{+}(S) & = & \set{\cthree} \cup \\
 & & \cone \cdot \cone \cdot (\wseq{\cone}{12})^{*} \cup \\
 & & \cone \cdot \cone \cdot \cone \cdot (\wseq{\cone}{12})^{*} \cup \\
 & & \wseq{\cone}{4} \cdot (\wseq{\cone}{12})^{*} \cup \\
 & & \wseq{\cone}{6} \cdot (\wseq{\cone}{12})^{*} \cup \\
 & & \wseq{\cone}{7} \cdot (\wseq{\cone}{12})^{*} \cup \\
 & & \wseq{\cone}{11} \cdot (\wseq{\cone}{12})^{*} \cup \\
 & &\wseq{\cone}{8} \cdot (\wseq{\cone}{12})^{*} \cup \\
 & &\wseq{\cone}{9} \cdot (\wseq{\cone}{12})^{*}.
\end{eqnarray*}
}

Since $\pi_{a^{Q}}^{-}(\set{\ctwo}) = L_3$ while from the expression derived above we note that $\cone \cdot \cone \cdot \cone \cdot (\wseq{\cone}{12})^{*} \cup \wseq{\cone}{9} \cdot (\wseq{\cone}{12})^{*} = L_3 \subset \pi_{a^{Q}}^{+}(S)$, it follows that $\ctwo \in \FC_Q(S)$.

Let us now check whether $S$ is admissible. We have already proved that $S$ is conflict free, therefore, from part (c) 
of Theorem \ref{thm:afs-ops}, we have to check whether  $\pi_{a^{Q}}^{-}(S) \setminus \pi_{a^{Q}}^{+}(S)=\emptyset$. 

Recalling 

{\allowdisplaybreaks
\noindent
\begin{eqnarray*}
\pi_{a^{Q}}^{-}(S) & = & L_3~\cup~\wseq{\cone}{11} \cdot (\wseq{\cone}{12})^{*}
\end{eqnarray*}
}

it is easily seen that $\pi_{a^{Q}}^{-}(S) \setminus \pi_{a^{Q}}^{+}(S) = \emptyset$.

We can also check whether $S$ is a stable extension. Since $S$ is conflict free, we just need to confirm that $S\cup\pi_{a}^{+}(S)=\XC$:

{\allowdisplaybreaks
\noindent
\begin{eqnarray*}
S \cup \pi_{a^{Q}}^{+}(S) & =  & \ctwo~\cup~\cfour~\cup~\cone \cdot (\wseq{\cone}{12})^{*}~\cup\\
             &   & \wseq{\cone}{5} \cdot (\wseq{\cone}{12})^{*}~\cup~\wseq{\cone}{10} \cdot (\wseq{\cone}{12})^{*}~\cup\\
             &   & \wseq{\cone}{12} \cdot (\wseq{\cone}{12})^{*}~\cup\\
             &   & \set{\cthree} \cup 
 \cone \cdot \cone \cdot (\wseq{\cone}{12})^{*}~\cup~\cone \cdot \cone \cdot \cone \cdot (\wseq{\cone}{12})^{*}~\cup \\
             &   & \wseq{\cone}{4} \cdot (\wseq{\cone}{12})^{*}~\cup~\wseq{\cone}{6} \cdot (\wseq{\cone}{12})^{*}~\cup\\
&   & \wseq{\cone}{7} \cdot (\wseq{\cone}{12})^{*}~\cup \\
&   & \wseq{\cone}{8} \cdot (\wseq{\cone}{12})^{*}~\cup~\wseq{\cone}{9} \cdot (\wseq{\cone}{12})^{*}~\cup \\
             &   & \wseq{\cone}{11} \cdot (\wseq{\cone}{12})^{*}~\cup \\
             & = & \cone \cdot \cone^{*}~\cup~\ctwo~\cup~\cthree~\cup~\cfour \\
             & = & \XC_Q
\end{eqnarray*}
}

Then $S$ is a stable extension of $AF(Q)$. From this fact it follows that $S$ is also a complete extension of $AF(Q)$, hence $\FC_Q(S)=S$.
This could be independently verified, according to part (f) 
of Theorem \ref{thm:afs-ops}, computing  $\FC_Q(S) = \XC_Q~\setminus~\pi^{+}_{a^Q}(~\XC_Q\setminus\pi^{+}_{a^Q}(S)~)$.
As we already know, $\XC_Q\setminus\pi^{+}_{a^Q}(S) = S$, hence  $\FC_Q(S) = \XC_Q~\setminus~\pi^{+}_{a^Q}(S) = S$.

It can also be observed that $AF(Q)$ is well-founded (Definition 29 of \cite{dung:1995}) namely there is no infinite sequence of arguments $X_0, X_1, \ldots, X_n, \ldots$ such that $X_{i+1}$ attacks $X_{i}$.
Note in particular that letting $X_0$ any argument in $L_{11}$, i.e. $X_0=\cone^{11+6\varx}$ for some $\varx \geq 0$, there is only a finite sequence $X_0, \ldots, X_{\varx + 1}$ satisfying the condition of Definition 29 in \cite{dung:1995}, with $X_1 = \cone^{11+6(\varx-1)}, \ldots, X_{\varx + 1}=\cone^{5}$. Note also that the framework would not be well-founded with a \virg{reverse} attack relation, namely if we had $(\cone^{11+6(\varx+1)}, \cone^{11+6(\varx)} ) \in \AC_Q$ instead of having $(\cone^{11+6(\varx-1)}, \cone^{11+6(\varx)} ) \in \AC_Q$.

Since $AF(Q)$ is well-founded, by Theorem 30 of \cite{dung:1995} it has exactly one complete extension which is grounded, preferred and stable, namely the set $S$ identified above.
It is described by the regular language $L_{\EC_{gr}^{Q}}$ accepted by the DFA depicted in Figure \ref{fig:af-non-finitary-dfa-grounded}.

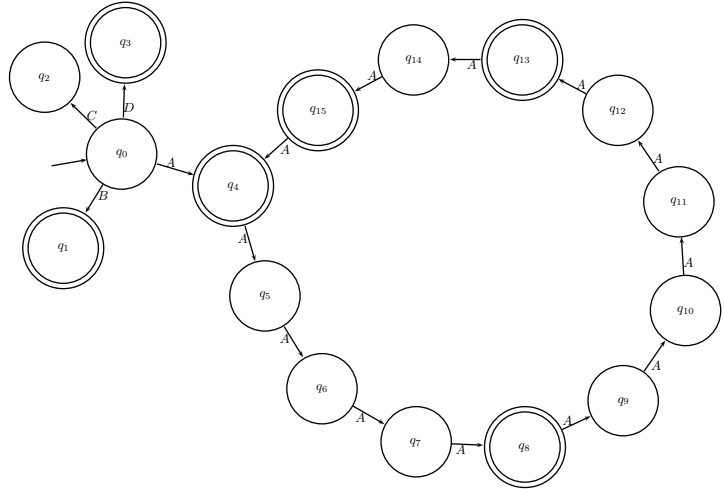
\begin{figure}[!htb]
\begin{center}
\scalebox{0.5}{

\begin{pspicture}[linewidth=1bp](0bp,0bp)(561.25bp,368.8bp)

  \pstVerb{2 setlinejoin} 
\psset{linecolor=black}
  \psbezier[arrows=->](109.32bp,280.67bp)(109.47bp,285.41bp)(109.63bp,290.43bp)(110.12bp,305.65bp)
  \psset{linecolor=[rgb]{0.0,0.0,0.0}}
  \rput(113.55bp,288.06bp){$\cfour$}
  \psset{linecolor=black}
  \psbezier[arrows=->](304.46bp,311.51bp)(300.73bp,309.5bp)(296.81bp,307.39bp)(283.79bp,300.39bp)
  \psset{linecolor=[rgb]{0.0,0.0,0.0}}
  \rput(296.67bp,312.39bp){$\cone$}
  \psset{linecolor=black}
  \psbezier[arrows=->](200.96bp,198.75bp)(202.53bp,193.33bp)(204.19bp,187.62bp)(208.65bp,172.23bp)
  \psset{linecolor=[rgb]{0.0,0.0,0.0}}
  \rput(199.38bp,189.39bp){$\cone$}
  \psset{linecolor=black}
  \psbezier[arrows=->](88.941bp,272.46bp)(85.033bp,276.29bp)(80.861bp,280.38bp)(69.498bp,291.53bp)
  \psset{linecolor=[rgb]{0.0,0.0,0.0}}
  \rput(85.847bp,282.44bp){$\cthree$}
  \psset{linecolor=black}
  \psbezier[arrows=->](233.3bp,265.04bp)(230.07bp,262.16bp)(226.71bp,259.17bp)(215.66bp,249.32bp)
  \psset{linecolor=[rgb]{0.0,0.0,0.0}}
  \rput(231.33bp,256.61bp){$\cone$}
  \psset{linecolor=black}
  \psbezier[arrows=->](457.97bp,298.53bp)(454.23bp,300.53bp)(450.29bp,302.63bp)(437.2bp,309.61bp)
  \psset{linecolor=[rgb]{0.0,0.0,0.0}}
  \rput(454.16bp,305.63bp){$\cone$}
  \psset{linecolor=black}
  \psbezier[arrows=->](512.88bp,240.14bp)(509.73bp,244.8bp)(506.37bp,249.77bp)(497.23bp,263.28bp)
  \psset{linecolor=[rgb]{0.0,0.0,0.0}}
  \rput(511.97bp,249.4bp){$\cone$}
  \psset{linecolor=black}
  \psbezier[arrows=->](356.89bp,34.467bp)(361.14bp,34.275bp)(365.6bp,34.074bp)(380.43bp,33.405bp)
  \psset{linecolor=[rgb]{0.0,0.0,0.0}}
  \rput(363.48bp,30.169bp){$\cone$}
  \psset{linecolor=black}
  \psbezier[arrows=->](282.57bp,62.857bp)(287.27bp,60.139bp)(292.3bp,57.238bp)(305.98bp,49.336bp)
  \psset{linecolor=[rgb]{0.0,0.0,0.0}}
  \rput(287.91bp,54.619bp){$\cone$}
  \psset{linecolor=black}
  \psbezier[arrows=->](501.71bp,89.112bp)(505bp,93.882bp)(508.53bp,99.004bp)(517.7bp,112.33bp)
  \psset{linecolor=[rgb]{0.0,0.0,0.0}}
  \rput(510.85bp,93.572bp){$\cone$}
  \psset{linecolor=black}
  \psbezier[arrows=->](55.233bp,244.13bp)(60.588bp,245.06bp)(66.288bp,246.04bp)(81.779bp,248.72bp)
  \psbezier[arrows=->](378.52bp,324.32bp)(374.24bp,324.33bp)(369.81bp,324.34bp)(355.43bp,324.36bp)
  \psset{linecolor=[rgb]{0.0,0.0,0.0}}
  \rput(371.98bp,320.33bp){$\cone$}
  \psset{linecolor=black}
  \psbezier[arrows=->](134.52bp,245.62bp)(140.33bp,243.91bp)(146.62bp,242.05bp)(162.45bp,237.36bp)
  \psset{linecolor=[rgb]{0.0,0.0,0.0}}
  \rput(144.68bp,246.92bp){$\cone$}
  \psset{linecolor=black}
  \psbezier[arrows=->](531.62bp,161.92bp)(531.27bp,167.76bp)(530.89bp,174.04bp)(529.91bp,190.33bp)
  \psset{linecolor=[rgb]{0.0,0.0,0.0}}
  \rput(535.07bp,171.05bp){$\cone$}
  \psset{linecolor=black}
  \psbezier[arrows=->](230.51bp,122.86bp)(233.37bp,118.21bp)(236.42bp,113.25bp)(244.74bp,99.736bp)
  \psset{linecolor=[rgb]{0.0,0.0,0.0}}
  \rput(230.97bp,113.61bp){$\cone$}
  \psset{linecolor=black}
  \psbezier[arrows=->](93.983bp,230.1bp)(91.475bp,226.07bp)(88.817bp,221.8bp)(80.77bp,208.86bp)
  \psset{linecolor=[rgb]{0.0,0.0,0.0}}
  \rput(94.073bp,221.82bp){$\ctwo$}
  \psset{linecolor=black}
  \psbezier[arrows=->](439.85bp,45.163bp)(443.86bp,47.027bp)(448.01bp,48.959bp)(461.48bp,55.235bp)
  \psset{linecolor=[rgb]{0.0,0.0,0.0}}
  \rput(443.98bp,52.016bp){$\cone$}
{%
  \psset{linecolor=[rgb]{0.0,0.0,0.0}}
  \psellipse[](64bp,182bp)(27bp,27bp)
  \psellipse[](64bp,182bp)(31bp,31bp)
  \rput(64.293bp,182.38bp){$q_1$}
}%
{%
  \psset{linecolor=[rgb]{0.0,0.0,0.0}}
  \psellipse[](108bp,253bp)(27bp,27bp)
  \rput(108.44bp,253.34bp){$q_0$}
}%
{%
  \psset{linecolor=[rgb]{0.0,0.0,0.0}}
  \psellipse[](111bp,337bp)(27bp,27bp)
  \psellipse[](111bp,337bp)(31bp,31bp)
  \rput(111.12bp,336.8bp){$q_3$}
}%
{%
  \psset{linecolor=[rgb]{0.0,0.0,0.0}}
  \psellipse[](50bp,311bp)(27bp,27bp)
  \rput(49.887bp,310.76bp){$q_2$}
}%
{%
  \psset{linecolor=[rgb]{0.0,0.0,0.0}}
  \psellipse[](216bp,146bp)(27bp,27bp)
  \rput(216.23bp,146.04bp){$q_5$}
}%
{%
  \psset{linecolor=[rgb]{0.0,0.0,0.0}}
  \psellipse[](192bp,229bp)(27bp,27bp)
  \psellipse[](192bp,229bp)(31bp,31bp)
  \rput(192.34bp,228.53bp){$q_4$}
}%
{%
  \psset{linecolor=[rgb]{0.0,0.0,0.0}}
  \psellipse[](330bp,36bp)(27bp,27bp)
  \rput(329.59bp,35.698bp){$q_7$}
}%
{%
  \psset{linecolor=[rgb]{0.0,0.0,0.0}}
  \psellipse[](259bp,76bp)(27bp,27bp)
  \rput(259.1bp,76.415bp){$q_6$}
}%
{%
  \psset{linecolor=[rgb]{0.0,0.0,0.0}}
  \psellipse[](486bp,67bp)(27bp,27bp)
  \rput(486.36bp,66.819bp){$q_9$}
}%
{%
  \psset{linecolor=[rgb]{0.0,0.0,0.0}}
  \psellipse[](412bp,32bp)(27bp,27bp)
  \psellipse[](412bp,32bp)(31bp,31bp)
  \rput(411.58bp,32bp){$q_8$}
}%
{%
  \psset{linecolor=[rgb]{0.0,0.0,0.0}}
  \psellipse[](328bp,324bp)(27bp,27bp)
  \rput(328.43bp,324.41bp){$q_{14}$}
}%
{%
  \psset{linecolor=[rgb]{0.0,0.0,0.0}}
  \psellipse[](256bp,286bp)(27bp,27bp)
  \psellipse[](256bp,286bp)(31bp,31bp)
  \rput(256.44bp,285.67bp){$q_{15}$}
}%
{%
  \psset{linecolor=[rgb]{0.0,0.0,0.0}}
  \psellipse[](528bp,217bp)(27bp,27bp)
  \rput(528.27bp,217.41bp){$q_{11}$}
}%
{%
  \psset{linecolor=[rgb]{0.0,0.0,0.0}}
  \psellipse[](533bp,135bp)(27bp,27bp)
  \rput(533.25bp,134.9bp){$q_{10}$}
}%
{%
  \psset{linecolor=[rgb]{0.0,0.0,0.0}}
  \psellipse[](410bp,324bp)(27bp,27bp)
  \psellipse[](410bp,324bp)(31bp,31bp)
  \rput(409.71bp,324.26bp){$q_{13}$}
}%
{%
  \psset{linecolor=[rgb]{0.0,0.0,0.0}}
  \psellipse[](482bp,286bp)(27bp,27bp)
  \rput(482.07bp,285.68bp){$q_{12}$}
}%
\end{pspicture}

}
\caption{A DFA accepting the regular language representing the unique complete, grounded, preferred, and stable extension of $AF(Q)$.}
\label{fig:af-non-finitary-dfa-grounded}
\end{center}
\end{figure}

\subsection{An example in multi-agent negotiation (from Sec. \ref{subsubsec:intr-ex-negot})}\label{examp:negotiation}

Referring to the description of the example given in Section \ref{subsubsec:intr-ex-negot}, the global argumentation framework $AF_{neg}$ arising from the non-terminating message exchanges among the three agents is depicted in Figure \ref{fig:negotiation-three}.

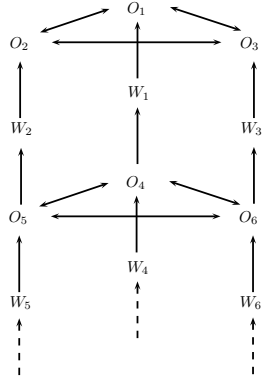
\begin{figure}[!htb]
\begin{center}
\scalebox{0.6} 
{
\begin{pspicture}(0,-4.1709375)(6.2671876,4.1909375)
\usefont{T1}{ppl}{m}{n}
\rput(3.1054688,3.9790626){$O_1$}
\usefont{T1}{ppl}{m}{n}
\rput(0.5054687,3.1990626){$O_2$}
\usefont{T1}{ppl}{m}{n}
\rput(5.6054688,3.1990626){$O_3$}
\usefont{T1}{ppl}{m}{n}
\rput(3.1554687,2.0990624){$W_1$}
\usefont{T1}{ppl}{m}{n}
\rput(0.55546874,1.3190625){$W_2$}
\usefont{T1}{ppl}{m}{n}
\rput(5.655469,1.3190625){$W_3$}
\psline[linewidth=0.04cm,arrowsize=0.05291667cm 2.0,arrowlength=1.4,arrowinset=0.4]{<->}(0.961875,3.4490626)(2.501875,3.9690626)
\psline[linewidth=0.04cm,arrowsize=0.05291667cm 2.0,arrowlength=1.4,arrowinset=0.4]{<->}(3.841875,4.0090623)(5.421875,3.4890625)
\psline[linewidth=0.04cm,arrowsize=0.05291667cm 2.0,arrowlength=1.4,arrowinset=0.4]{<->}(1.221875,3.2290626)(4.981875,3.2290626)
\psline[linewidth=0.04cm,arrowsize=0.05291667cm 2.0,arrowlength=1.4,arrowinset=0.4]{<-}(3.121875,3.6890626)(3.121875,2.4290626)
\psline[linewidth=0.04cm,arrowsize=0.05291667cm 2.0,arrowlength=1.4,arrowinset=0.4]{<-}(0.521875,2.8090625)(0.521875,1.5490625)
\psline[linewidth=0.04cm,arrowsize=0.05291667cm 2.0,arrowlength=1.4,arrowinset=0.4]{<-}(5.701875,2.8090625)(5.701875,1.5490625)
\usefont{T1}{ppl}{m}{n}
\rput(3.0854688,0.1190625){$O_4$}
\usefont{T1}{ppl}{m}{n}
\rput(0.48546875,-0.6609375){$O_5$}
\usefont{T1}{ppl}{m}{n}
\rput(5.585469,-0.6609375){$O_6$}
\usefont{T1}{ppl}{m}{n}
\rput(3.1354687,-1.7609375){$W_4$}
\usefont{T1}{ppl}{m}{n}
\rput(0.53546876,-2.5409374){$W_5$}
\usefont{T1}{ppl}{m}{n}
\rput(5.635469,-2.5409374){$W_6$}
\psline[linewidth=0.04cm,arrowsize=0.05291667cm 2.0,arrowlength=1.4,arrowinset=0.4]{<->}(0.941875,-0.4109375)(2.481875,0.1090625)
\psline[linewidth=0.04cm,arrowsize=0.05291667cm 2.0,arrowlength=1.4,arrowinset=0.4]{<->}(3.821875,0.1490625)(5.401875,-0.3709375)
\psline[linewidth=0.04cm,arrowsize=0.05291667cm 2.0,arrowlength=1.4,arrowinset=0.4]{<->}(1.201875,-0.6309375)(4.961875,-0.6309375)
\psline[linewidth=0.04cm,arrowsize=0.05291667cm 2.0,arrowlength=1.4,arrowinset=0.4]{<-}(3.101875,-0.1709375)(3.101875,-1.4309375)
\psline[linewidth=0.04cm,arrowsize=0.05291667cm 2.0,arrowlength=1.4,arrowinset=0.4]{<-}(0.501875,-1.0509375)(0.501875,-2.3109374)
\psline[linewidth=0.04cm,arrowsize=0.05291667cm 2.0,arrowlength=1.4,arrowinset=0.4]{<-}(5.681875,-1.0509375)(5.681875,-2.3109374)
\psline[linewidth=0.04cm,arrowsize=0.05291667cm 2.0,arrowlength=1.4,arrowinset=0.4]{<-}(0.541875,0.8890625)(0.541875,-0.3709375)
\psline[linewidth=0.04cm,arrowsize=0.05291667cm 2.0,arrowlength=1.4,arrowinset=0.4]{<-}(3.121875,1.7890625)(3.121875,0.5290625)
\psline[linewidth=0.04cm,arrowsize=0.05291667cm 2.0,arrowlength=1.4,arrowinset=0.4]{<-}(5.701875,0.9090625)(5.701875,-0.3509375)
\psline[linewidth=0.04cm,linestyle=dashed,dash=0.16cm 0.16cm,arrowsize=0.05291667cm 2.0,arrowlength=1.4,arrowinset=0.4]{<-}(0.501875,-2.8909376)(0.501875,-4.1509376)
\psline[linewidth=0.04cm,linestyle=dashed,dash=0.16cm 0.16cm,arrowsize=0.05291667cm 2.0,arrowlength=1.4,arrowinset=0.4]{<-}(3.121875,-2.0709374)(3.121875,-3.3309374)
\psline[linewidth=0.04cm,linestyle=dashed,dash=0.16cm 0.16cm,arrowsize=0.05291667cm 2.0,arrowlength=1.4,arrowinset=0.4]{<-}(5.681875,-2.8709376)(5.681875,-4.1309376)
\end{pspicture} 
}
\caption{Graphical representation of $AF_{neg}$.}
\label{fig:negotiation-three}
\end{center}
\end{figure}

Upon detection of a long sequence of withdrawals and reiterations of the same offers (and assuming that the agents programmatically repeat their behavior), the market authority can identify\footnote{The problem of identifying an AF specification from a regular sequence of observations has direct connections with the the widely studied (and partially overlapping) fields of automata identification and grammatical inference \cite{Higuera05}. Defining algorithms for the identification of AF specification is an interesting issue for future work, that we are confident can be faced resorting to techniques borrowed from the above mentioned areas.} the relevant AF specification, which can be given as follows. 

Let $\Sigma=\set{0}$ and $\XC~=~\set{0^i~:~i\geq1}~=~\set{0}\cdot\set{0}^{*}$.
We can partition $\XC$ into six sets $L_A$, $L_B$, $L_C$, $L_D$, $L_E$, and $L_F$ (corresponding respectively to the six sequences $\set{O_1, O_4, \ldots}$, $\set{O_2, O_5, \ldots}$, $\set{O_3, O_6, \ldots}$, $\set{W_1, W_4, \ldots}$, $\set{W_2, W_5, \ldots}$, $\set{W_3, W_6, \ldots}$) as follows:
\[
\begin{array}{lclcl}
L_A&\mbox{ $=$ }&\set{0^{6i+1}~:~i\geq0}&\mbox{ $=$ }&\set{0}\cdot\set{000000}^{*}\\
L_B&\mbox{ $=$ }&\set{0^{6i+2}~:~i\geq0}&\mbox{ $=$ }&\set{00}\cdot\set{000000}^{*}\\
L_C&\mbox{ $=$ }&\set{0^{6i+3}~:~i\geq0}&\mbox{ $=$ }&\set{000}\cdot\set{000000}^{*}\\
L_D&\mbox{ $=$ }&\set{0^{6i+4}~:~i\geq0}&\mbox{ $=$ }&\set{0000}\cdot\set{000000}^{*}\\
L_E&\mbox{ $=$ }&\set{0^{6i+5}~:~i\geq0}&\mbox{ $=$ }&\set{00000}\cdot\set{000000}^{*}\\
L_F&\mbox{ $=$ }&\set{0^{6i}~:~i\geq1}&\mbox{ $=$ }&\set{000000}\cdot\set{000000}^{*}
\end{array}
\]

The attack expression can then be formulated as follows:

$ a = (I \cdot 000)~\cup~ (I\cap L_A)\cdot0 \cup (I\cap L_A)\cdot(00) \cup (I\cap L_B)\cdot 0 \cup tl((I\cap L_B)) \cup  tl(I\cap L_C) \cup tl(tl(I\cap L_C))$.

The market authority can then stop the activities of the agents and check whether some combination of offers and withdrawals can be regarded as a feasible solution (the market authority is interested in favouring the execution of as many exchanges as possible).
Using the algorithms presented in Section \ref{sec:computing}, it can be checked that:
\begin{itemize}
\item all three sets representing the reiteration of a specific offer, namely $L_A$, $L_B$, and $L_C$ corresponding respectively to $\set{O_1, O_4, O_7, \ldots}$, $\set{O_2, O_5, O_8, \ldots}$, and $\set{O_3, O_6, O_9, \ldots}$, are admissible
\item none of the possible pairwise unions of the three sets above is admissible
\item each set consisting of the reiteration of an offer and of the withdrawals of the two other offers (i.e. each of the following sets $L_A \cup L_E \cup L_F$; $L_B \cup L_D \cup L_F$; $L_C \cup L_D \cup L_E$;) is stable.
\end{itemize}

On the basis of these evaluations, it emerges that exactly one of the three exchanges can be executed, with the choice left to the authority itself.

Consider now a similar situation with four agents involved in the loop, with the initial situation as described in Table \ref{tab:4negotiation-start}.

\begin{table}[htb]
  \centering
  \begin{tabular}{| c | c | c | c |}
  \hline
  Agent ID & Owns  & Knows & Preference rank \\
  \hline
  $A_1$    & $R_d$ & $A_2$ owns $R_c$  & $R_a > R_b > R_c > R_d$ \\
  \hline
  $A_2$    & $R_c$ & $A_3$ owns $R_b$  & $R_d > R_a > R_b > R_c$\\
  \hline
  $A_3$    & $R_b$ & $A_4$ owns $R_a$  & $R_c > R_d > R_a > R_b$ \\
  \hline
  $A_4$    & $R_a$ & $A_1$ owns $R_d$  & $R_b > R_c > R_d > R_a$ \\
  \hline
  \end{tabular}
  \caption{Initial state of the negotiation example with 4 agents}
  \label{tab:4negotiation-start}
\end{table}

In this case in the first round we have four offers, namely: 
\begin{itemize}
\item $O_1 = Off(t_0, (A_1,A_2, Exch(R_d,R_c)))$
\item $O_2 = Off(t_0, (A_2,A_3, Exch(R_c,R_b)))$
\item $O_3 = Off(t_0, (A_3,A_4, Exch(R_b,R_a)))$
\item $O_4 = Off(t_0, (A_4,A_1, Exch(R_a,R_d)))$
\end{itemize} 
As in the case above we have consequently four withdrawals, four offers in turn and so on (see the framework $AF_{neg4}$ in Figure \ref{fig:negotiation-four}).

\begin{figure}[!htb]
\begin{center}
\scalebox{0.6} 
{
\begin{pspicture}(0,-4.8009377)(7.5671873,4.8209376)
\usefont{T1}{ptm}{m}{n}
\rput(1.9754688,4.6090627){$O_1$}
\usefont{T1}{ptm}{m}{n}
\rput(5.2654686,4.6090627){$O_2$}
\usefont{T1}{ptm}{m}{n}
\rput(5.2654686,2.5490625){$O_3$}
\usefont{T1}{ptm}{m}{n}
\rput(1.9754688,2.5490625){$O_4$}
\psline[linewidth=0.04cm,arrowsize=0.05291667cm 2.0,arrowlength=1.4,arrowinset=0.4]{<->}(2.781875,4.6590624)(4.661875,4.6590624)
\psline[linewidth=0.04cm,arrowsize=0.05291667cm 2.0,arrowlength=1.4,arrowinset=0.4]{<->}(2.721875,2.5790625)(4.601875,2.5790625)
\psline[linewidth=0.04cm,arrowsize=0.05291667cm 2.0,arrowlength=1.4,arrowinset=0.4]{<->}(2.001875,4.2590623)(2.001875,2.7990625)
\psline[linewidth=0.04cm,arrowsize=0.05291667cm 2.0,arrowlength=1.4,arrowinset=0.4]{<->}(5.301875,4.2590623)(5.301875,2.7990625)
\usefont{T1}{ptm}{m}{n}
\rput(0.49546874,1.6090626){$W_1$}
\usefont{T1}{ptm}{m}{n}
\rput(2.7754688,1.6090626){$W_4$}
\usefont{T1}{ptm}{m}{n}
\rput(4.5154686,1.6090626){$W_3$}
\usefont{T1}{ptm}{m}{n}
\rput(6.9154687,1.6090626){$W_2$}
\psline[linewidth=0.04cm,arrowsize=0.05291667cm 2.0,arrowlength=1.4,arrowinset=0.4]{->}(2.741875,1.8990625)(2.081875,2.2790625)
\psline[linewidth=0.04cm,arrowsize=0.05291667cm 2.0,arrowlength=1.4,arrowinset=0.4]{->}(4.481875,1.8590626)(5.161875,2.2790625)
\psline[linewidth=0.04cm,arrowsize=0.05291667cm 2.0,arrowlength=1.4,arrowinset=0.4]{->}(0.521875,1.9790626)(1.601875,4.3190627)
\psline[linewidth=0.04cm,arrowsize=0.05291667cm 2.0,arrowlength=1.4,arrowinset=0.4]{->}(6.841875,2.0390625)(5.561875,4.3990626)
\usefont{T1}{ptm}{m}{n}
\rput(2.0154688,-0.1109375){$O_5$}
\usefont{T1}{ptm}{m}{n}
\rput(5.3054686,-0.1109375){$O_6$}
\usefont{T1}{ptm}{m}{n}
\rput(5.3054686,-2.1709375){$O_7$}
\usefont{T1}{ptm}{m}{n}
\rput(2.0154688,-2.1709375){$O_8$}
\psline[linewidth=0.04cm,arrowsize=0.05291667cm 2.0,arrowlength=1.4,arrowinset=0.4]{<->}(2.821875,-0.0609375)(4.701875,-0.0609375)
\psline[linewidth=0.04cm,arrowsize=0.05291667cm 2.0,arrowlength=1.4,arrowinset=0.4]{<->}(2.761875,-2.1409376)(4.641875,-2.1409376)
\psline[linewidth=0.04cm,arrowsize=0.05291667cm 2.0,arrowlength=1.4,arrowinset=0.4]{<->}(2.041875,-0.4609375)(2.041875,-1.9209375)
\psline[linewidth=0.04cm,arrowsize=0.05291667cm 2.0,arrowlength=1.4,arrowinset=0.4]{<->}(5.341875,-0.4609375)(5.341875,-1.9209375)
\usefont{T1}{ptm}{m}{n}
\rput(0.53546876,-3.1109376){$W_5$}
\usefont{T1}{ptm}{m}{n}
\rput(2.8154688,-3.1109376){$W_8$}
\usefont{T1}{ptm}{m}{n}
\rput(4.5554686,-3.1109376){$W_7$}
\usefont{T1}{ptm}{m}{n}
\rput(6.9554687,-3.1109376){$W_6$}
\psline[linewidth=0.04cm,arrowsize=0.05291667cm 2.0,arrowlength=1.4,arrowinset=0.4]{->}(2.781875,-2.8209374)(2.121875,-2.4409375)
\psline[linewidth=0.04cm,arrowsize=0.05291667cm 2.0,arrowlength=1.4,arrowinset=0.4]{->}(4.521875,-2.8609376)(5.201875,-2.4409375)
\psline[linewidth=0.04cm,arrowsize=0.05291667cm 2.0,arrowlength=1.4,arrowinset=0.4]{->}(0.561875,-2.7409375)(1.641875,-0.4009375)
\psline[linewidth=0.04cm,arrowsize=0.05291667cm 2.0,arrowlength=1.4,arrowinset=0.4]{->}(6.881875,-2.6809375)(5.601875,-0.3209375)
\psline[linewidth=0.04cm,arrowsize=0.05291667cm 2.0,arrowlength=1.4,arrowinset=0.4]{->}(1.981875,0.1990625)(0.541875,1.2990625)
\psline[linewidth=0.04cm,arrowsize=0.05291667cm 2.0,arrowlength=1.4,arrowinset=0.4]{->}(5.361875,0.1590625)(7.061875,1.3590626)
\psbezier[linewidth=0.04,arrowsize=0.05291667cm 2.0,arrowlength=1.4,arrowinset=0.4]{->}(2.281875,-1.8609375)(3.701875,-0.3009375)(3.161875,0.2190625)(3.061875,1.3590626)
\psbezier[linewidth=0.04,arrowsize=0.05291667cm 2.0,arrowlength=1.4,arrowinset=0.4]{->}(5.081875,-1.8609375)(3.901875,-0.9209375)(4.261875,0.2790625)(4.401875,1.3790625)
\psline[linewidth=0.04cm,linestyle=dashed,dash=0.16cm 0.16cm,arrowsize=0.05291667cm 2.0,arrowlength=1.4,arrowinset=0.4]{->}(6.181875,-4.7609377)(6.901875,-3.5009375)
\psline[linewidth=0.04cm,linestyle=dashed,dash=0.16cm 0.16cm,arrowsize=0.05291667cm 2.0,arrowlength=1.4,arrowinset=0.4]{->}(1.301875,-4.7809377)(0.521875,-3.5009375)
\psline[linewidth=0.04cm,linestyle=dashed,dash=0.16cm 0.16cm,arrowsize=0.05291667cm 2.0,arrowlength=1.4,arrowinset=0.4]{->}(3.001875,-4.7209377)(3.001875,-3.5009375)
\psline[linewidth=0.04cm,linestyle=dashed,dash=0.16cm 0.16cm,arrowsize=0.05291667cm 2.0,arrowlength=1.4,arrowinset=0.4]{->}(4.521875,-4.7209377)(4.521875,-3.5009375)
\end{pspicture} 
}
\caption{Graphical representation of $AF_{neg4}$.}
\label{fig:negotiation-four}
\end{center}
\end{figure}
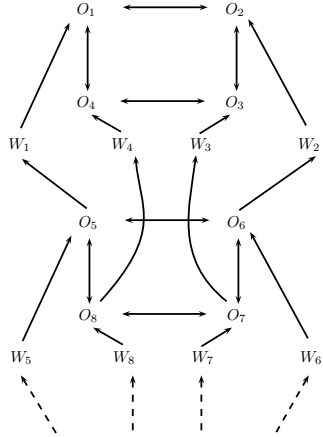

Skipping technical details, it turns out that:
\begin{itemize}
\item all three sets representing the reiteration of a specific offer, namely ${O_1, O_5, O_9, \ldots}$, ${O_2, O_6, O_{10}, \ldots}$, ${O_3, O_7, O_{11}, \ldots}$, and ${O_4, O_8, O_{12}, \ldots}$, are admissible;
\item two of the pairwise unions of these sets are admissible namely ${O_1, O_3, O_5, \ldots}$, and ${O_2, O_4, O_6, \ldots}$;
\item each set consisting of one of the above mentioned pairwise unions and of the withdrawals of the two other offers is stable.
\end{itemize}

On the basis of these evaluations, it emerges that two exchanges can be executed, with the choice left again to the authority. 

In general, using the evaluation of an infinite framework, the authority can go beyond detecting and stopping non terminating situations in this kind of multi-agent dialogues: the added-value consists in identifying which exchanges are anyway feasible in such situations.

\subsection{An example in ambient intelligence (from Sec. \ref{subsubsec:intr-ex-ambient})} \label{examp:ambient}

Referring to the description of the example given in Section \ref{subsubsec:intr-ex-ambient} and omitting the burden of some uninteresting details (in particular all arguments corresponding to default assumptions which are contradicted by facts), the argumentation framework corresponding to the interactions among the components of the ambient intelligence system consists of:

\begin{itemize}
\item a finite part corresponding to basic facts which are not time-dependent, namely $F_1=person(\Brian)$, $F_2=room(\office)$, $F_3=phone(\Brianphone)$, $F_4=owner(\Brianphone, \Brian)$ and are not involved in attack relations;
\item an infinite part consisting of the regular iteration of a section corresponding to even time instants and a section corresponding to odd time instants\footnote{A similar but more articulated structure would arise in case the different sensors produce data with different periods.}.
\end{itemize}

The following arguments and attacks are common to all sections independently of oddness or evenness of the time instant $i$:
\begin{itemize}
\item two facts corresponding to device readings: $NVR(i) = \truen videorecogn(\Brian, \office, i)$, $PI(i)=phonein(\Brianphone, \office, i)$;
\item an argument $PL(i)$ with conclusion $phlocated(\Brianphone, i)$ derived from fact $phonein(\Brianphone, \office, i)$ using $(r8)$;
\item an argument $VV(i)$ with conclusion $videovalid(\office,i)$ derived using $(r9)$ on the basis of the default assumption $\n dark(\office, i)$;
\item an argument $IN(i)$ with conclusion $in(\Brian, \office, i)$ derived using $(r1)$ on the basis of the default assumption $\n videovalid(\office, i)$;
\item an argument $NIN(i)$ with conclusion $\truen in(\Brian, \office, i)$ derived using $(r3)$ on the basis of the fact $\truen videorecogn(\Brian, \office, i)$ and of the previously derived conclusion $videovalid(\office,i)$;
\end{itemize}

The following arguments are included only in sections corresponding to an even time instant $j$:
\begin{itemize}
\item the fact $D(j)$ corresponding to the device reading $dark(\office, j)$;
\item an argument $LO(j)$ with conclusion  $lighton(\office, j+1)$ derived using $(r6)$.
\end{itemize}

The following arguments are included only in sections corresponding to an odd time instant $k$:
\begin{itemize}
\item the fact $ND(k)$ corresponding to the device reading: $\truen dark(\office, k)$;
\item an argument $NLO(k)$ with conclusion  $\truen lighton(\office, k+1)$ derived using $(r7)$.
\end{itemize}

As to attacks:
\begin{itemize}
\item each argument with conclusion $videovalid(\office,i)$ attacks the argument with conclusion $in(\Brian, \office, i)$;
\item arguments with conclusion $in(\Brian, \office, i)$ and $\truen in(\Brian, \office, i)$ mutually attack each other;
\item each fact $dark(\office, i)$ attacks the arguments with conclusions $videovalid(\office,i)$ and $\truen in(\Brian, \office, i)$.
\end{itemize}

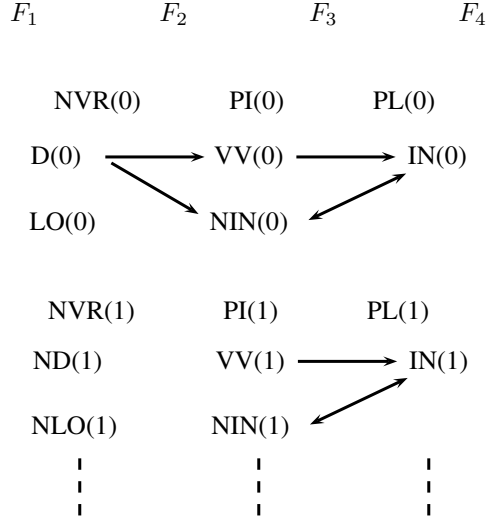
\begin{figure}[!htb]
\begin{center}
\scalebox{1} 
{
\begin{pspicture}(0,-3.4292188)(6.882813,3.4492188)
\usefont{T1}{ptm}{m}{n}
\rput(0.42234376,3.2607813){$F_1$}
\usefont{T1}{ptm}{m}{n}
\rput(2.4090104,3.2607813){$F_2$}
\usefont{T1}{ptm}{m}{n}
\rput(4.395677,3.2607813){$F_3$}
\usefont{T1}{ptm}{m}{n}
\rput(6.382344,3.2607813){$F_4$}
\usefont{T1}{ptm}{m}{n}
\rput(1.3946875,2.0907812){NVR(0)}
\usefont{T1}{ptm}{m}{n}
\rput(3.509375,2.0907812){PI(0)}
\usefont{T1}{ptm}{m}{n}
\rput(5.469375,2.0907812){PL(0)}
\usefont{T1}{ptm}{m}{n}
\rput(0.83921874,1.3507812){D(0)}
\usefont{T1}{ptm}{m}{n}
\rput(3.402969,1.3507812){VV(0)}
\usefont{T1}{ptm}{m}{n}
\rput(5.9078126,1.3507812){IN(0)}
\usefont{T1}{ptm}{m}{n}
\rput(3.4046874,0.48078126){NIN(0)}
\usefont{T1}{ptm}{m}{n}
\rput(0.92875,0.48078126){LO(0)}
\psline[linewidth=0.04cm,arrowsize=0.05291667cm 2.0,arrowlength=1.4,arrowinset=0.4]{->}(1.4973438,1.3607812)(2.8173437,1.3607812)
\psline[linewidth=0.04cm,arrowsize=0.05291667cm 2.0,arrowlength=1.4,arrowinset=0.4]{->}(4.037344,1.3607812)(5.3573437,1.3607812)
\psline[linewidth=0.04cm,arrowsize=0.05291667cm 2.0,arrowlength=1.4,arrowinset=0.4]{->}(1.5773437,1.2807813)(2.7373438,0.60078126)
\psline[linewidth=0.04cm,arrowsize=0.05291667cm 2.0,arrowlength=1.4,arrowinset=0.4]{<->}(4.197344,0.5207813)(5.4973435,1.1407813)
\usefont{T1}{ptm}{m}{n}
\rput(0.9946875,-1.3692187){ND(1)}
\usefont{T1}{ptm}{m}{n}
\rput(3.4229689,-1.3692187){VV(1)}
\usefont{T1}{ptm}{m}{n}
\rput(5.9278126,-1.3692187){IN(1)}
\usefont{T1}{ptm}{m}{n}
\rput(3.4246874,-2.2392187){NIN(1)}
\usefont{T1}{ptm}{m}{n}
\rput(1.0846875,-2.2392187){NLO(1)}
\psline[linewidth=0.04cm,arrowsize=0.05291667cm 2.0,arrowlength=1.4,arrowinset=0.4]{->}(4.057344,-1.3592187)(5.3773437,-1.3592187)
\psline[linewidth=0.04cm,arrowsize=0.05291667cm 2.0,arrowlength=1.4,arrowinset=0.4]{<->}(4.217344,-2.1992188)(5.5173435,-1.5792187)
\usefont{T1}{ptm}{m}{n}
\rput(1.3146875,-0.70921874){NVR(1)}
\usefont{T1}{ptm}{m}{n}
\rput(3.429375,-0.70921874){PI(1)}
\usefont{T1}{ptm}{m}{n}
\rput(5.389375,-0.70921874){PL(1)}
\psline[linewidth=0.04cm,linestyle=dashed,dash=0.16cm 0.16cm](1.1573437,-2.6492188)(1.1573437,-3.4092188)
\psline[linewidth=0.04cm,linestyle=dashed,dash=0.16cm 0.16cm](3.5373437,-2.6492188)(3.5373437,-3.4092188)
\psline[linewidth=0.04cm,linestyle=dashed,dash=0.16cm 0.16cm](5.7973437,-2.6492188)(5.7973437,-3.4092188)
\end{pspicture} 
}
\caption{Graphical representation of $AF_{amb}$.}
\label{fig:ambient}
\end{center}
\end{figure}

The corresponding argumentation framework $AF_{amb}$ is depicted in Figure \ref{fig:ambient}.
The relevant AF specification can be given as follows. 
Let $\Sigma=\set{F_1, F_2, F_3, F_4, 0}$ and $\XC=\set{F_1, F_2, F_3, F_4} \cup \set{0^i~:~i\geq1}$.
We can partition $\XC \setminus \set{F_1, F_2, F_3, F_4}$ into 10 sets $L_A$, $L_B$, $L_C$, $L_D$, $L_E$, $L_F$, $L_G$, $L_H$, $L_I$, $L_J$,(corresponding respectively to the 10 sequences NVR(i), PI(i), PL(i), VV(i), IN(i), NIN(i), D(2i), LO(2i), ND(2i+1), NLO(2i+1), with $i \geq 0$:
\[
\begin{array}{lclcl}
L_A&\mbox{ $=$ }&\set{0^{8i+1}~:~i\geq0}\\
L_B&\mbox{ $=$ }&\set{0^{8i+2}~:~i\geq0}\\
L_C&\mbox{ $=$ }&\set{0^{8i+3}~:~i\geq0}\\
L_D&\mbox{ $=$ }&\set{0^{8i+4}~:~i\geq0}\\
L_E&\mbox{ $=$ }&\set{0^{8i+5}~:~i\geq0}\\
L_F&\mbox{ $=$ }&\set{0^{8i+6}~:~i\geq0}\\
L_G&\mbox{ $=$ }&\set{0^{16i+7}~:~i\geq0}\\
L_H&\mbox{ $=$ }&\set{0^{16i+8}~:~i\geq0}\\
L_I&\mbox{ $=$ }&\set{0^{16i+15}~:~i\geq0}\\
L_J&\mbox{ $=$ }&\set{0^{16i+16}~:~i\geq0}
\end{array}
\]

The attack expression can then be formulated as follows:

$ a = tl(I \cap L_E) \cup (((I \cap L_F) \cdot 0) \cap L_G) \cup (((I \cap L_D) \cdot 000) \cap L_G)) \cup ((I \cap L_E) \cdot 0) \cup (tl(I \cap L_F) )$.

We observe that the attack expression satisfies the hypothesis of Proposition \ref{propn:finitary} hence it can be determined that the argumentation framework is finitary.
Algorithm \ref{algorithm:grounded} can then be applied and it can be verified that it terminates determining the grounded extension $G = \set{F_1, F_2, F_3, F_4} \cup L_A \cup L_B \cup L_C \cup (L_D \cap \set{0^{16i+12}~:~i\geq0}) \cup (L_E \cap \set{0^{16i+5}~:~i\geq0}) \cup (L_F \cap \set{0^{16i+14}~:~i\geq0}) \cup L_G \cup L_H \cup L_I \cup L_J$.
Using the method included in the proof of Theorem \ref{thm:afs-ops} it can also be verified that $G$ is stable, which implies that $G$ is also the unique preferred extension. 

From the argumentation perspective the situation is not pahological \emph{per se} and in fact this oscillating behavior is the desired one in case a person continuously enters and exits a room. Computing a compact representation of the grounded extension is however useful since it can be passed to a higher-level reasoning module which may detect the anomaly that the conclusions entailed by the system involve a person entering and exiting the same room let say every 5 seconds (or less).
It can also be observed that, in this case, the produced sequence of arguments is not actually infinite since the oscillating behavior will stop with the sunrise the morning after (or the semester after if we are in a polar winter). However, we are interested in analyzing (and stopping) such a very long sequence of arguments produced with a regular pattern well before it reaches its \virg{natural} termination. To this purpose it can be definitely more advantageous to treat it as an infinite sequence with compact representation rather than dealing explicitly with a finite sequence of thousands (if not millions) of \virg{machine-produced-always-the-same} arguments.

\section{Related Work}\label{sec-related}
Treatments of infinite {\sc af}s have, as already outlined, been largely limited to specific instances exemplifying
particular properties, e.g. that infinitary frameworks may occur naturally, as in the main example from \cite{dung:1995} presented in the previous section, or the issue of existence of semi-stable extensions \cite{camver:2010,Weydert:2011}. Beyond such examples the principal results have not advanced noticeably since the
general properties proven in \cite{dung:1995} were established. In particular, the question of \emph{computational} issues
in infinite {\sc af}s has not been considered.

At heart (interpretative matters aside) Dung's {\sc af} model is graph-theoretic (a property exploited in much
extant work on algorithmic and complexity treatments of {\sc af}s). The computational theory of infinite directed
graphs has, in contrast, long been recognised as a core area of graph theory, arguably dating back
to the beginning of the 20th century in the work of Thue~\cite{thue:1910}. Indeed, as observed by Morvan~\cite{Morvan:2000}:
``When dealing with computers, infinite graphs are natural objects''.

The idea of viewing vertex sets as a formal language with an edge relationship determined by
operations on words representing vertices dates back at least as far as Muller and Schupp~\cite{muller-schupper:1985} and
much of the focus of such computational treatments from a graph-theoretic perspective has
tended to concentrate on, what may loosely be termed, ``specification processes'' for generating
families of infinite graphs and model-theoretic treatments of logics defined via these processes.
Thus, Courcelle~\cite{Courcelle1989} addresses properties expressible in monadic second-order logic with respect
to bounded-width infinite graphs; Blumensath and Gr{\"a}del~\cite{BlumensathGradel:2004} consider model-theoretic issues
for properties expressible in first-order logic augmented with a quantifier, $\exists^{\omega}$, expressing
the existence of infinitely many objects within its scope. The ``reachability problem'' (given $u$ and $v$ is there
a directed path of edges from $u$ to $v$) of importance in analyses of program behaviour,
has been widely studied, e.g in Thomas~\cite{Thomas:2009} and Colcombet~\cite{Colcombet:2002}.

In these treatments, as well as in our own approach, the central concern is that of ``finite presentations
of infinite objects''  and so, unsurprisingly, the mechanisms adopted exhibit some structural similarities, e.g. in the use
of automata-theoretic models. Overall, however, the issues of interest differ: in particular, aside from specialised studies 
such as that of Bean~\cite{Bean1976} regarding colourings of
infinite graphs, properties impinging directly on graph-theoretic views of extension-based
semantics have not explicitly been dealt with.

Turning to another field related to argumentation, infinite structures have also received a significant deal of attention in the field of logic programming where admitting function symbols and recursion in the language gives rise to possibly infinite domains.
Hence, a significant gain in expressiveness has to be traded off with the possibility of actual implementation in practical solvers.
Focusing on the family of ASP (Answer Set Programming) solvers, Bonatti \cite{Bonatti04,Bonatti08} investigated the class of \emph{finitary} logic programs which admit unbounded (possibly infinite) domains and cyclic definitions while ensuring that inference is r.e.-complete.  Finitary logic programs are therefore amenable to implementation within existing ASP solvers with suitable extensions.
A larger class of logic programs with functions called \emph{finitely grounded} is shown to preserve most of the good properties of finitary programs in \cite{BaseliceBC07}.
Unfortunately the class of \emph{finitary} logic programs is undecidable: several subsequent works have then been devoted to investigate other classes of logic programs allowing functions, trading off expressiveness and tractability in various ways.
In \cite{SimkusEiter07,EiterSimkus10} a decidable class of disjunctive logic programs with function symbols under stable model semantics, called $\mathbb{FDNC}$, is introduced and a method is provided to finitely represent all the (possibly infinite) stable models of a given $\mathbb{FDNC}$ program.
In \cite{Calimerietal08,Calimerietal10} the semi-decidable class of \emph{finitely ground} programs is considered, along with its decidable subclass of \emph{finite domain} programs, 
while another decidable subclass, called \emph{argument restricted}, has been analyzed in \cite{LierlerLifschitz09}. Further, a decidable subclass of finitary programs, called \emph{FP2}, has been recently presented in \cite{BaseliceBonatti10}.
On the implementation side, the DLV solver \cite{Leoneetal06} has been extended to encompass the treatment of finitely ground and finite domain programs resulting in a publicly available system called DLV-complex \cite{Calimerietal09}.

The above studies witness a large interest in reasoning with infinite domains in answer set programming, with a range of motivations including the explicit treatment of recursive data structures like lists and trees, the encoding of problems not admitting a priori bounds on the solution size (e.g. planning or reasoning about actions), and the consideration of potentially infinite processes in time (a biology-inspired example is provided in \cite{EiterSimkus10}).
While many of the above needs are common to argumentation theory (and more generally to any approach to defeasible reasoning, as remarked in Section \ref{section:motivations}) it has to be acknowledged that the significant advancements both on the theoretical and on the application side surveyed above have no counterpart (yet) in the argumentation field, so that the useful connections and interplay between the two fields have definitely to be regarded as a future research subject.
As far as the present work is concerned, it can be remarked in particular that the investigations surveyed above lie at the level of the representation language, which is abstracted away in Dung's framework, hence our work concerns a different, and not directly comparable, abstraction level. 
Moreover the above works are based on the stable model semantics adopted in the context of ASP solvers, while the approach proposed in this paper is not committed to a specific semantics choice and hence is applicable beyond the limits of the stable semantics, which, as well-known, does not always guarantees the existence of extensions (the existence of models in the logic programming context) and does not feature, in general, some desirable properties like directionality or relevance (see \cite{Baronietal11} for a discussion).

%
%

\section{Further Work and Conclusions}
\label{sec-conclusions}
Our main aim in this paper has been to present a formal approach to describe both finite
and infinite {\sc af} structures, the argument set being the set of words
within some regular language, $\XC\subseteq\Sigma^{*}$, and the attack relation, $\AC$ over $\XC\times \XC$ being
given through a sentence, $a\in\AC\EC(\Sigma)$ constructed by a limited set of operations
so that for $S\subseteq\XC$, $\underline{a}(S)$ satisfies additivity (hence also
monotonicity) and preserves regularity.
We provided some illustrations of the flexibility of our approach using examples from \cite{dung:1995,camver:2010}.
More generally, the approach has been shown to be able to capture standard finite {\sc af}s and arbitrary finite combinations of finite and infinite {\sc af}s, which can reasonably be regarded as covering most (if not all) situations of practical interest.

A related research line we are developing in parallel concerns the use of this kind of techniques to represent infinite structures in extended versions of Dung's framework, some initial results concerning the {\sc afra} formalism (Argumentation Framework with Recursive Attacks) having been recently obtained \cite{baroni-et-al-tafa11}.

We have concentrated on the expressive potential of {\sc afs}, indicating that, in contrast to ``naive'' encodings, processes which can be dealt with efficiently in the finite
setting -- deciding conflict-freeness, admissibility, acceptability, verifying whether a set is a stable or complete
extension as well as construction problems such as computing the characteristic function
-- all admit effective decision methods and algorithms for building automata
accepting the corresponding sets, even when the instances being checked or the results reported
are themselves infinite subsets of $\XC$. For the case of two problems, -- existence of stable extensions and determining credulous acceptance wrt preferred semantics -- unlikely to be efficiently decidable
in the finite context we have shown that 
within {\sc afs} these are (at worst) semi-decidable. 

We conclude by reviewing some topics meriting further development, a number of which are the
subject of current work. One such immediate area of interest concerns the \emph{efficiency} with which
particular procedures can be implemented (as opposed to the issue of \emph{effectiveness}). While some
preliminary study of such questions is underway we have chosen, partly for reasons of space, not to
develop this aspect in detail within the current paper. We note that such questions concern two elements: the \emph{size} (i.e. number
of states) of automata achieving particular tasks, and the computational complexity of the problems themselves.
The former, referred to as \emph{state complexity} in the associated literature has been widely
studied\footnote{Important contributions may be found in \cite{Moore:1971,Birget:1992,Jiraskova:2005,PS:2002,YuZhuangSalomaa:1994}.}
so that tight bounds on state complexity delineating the number of states necessary and sufficient for an automaton
accepting $R\theta S$ or $\theta(R)$, in terms of the state complexity of the languages $R$ and $S$ have
been obtained for each of the principal operations $\theta$ with each of the finite automaton forms discussed. It is, clearly,
the case that the extent to which, say, $\pi_{a}^{-}(S)$, may be recognised by a ``small'' automaton
will depend significantly not only on the state complexity of $S$ itself, but also on the
exact specification of $a\in\AC\EC(\Sigma)$. As such it would seem unlikely that a completely general
treatment of state complexity for $\AC\EC(\Sigma)$ (even if such is possible) will yield results of
much interest since this generality is likely to overestimate state complexity for those cases that might arise in practice. 
A rather more promising approach is to consider sub-classes of $\AC\EC(\Sigma)$ obtained by constraining
the operational structures, e.g. given some \emph{finite} ``base language'', $B\subseteq\Sigma^{*}$ consider
attack structures, $\underline{a}$ satisfying ``$v\rightarrow_{a}w$ only if $v\in B$ or $r(|v|,|w|)$'' (so that
$r(...)$ is determined through some aspect of the lengths of $v$ and $w$). In fact preliminary results of
the authors, with $B\subseteq\Sigma^{2}$ and the constraint
``$v\rightarrow_{a} w$ iff ($v\in B$ and $tl(B)=hd(w)$) or ($v\in\Sigma\cdot w)$'' indicate, using a careful
treatment of the {\sc dfa} form accepting $\XC$ that all of the cases shown to be
effectively computable in Thm. \ref{thm:afs-ops} 
may be efficiently implemented
(in terms of state complexity and polynomial run-time).\footnote{We remark that the {\sc afra} structures described
in \cite{baroni-et-al-tafa11} are a special case of this restriction: illustrative efficient automata
constructions (in terms of both state complexity and algorithm run-time) have been obtained.}

A further topic of some interest concerns the use of our approach in \emph{finite} frameworks. Although it is, of
course, unnecessary to resort to {\sc afs} schema to describe finite $\tuple{\XC,\AC}$ there are, however,
cases where it may be advantageous to do so. For example, suppose $|\XC|=2^m$ for some $m\in\textbf{N}$, then using $\Sigma~=~\set{0,1}$
the set $\XC$ can be viewed as $\set{0,1}^m$ a language that is accepted by a {\sc dfa} with exactly $m+1$ states. Thus,
for suitable $\AC\subseteq\XC\times\XC$ the {\sc af}, $\tuple{\XC,\AC}$ rather than requiring a description whose size is $O(2^m)+|\AC|$ could be presented by one whose size is $O(m^k)$ for some $k\in\textbf{N}$. In cases where
such compaction can be achieved, an important issue is the resulting cost of implementing standard
decision procedures: an obvious concern is that, for this particular finite case, some subsets of $\set{0,1}^m$ will require
automata whose state complexity is $2^{O(m)}$. It is, however, unclear whether this behaviour would be
the \emph{only} potential drawback, e.g. what can be said regarding the complexity of $\mbox{{\sc ca}}_{adm}$ (for single, rather
than sets of arguments) in such settings?

As a final collection of problems we note that several issues remain open concerning \emph{effective} decision
processes for extension-based semantics in {\sc af}s. In particular, although we have shown questions
such as $\EC_{stab}=\emptyset$ to be semi-decidable (in finitary {\sc af}s), the status of its converse is open, i.e. is it the
case that $\EC_{stab}\not=\emptyset$ is semi-decidable? A positive answer would, of course, lead to
an effective procedure for $\mbox{{\sc exist}}_{stab}$, while a negative answer motivates the question of identifying
decidable fragments of $\AC\EC(\Sigma)$.

On a different side, it has to be acknowledged that {\sc afs} is not an immediately usable formal tool and that the specification of each example has been crafted individually. In perspective, {\sc afs} can be regarded as a \virg{low level} language which can represent the basis for the definition of higher level constructs for the description of infinite {\sc af}s, accompanied by suitable methodologies for their application. In fact, some recurrent structural and representation patterns can be identified in the examples considered in the paper and procedures to derive {\sc afs} from logic programs could be considered, but a full investigation of these issues is left for future work.

In conclusion we emphasise once more that the development put forward in this paper, while establishing
many cases where an effective treatment of infinite argumentation forms is realistic, provides a 
starting point for a wider investigation of this matter.

\appendix

\section{Formal Languages and Automata}\label{subsection:flt}

A standard approach to the problem of representing an infinite collection of objects
via a finite specification is to exploit so-called formal grammars and their associated
machine models. In this section we review some basic elements and results from this discipline to complement the basic definitions given in Section \ref{subsection:back-flt}.

Given a formal grammar, $G=\tuple{\Sigma,V,P,S}$ (Definition \ref{defn:fg}), and a production
rule $\alpha\rightarrow\beta\in P$, for all $\gamma,~\delta\in(V\cup\Sigma)^{*}$ we say that $\gamma\alpha\delta$
\emph{derives} $\gamma\beta\delta$ in $G$ ($\gamma\alpha\delta\Rightarrow_{G}\gamma\beta\delta$)
and in general $u\Rightarrow_{G}^{*} w$ whenever there is a finite sequence of derivations such that
\[
u\Rightarrow_Gu_1\Rightarrow_Gu_2\Rightarrow_G\cdots\Rightarrow_Gu_k\Rightarrow_Gw
\]
A derivation $u\Rightarrow_Gw$ is \emph{terminated} if $w\in\Sigma^{*}$. The \emph{language generated} by $G=\tuple{\Sigma,V,P,S}$, 
denoted as $L(G)$, is
\[
L(G)~~=~~\set{~w\in\Sigma^{*}~:~S\Rightarrow_{G}^{*}w}
\]
A language, $L\subseteq\Sigma^*$, is \emph{recognisable} if there is a formal grammar $G=\tuple{\Sigma,V,P,S}$ for which
$w\in L$ if and only if $w\in L(G)$.

Notice that, in general, formal grammars provide a process for proving that $w\in L(G)$ and that there is not, necessarily, a \emph{unique}
sequence of derivations under which $S\Rightarrow_{G}^{*}w$.


\begin{defn}
A grammar $\tuple{\Sigma,V,P,S}$ is \emph{unrestricted} if $P$ is allowed to contain arbitrary rules $\alpha\rightarrow\beta$ (subject to the
constraint that $\alpha\not\in\Sigma^{*}$). It is \emph{context--sensitive} if $\forall~\alpha\rightarrow\beta\in P$ we have $|\beta|\geq|\alpha|$;
\emph{context--free} if $\forall~\alpha\rightarrow\beta\in P$ we have $\alpha\in V$ and \emph{right--linear} if every
$\alpha\rightarrow\beta\in P$ has the form $V_i\rightarrow\varepsilon$ or $V_i\rightarrow\sigma$ or $V_i\rightarrow\sigma V_j$ for
$V_i$, $V_j\in V$ and $\sigma\in\Sigma$.
\end{defn}

\label{effectiveness}
Recall that a language $L$ is
\emph{recursively enumerable} (r.e.) if there is a Turing machine (TM) program, $M$, that given any $w\in L$ as input
will eventually halt and accept $w$; with $L$ being \emph{recursive} if there is a TM program, $M$, that given any $w\in\Sigma^{*}$ as input
eventually halts and accepts any $w\in L$ and halts and rejects any $w\not\in L$. We use the term \emph{decidable} to describe
languages which are recursive and \emph{semi-decidable} for those which are recursively enumerable. The term \emph{effective algorithm} for $L$
will be used for an algorithmic process, e.g. a Turing machine program, that witnesses $L$ as 
decidable.\footnote{It should be noted that, some closure properties are established non-constructively so that effective
algorithms yielding machines recognising the resulting language do not necessarily follow, see e.g. \cite[pp.~62--63]{HU:1979}.}

\begin{fact}\label{fact:correspondence}
$\ \ \ $
\begin{enumerate}
\item[a.]
$L\subseteq\Sigma^{*}$ is \emph{recursively enumerable} if and only if there is an unrestricted grammar, $G$ such that
$L(G)=L$.
\item[b.]
$L\subseteq\Sigma^{*}$ is \emph{recursive} if and only if there are unrestricted grammars, $G_1$ and $G_2$ such that
$L(G_1)=L$ and $L(G_2)=\overline{L}$, i.e. $L(G_2)=\Sigma^{*}\setminus L$.
\end{enumerate}
\end{fact}

It is well known that there are languages that fail to be r.e.

Regular languages (Definition \ref{defn:rl}) are captured by a syntactic formalism called \emph{regular expressions}. A regular expression, $E$ over $\Sigma$ is
constructed by a finite number of applications of the following

\[
\left\{
{
\begin{array}{ll}
\emptyset&\mbox{ is a regular expression}\\
\varepsilon&\mbox{ is a regular expression}\\
\sigma&\mbox{ }\forall~\sigma\in\Sigma\mbox{ is a regular expression}\\
(R+S)&\mbox{ for regular expressions $R$, $S$}\\
R\cdot S&\mbox{ for regular expressions $R$, $S$}\\
(R)^*&\mbox{ for a regular expression $R$}
\end{array}
}
\right.
\]

The associated regular languages $L(E)\subseteq\Sigma^{*}$ being,
\[
\left\{
{
\begin{array}{ll}
\emptyset&\mbox{ if $E=\emptyset$}\\
\set{\varepsilon}&\mbox{ if $E=\varepsilon$}\\
\set{\sigma}&\mbox{ if $E=\sigma$}\\
L(R)\cup L(S)&\mbox{ if $E=(R+S)$}\\
L(R)\cdot L(S)&\mbox{ if $E=R\cdot S$}\\
L(R)^{*}&\mbox{ if $E=(R)^{*}$}
\end{array}
}
\right.
\]
In order to reduce notational complications we will, in general, equate a regular expression, $R$, with the language, $L(R)$,
it describes, thus writing $R$ for both cases. Where no ambiguity arises, we dispense with superflous parentheses.

\begin{fact}\label{fact:closure-props}
Let $\mbox{{\sc reg}}\subseteq 2^{\Sigma^{*}}$ be the property describing all regular languages, i.e. $L\in\mbox{{\sc reg}}$ if and
only if $L$ is a regular language. The class {\sc reg} is closed with respect to \emph{all} of the operations $\theta\in\set{\cup,~\cap,\overline{\set{~~}},\setminus,\cdot,^{*},/,rev}$.
\end{fact}

The class of machine models that express exactly the regular languages are the \emph{deterministic finite automata} (Definition \ref{defn:dfa}), other classes of finite automata can also be considered.

\begin{defn} \label{defn:epsilon-ndfa}
A {\em non-deterministic finite automaton} ({\sc ndfa}) $NM=\tuple{\Sigma,Q,q_0,F,\delta}$ has $\delta~:~Q\times\Sigma\rightarrow2^{Q}$, indicating
that in some states and symbols there may be more than one ``next'' state (or even that no state
at all can be reached should $\delta(q,\sigma)=\emptyset$). An $\varepsilon$-{\sc ndfa} has a state transition function
$\delta~:~Q\times\Sigma\cup\set{\varepsilon}\rightarrow2^{Q}$ where the interpretation of $\delta(q,\varepsilon)=Q'\subseteq Q$ is
that having reached state $q$ the automaton may process its next input symbol $\sigma\in\Sigma$ from $q$ itself \emph{or} from 
any state in $\delta(q,\varepsilon)$. We identify a sub-class, the so-called ``$\varepsilon$--{\sc dfa}'' of
$\varepsilon$-{\sc ndfa} via those whose transition function satisfies: $\delta~:~Q\times\Sigma\rightarrow Q$
and $\delta~:~Q\times\set{\varepsilon}\rightarrow2^Q$, i.e. $\varepsilon$-{\sc dfa} specify \emph{exactly one} successor
state for each $q\in Q$ and $\sigma\in\Sigma$ but can allow arbitrary $\varepsilon$ transitions between states.

For a {\sc ndfa}, $M$, $w=w_k\cdot w_2\cdots w_1\in\Sigma^{*}$ is accepted by $M$, written $w\in L(M)$
if there is \emph{at least one} sequence of states $q_{i_1} q_{i_2}\ldots q_{i_{k}}$ such that
$q_{i_1}\in\delta(q_0,w_1)$, $q_{i_j}\in\delta(q_{i_{j-1}},w_j)$ for $2 \leq j\leq k$ and $q_{i_k}\in F$.
For $\varepsilon$-{\sc ndfa} $w=w_k\cdot w_{k-1}\cdots w_1\in\Sigma^{*}$ is accepted by $M$,
if there is a \emph{finite} sequence $q_{j_1} q_{j_2}\ldots q_{j_{r}}$ of states with $r\geq k$ and a
finite sequence $\alpha_1\alpha_2\ldots\alpha_r$ with $\alpha_i\in\set{\varepsilon}\cup\Sigma$ such that:
$\alpha_r\cdot\alpha_{r-1}\cdots\alpha_1=w$, 
$q_{i_1}\in\delta(q_0,\alpha_1)$, $q_{i_j}\in\delta(q_{i_{j-1}},\alpha_j)$ for $2\leq j\leq r$, and $q_{j_r}\in F$.
The concept of acceptance by $\varepsilon$-{\sc dfa} is defined similarly.
\end{defn}

\begin{fact}\label{fact:reg-fa}
For $L\subseteq\Sigma^{*}$ the following are equivalent.
\begin{enumerate}
\item[a.]
$L$ is a regular language.
\item[b.]
There is an $\varepsilon$-{\sc ndfa}, $M$, for which $L(M)=L$.
\item[c.]
There is an $\varepsilon$-{\sc dfa}, $M$, for which $L(M)=L$.
\item[d.]
There is a {\sc ndfa}, $M$, with $L(M)=L$
\item[e.]
There is a {\sc dfa}, $M$ with $L(M)=L$.
\item[f.]
There is a right-linear grammar, $G$, for which $L(G)=L$.
\end{enumerate}
\end{fact}
\begin{fact}\label{fact:properties}
$\ \ \ $
\begin{enumerate}
\item[a.]
Given any finite automaton ({\sc dfa}, {\sc ndfa}, $\varepsilon$-{\sc dfa} or $\varepsilon$-{\sc ndfa}), 
$\MC=\tuple{\Sigma,Q,q_0, F, \delta}$, it may decided in polynomial time (in $|Q|+|\Sigma|$)
if $L(\MC)=\emptyset$.
\item[b.]
Given two {\sc dfa}s accepting languages $L_1$ and $L_2$ there are effective algorithms for constructing a {\sc dfa} accepting $L_1 \cap L_2$, $L_1 \cup L_2$, $L_1 \setminus L_2$, $L_1 / L_2$.
\item[c.]
Every regular language $L\subseteq\Sigma^{*}$ has a \emph{unique}\footnote{``Uniqueness'' is modulo relabelling states
of the automaton.} minimal number of states {\sc dfa}, $M$ for which $L(M)=L$. Furthermore, given $M'$ with $L(M')=L$
the unique minimized automaton, $M$ with $L(M)=L(M')=L$ may be constructed in polynomial time in
$|Q_{M'}|+|\Sigma|$.

\end{enumerate}
\end{fact}

Fact~\ref{fact:properties}~(c) is the Myhill-Nerode Theorem \cite{Nerode:1958}, a polynomial time algorithm
for constructing the minimal automaton may be found in \cite[Ch.~3.4]{HU:1979}; the most efficient
(currently known) algorithm is that of Hopcroft~\cite{Hopcroft:1971}
which takes at most $O(|Q|\log|Q|)$ steps to minimise a {\sc dfa} with $|Q|$ states.

\section{Proofs}
\label{appendix-proofs}

\subsection{Proofs of Section \ref{section:naive-rep}}\label{proofs-naive}

\noformalgrammar*
\begin{proof}
It is well-known that Turing machine programs may be encoded as words in $\set{0,1}^{*}$ in such
a way that if $\beta(M)$ is the encoding of some TM, $M$, then there is a (so-called \emph{universal}) TM
which given the pair $\tuple{\beta(M),w}$ as input, exactly simulates the computational steps of $M$
on input $w$.\footnote{See, for example, Hopcroft and Ullman~\cite[Chap.~8.3]{HU:1979} or Dunne~\cite[Chap.~4]{Dunne:1991}
for example constructions of such universal TMs.} Furthermore it can be decided if any $w\in\set{0,1}^{*}$
is such that $w=\beta(M)$ for some TM program $M$. For any such encoding scheme we may use
the standard \emph{lexicographic ordering}\footnote{That is, the total ordering $<_{{\rm lex}}$ in which
$0<_{{\rm lex}}1$, $w<_{{\rm lex}}u$ if $|w|<|u|$, and, when $|w|=|u|$
$w<_{{\rm lex}}u$ if $w=0\cdot v$, $u=1\cdot x$ or (when $w=\alpha\cdot v$ and $u=\alpha\cdot x$) if $v<_{{\rm lex}}x$.}
of $\set{0,1}^{*}$ to order TM programs, so that
\[
\begin{array}{lcl}
\beta(M_i)&\mbox{ $=$ }&\mbox{The $i$'th word in the lexicographic ordering of $\set{0,1}^*$}\\
&&\mbox{ that describes a valid TM encoding.}
\end{array}
\]
Finally, we recall that there is no formal grammar, $G$, that generates the following language:
\[
\mbox{{\sc non-halt-empty}}~=~\set{\beta(M)~:~\mbox{$M$ does \emph{not} halt given $\varepsilon$ as input}}
\]
We can now define the language $L_{\AC}$ with the property required via
\[
L_{\AC}~~=~~~\set{ 0\cdot1\cdot0^{k}~:~\beta(M_k)\in\mbox{{\sc non-halt-empty}}}
\]
From which it follows that a grammar $G$ with $L(G)=L_{\AC}$ allows a grammar $G_{{\neg\varepsilon-halt}}$
with $L(G_{{\neg\varepsilon-halt}})=\mbox{{\sc non-halt-empty}}$ to be built.
\end{proof}

\unrestrictednondecidable*
\begin{proof}
Immediate consequence of Rice's Theorem for r.e. Index Sets,~\cite{Rice:1956}, see e.g. \cite[pp.~189--192]{HU:1979} or
\cite[pp.~57--66]{Dunne:1991}.\footnote{Rice's Theorem for r.e. Index Sets characterises those ``properties'' of TMs (equivalently,
formal grammars) that are semi-decidable. It is trivial to show that grammars generating subsets
of $\set{0^i\cdot1\cdot0^j~:~i,~j\geq 1}$ fail to meet the conditions for a property to be semi-decidable.}
\end{proof}

\contextsensitivenonsemidecidable*
\begin{proof}
The problem of determining if $L(G)=\emptyset$ for an arbitrary context-sensitive grammar (over alphabet $\set{0,1}$)
is not semi-decidable.
Given a context-sensitive grammar $G$ over $\set{0,1}$ we construct a context-sensitive grammar $G'$ over $\set{0,1}$
with the property that $L(G')\subseteq\set{0^i\cdot1\cdot0^j~:~i,~j\geq 1}$ if and only if $L(G)=\emptyset$. Let $S$
be the start symbol of $G$. Add a new starting symbol $S'$ to $G$ with a single production
$S'\rightarrow1\cdot S$ to give the new grammar $G'$. Then given that any word actually generated
by $G'$ must begin with the symbol $1$, the only way in which $L(G')\subseteq\set{0^i\cdot1\cdot0^j~:~i,~j\geq 1}$
would be if $L(G')=\emptyset$. This can only be the case if $L(G)=\emptyset$ to begin with.
\end{proof}

\contextfreedecidable*

\begin{proof}
First note that $L_{010}=\set{0^i\cdot1\cdot0^j~:~i,~j\geq 1}$ is a regular language, and hence its complement $\overline{L_{010}}$ is a regular language too.
Now, given $G$ a context-free grammar over the alphabet $\set{0,1}$, checking $L(G)\subseteq L_{010}$ is equivalent to check $L(G) \cap \overline{L_{010}} = \emptyset$.
It is well-known \cite{HU:1979} that the intersection of a context-free language (in our case $L(G)$) with a regular language (in our case $\overline{L_{010}}$) is a context-free language, whose specification can be constructed from those of $L(G)$ and $\overline{L_{010}}$.
The conclusion then follows from the fact that verifying the emptyness of the language generated by a context-free grammar can be done in polynomial time \cite{HU:1979}.
\end{proof}
\dfadecidable*
\begin{proof}
The {\sc dfa}, $M$, accepts a subset of $\set{0^i\cdot1\cdot0^j~:~i,~j\geq 1}$ if and only if
$L(M)\setminus\set{0^i\cdot1\cdot0^j~:~i,~j\geq 1}=\emptyset$. Noting the language $\set{0^i\cdot1\cdot0^j~:~i,~j\geq 1}$
is regular and that for {\sc dfa}s, $M$, $M'$
a {\sc dfa} accepting exactly $L(M)\setminus L(M')$ may be constructed in polynomial time, the proof is completed
by observing that $L(M)=\emptyset$ is also decidable in polynomial time for any given {\sc dfa}.
\end{proof}

\limiteddfa*
\begin{proof}
From the Pumping Lemma for regular languages, cf. \cite[Chap.~3.1]{HU:1979}, with any regular language, $L$, there is
an associated constant, $K_L$, such that: for all $w\in L$, with $|w|\geq K_L$, $w=x\cdot y\cdot z$ with $|x\cdot y|<K_L$, $|y|\geq1$
and $x\cdot y^t\cdot z\in L$ for all $t\geq 0$. Thus proceeding by contradiction
it suffices to consider some $w=0^r\cdot1\cdot0^s\in L$ with $r\geq K_L$: note that the existence of a suitable $w$ is
guaranteed by the premise that there are infinitely many distinct values of $k$ for which $0^k\cdot1\cdot0^m\in L$.
Now, since by the condition $|x\cdot y|<K_L \leq r$ $1$ does not belong to $x \cdot y$, we can write $w=x\cdot y\cdot 0^{a} \cdot 1 \cdot 0^{s}=0^p\cdot0^q\cdot1\cdot0^s$ with $q = |y|$, $p = |x| + a$, and $p+q<K_L$.  It follows that all words of the form
$0^{p+tq}\cdot1\cdot0^s$ are in $L$ for all $t\geq0$. Now choosing $t$ so that $p+tq>s$ yields a word which violates
the conditions for membership in $L$.
\end{proof}

\subsection{Proofs of Section \ref{section:sec-generic-representation}}
\label{sec:proof-sect-representation}

\thmreasonablemapping*

\begin{proof}
Let $p$ be an attack expression over $\Sigma$.
We proceed by induction on $size(p)\geq0$.

The inductive base case involves $p\in\set{\sigma_1,\ldots,\sigma_k,I}$. First observe 
that $\underline{p}$ in these cases satisfies the additivity requirement (R1) of Defn.~\ref{defn:reasonable-attack}
since for any $S\subseteq\Sigma^{*}$ we have
$\underline{p}(S)~\in\set{\sigma_1,\ldots,\sigma_k,S}$ and for each $w\in S$,
$\underline{p}(\set{w})~\in\set{\sigma_1,\ldots,\sigma_k,\set{w}}$. Hence in the case $p=\sigma$ we obtain
\[
\underline{p}(S)~~=~~\set{\sigma}~~=~~\bigcup_{w\in S}~\set{\sigma}~~=~~\bigcup_{w\in S}~\underline{p}(\set{w})
\]
whereas for $p=I$ we have
\[
\underline{p}(S)~~=~~S~~=~~\bigcup_{w\in S}~\set{w}~~=~~\bigcup_{w\in S}~\underline{p}(\set{w})
\]
Finally, since $S$
is assumed regular to begin with, for each of the base case possibilities, we have $\underline{p}(S)$ is also regular.

Now inductively assume for some $k>0$ and all attack expressions over $\Sigma$, $q$, with $size(q)<k$ the
mapping given via $\underline{q}$ is a reasonable attack function. Consider any attack expression, $p$, over $\Sigma$
for which $size(p)=k$. Since $size(p)>0$ its construction must involve (at least) one of the operations
from $\set{\cup,K_{\Sigma}\cdot,\cdot K_{\Sigma},/K_{\Sigma},K_{\Sigma}/,\cap K_{\Sigma},hd,tl}$. We consider these in turn.

If $p=q\cup r$ then $\underline{p}(S)=(\underline{q}(S)\cup\underline{r}(S))$. By the inductive hypothesis $\underline{q}$ and $\underline{r}$ are both reasonable, hence since $\cup$ preserves both the properties (R1) and (R2) it follows that $\underline{p}$ is reasonable.

If $p=K_{\Sigma}\cdot q$ for some regular subset $K_{\Sigma}$ of $\Sigma^{*}$
\[
\underline{p}(S)~=~K_{\Sigma}\cdot\underline{q}(S)~=~\left({~\bigcup_{u\in K_{\Sigma}}~u\cdot\underline{q}(S)~}\right)
\]
where, from the inductive hypothesis, $\underline{q}$ is reasonable. Thus
\[
\bigcup_{w\in S}~\underline{p}(\set{w})~=~
\bigcup_{w\in S}~\left({~\bigcup_{u\in K_{\Sigma}}~u\cdot\underline{q}(\set{w})~}\right)
\]
which is
\[
\left({~\bigcup_{u\in K_{\Sigma}}~u\cdot\underline{q}(S)~}\right)
\]
by the additivity of $\underline{q}$. Again (R2) holds by virtue
of the fact that $K_{\Sigma}\cdot$ preserves regularity.

The argument for $p=q\cdot K_{\Sigma}$ is similar.

If $p=q/K_{\Sigma}$ then:
$\underline{p}(S)~=~\underline{q}(S)/K_{\Sigma}$, which, by the additivity of $\underline{q}$, is equivalent to
\[
\left({\bigcup_{w\in S}~q(\set{w})}\right)/K_{\Sigma}~=~\bigcup_{w\in S}~q(\set{w})/K_{\Sigma}~=~\bigcup_{w\in S}~\underline{p}(\set{w})
\]
so that again $\underline{p}$ is additive from the fact that $\underline{q}$ is additive. The closure property is again easily
verified.

The case $p=K_{\Sigma}/q$ is similar.

If $p=q\cap K_{\Sigma}$ then $\underline{p}(S)$ is
\[
\underline{q}(S)\cap K_{\Sigma}~~=~~\left({\bigcup_{w\in S}~\underline{q}(\set{w})}\right)~\cap~K_{\Sigma}
\]
and
\[
\left({\bigcup_{w\in S}~\underline{p}(\set{w})}\right)~~=~~
\left({\bigcup_{w\in S}~\underline{q}(\set{w})~\cap~K_{\Sigma}}\right)~~=~~\left({\bigcup_{w\in S}~\underline{q}(\set{w})}\right)~\cap~K_{\Sigma}
\]
so that again additivity holds. Closure is trivially established.

For $p=hd(q)$, we get
\[
\underline{p}(S)~~=~~hd(\set{~w~:~w\in\underline{q}(S)})~~=~~\bigcup_{u\in S}~hd(\set{~w~:~w\in\underline{q}(\set{u})}~~=~~\bigcup_{u\in S}~\underline{p}(\set{u})
\]
using additivity of $\underline{q}$ for the second equality.

To see that (R2) holds it suffices to note that (given a regular language $S$)
$hd(\underline{q}(S))\subseteq\Sigma$ is finite and hence trivially a regular language.

For $p=tl(q)$:
\[
\underline{p}(S)~~=~~ tl(\set{~w~:~w\in\underline{q}(S)})~~=~~ 
\set{~u~: \exists \sigma \in \Sigma \mbox{ s.t. }~\sigma\cdot u\in\underline{q}(S)}~~= 
\]
\[
\bigcup_{w\in S}~\set{~u~:~\sigma\cdot u\in\underline{q}(\set{w})~}~~=~~\bigcup_{w\in S}~\underline{p}(\set{w}) 
\]
using the additivity of $\underline{q}$ for the third equality. 


It remains to show $\underline{p}(S)$ is regular if $S$ is so. Consider a {\sc dfa}, $\MC_{q}$,
accepting $\underline{q}(S)$ -- such a {\sc dfa} being guaranteed by the fact that $\underline{q}(S)$ is regular.
In order to build a {\sc dfa} accepting $tl(\underline{q}(S))$ it suffices to replace its accepting
states, $\FC_{q}$ by $\set{~r~:~\delta(r,\sigma)\in\FC_{q}}$. We deduce that $tl(q)$ gives rise
to $\underline{p}$ satisfying R2 thus completing the inductive argument.
\end{proof}

\propreasonablemapping*
\begin{proof}

Let $\XC\subseteq\Sigma^{*}$ and $\mu~:~2^{\Sigma^{*}}\rightarrow2^{\Sigma^{*}}$ be a reasonable mapping. First note
that $\mu_{\XC}~:~2^{\XC}\rightarrow2^{\XC}$ is additive since for any $S\subseteq\XC$ we have $\mu_{\XC}(S)$ equal to
\[
\mu(S)\cap\XC~=~\left({\bigcup_{w\in S}~\mu(\set{w})}\right)\cap\XC~=~\bigcup_{w\in S}\mu(\set{w})\cap\XC~=~
\bigcup_{w\in S}~\mu_{\XC}(\set{w})
\]
Finally, that $\mu_{\XC}$ preserves regularity for regular subsets $S$ of $\XC$ is immediate from $\mu_{\XC}(S)=\mu(S)\cap\XC$.
\end{proof}

\thmclosureregular*
\begin{proof}

Consider the various forms that $a\in\AC\EC(\Sigma)$ may have. We show by induction on $size(a)$ that if $S$ is a regular language then
$\inv{a}(S)$ is a regular language too.

\noindent
\textbf{Case 1.} $size(a)=0$ (Inductive base - Fact \ref{conjectures-inverse}.\ref{Prop0})

In this case, $a\in\set{\sigma_1,\ldots,\sigma_k,I}$. 

For $a=\sigma_i$,
\[
\inv{a}(S)~~=~~\left\{
\begin{array}{lcl}
\Sigma^{*}&\mbox{ if }&\sigma_i\in S\\
\emptyset&\mbox{ if }&\sigma_i\not\in S
\end{array}
\right.
\]
We note, in view of the properties $tl(\emptyset)=\emptyset$, $tl(\set{\sigma})=\varepsilon$ and $\emptyset\cdot L=\emptyset$,
that
$\inv{a}(S)=tl(S\cap\set{\sigma_i})\cdot\Sigma^{*}$, i.e. we do not need to explicitly represent the conditional behaviour, thus allowing one to express $a^{+}$ as $tl(I \cap \sigma_i)\cdot\Sigma^{*}$. 

For $a=I$, $\inv{a}(S)=S$.
Thus in each case $\inv{a}(S)$ is a regular language.

Assuming for each $a$ with $size(a)<k$, that $\inv{a}(S)$ is regular, consider $a\in\AC\EC(\Sigma)$ with $size(a)=k$.

\noindent
\textbf{Case 2.1} $a=b\cup c$ (Fact \ref{conjectures-inverse}.\ref{Prop1})

Then
\[
\begin{array}{lcl}
\inv{a}(S)&\mbox{ $=$ }&\set{~v\in\Sigma^{*}~:~\exists u\in S~\mbox{ s.t. }u\in \underline{b}(\set{v})\mbox{ or }u\in\underline{c}(\set{v})}\\
&\mbox{ $=$ }&\set{~v\in\Sigma^{*}~:~\exists u\in S~\mbox{ s.t. }u\in \underline{b}(\set{v})}~\cup~\\
&&\set{~v\in\Sigma^{*}~:~\exists u\in S~\mbox{ s.t. }u\in\underline{c}(\set{v})}\\
&\mbox{ $=$ }&\inv{b}(S)\cup\inv{c}(S)
\end{array}
\]
Via the inductive hypothesis and the closure properties of regular languages (see Fact \ref{fact:closure-props} in \ref{subsection:flt}), this is a regular language.

\noindent
\textbf{Case 2.2} $a=b\cdot K_{\Sigma}$ (Fact \ref{conjectures-inverse}.\ref{Prop2})
\[
\inv{a}(S)~~=~~\set{~v\in\Sigma^{*}~:~\exists u\in S\mbox{ s.t. }u\in \underline{b}(\set{v})\cdot K_{\Sigma}}
\]
Recall that the \emph{quotient} of a language $L_1$ wrt $L_2$ (denoted $L_1/L_2$) is
\[
L_1/L_2~~=~~\set{~p~:~\exists~q\in L_2~\mbox{ s.t. }p\cdot q\in L_1}
\]
It is easily seen that for $a=b\cdot K_{\Sigma}$ this leads to 
\[
\begin{array}{lcl}
\inv{a}(S)&\mbox{ $=$ }&\set{~v\in\Sigma^{*}~:~\exists~p\in S/K_{\Sigma}~\mbox{ s.t. }p\in\underline{b}(\set{v})}\\
&\mbox{ $=$ }&\inv{b}(S/K_{\Sigma})
\end{array}
\]
That is, unless $u\in S$ has the form $p\cdot q$ with $q\in K_{\Sigma}$, then $\inv{a}(\set{u})=\emptyset$; for $u\in S$
which is of the required form, it is necessary to identify which arguments these (with the $K_{\Sigma}$ component
removed, i.e. replacing $p\cdot q$, $q\in K_{\Sigma}$ with $p$) attack according to the specification $b$. 

Again, this case
is completed by recalling that regular languages -- which $S$ and $K_{\Sigma}$ are by definition - are closed under the quotient operator (see Fact \ref{fact:closure-props} in \ref{subsection:flt}) and the inductive hypothesis which
ensures that $\inv{b}$ preserves regularity.

\noindent
\textbf{Case 2.3} $a=K_{\Sigma}\cdot b$ (Fact \ref{conjectures-inverse}.\ref{Prop3})

The argument is similar to that used in Case 2.2, so that:
\[
\inv{a}(S)\mbox{ $=$ }\set{~v\in\Sigma^{*}~:~\exists~u\in S~\mbox{ s.t. }u\in K_{\Sigma}\cdot\underline{b}(\set{v})}
\]
Hence if $w\in\Sigma^{*}$ does not have the form $p\cdot q$ for some $p\in K_{\Sigma}$ then $\inv{a}(\set{w})=\emptyset$
otherwise $\inv{a}(\set{w})=\inv{b}(\set{q})$, i.e.
\[
\inv{a}(S)~~=~~\inv{b}(~\set{~q~:~\exists~p\in K_{\Sigma}~\mbox{s.t.}~p\cdot q\in S}~)
\]
However, noting that $rev(L\cdot R)~~=~~rev(R)\cdot rev(L)$, it
follows that $p\in K_{\Sigma}$ and $p\cdot q\in S$ if and only if
$rev(q)\cdot rev(p)\in rev(S)$ and $rev(p)\in rev(K_{\Sigma})$ so that
$\set{~q~:~\exists~p\in K_{\Sigma}~\mbox{s.t.}~p\cdot q\in S})$ is $rev(T)$ where
\[
\begin{array}{lcl}
T&\mbox{ $=$ }&\set{~q\in\Sigma^{*}~:~\exists~p\in rev(K_{\Sigma})\mbox{ s.t. }q\cdot p\in rev(S)~}\\
&\mbox{ $=$ }&rev(S)/rev(K_{\Sigma})
\end{array}
\]
and $\inv{a}(S)~=~\inv{b}(rev(rev(S)/rev(K_{\Sigma})))$ with this case following since regular languages
are closed under $rev()$.

\noindent
\textbf{Case 2.4} $a=b/K_{\Sigma}$ (Fact \ref{conjectures-inverse}.\ref{Prop4})
\[
\begin{array}{lcl}
\inv{a} (S)&\mbox{ $=$ }&\set{~v\in\Sigma^{*}~:~\exists u\in S\mbox{ s.t. }u\in\underline{b}(\set{v})/K_{\Sigma}}\\
&\mbox{ $=$ }&\inv{b}(S\cdot K_{\Sigma})
\end{array}
\]

\noindent
\textbf{Case 2.5} $a=K_{\Sigma}/b$ (Fact \ref{conjectures-inverse}.\ref{Prop5})

\[
\begin{array}{lcl}
\inv{a} (S)&\mbox{ $=$ }&\set{~v\in\Sigma^{*}~:~\exists u\in S\mbox{ s.t. }u\in K_{\Sigma}/\underline{b}(\set{v})}\\
&\mbox{ $=$ }&\set{~v\in\Sigma^{*}~:~\exists u\in S, w \in\underline{b}(\set{v}) \mbox{ s.t. }u \cdot w \in K_{\Sigma}}\\
&\mbox{ $=$ }& \set{~v\in\Sigma^{*}~:~\exists u\in S, w \in\underline{b}(\set{v}) \mbox{ s.t. } rev(w) \cdot rev(u) \in rev(K_{\Sigma})}\\
&\mbox{ $=$ }& \set{~v\in\Sigma^{*}~:~\exists u\in rev(S), w \in rev(\underline{b}(\set{v})) \mbox{ s.t. } w \cdot u \in rev(K_{\Sigma})}\\
&\mbox{ $=$ }& \inv{b}(rev(rev(K_{\Sigma})/rev(S)))
\end{array}
\]

\noindent

\textbf{Case 2.6} $a=b\cap K_{\Sigma}$ (Fact \ref{conjectures-inverse}.\ref{Prop6})
\[
\begin{array}{lcl}
\inv{a}(S)&\mbox{ $=$ }&\set{~v\in\Sigma^{*}~:~\exists u\in S\mbox{ s.t. }u\in (\underline{b}(\set{v})\cap K_{\Sigma})}\\
&\mbox{ $=$ }&\set{~v\in\Sigma^{*}~:~\exists u\in S \cap K_{\Sigma}\mbox{ s.t. }u\in\underline{b}(\set{v})}\\
&\mbox{ $=$ }&\inv{b}(S\cap K_{\Sigma})
\end{array}
\]

\noindent
\textbf{Case 2.7} $a=hd(b)$ (Fact \ref{conjectures-inverse}.\ref{Prop7})
\[
\begin{array}{lcl}

\inv{a}(S)&\mbox{ $=$ }&\set{~v\in\Sigma^{*}~:~\exists~u\in S~\mbox{ s.t. }u\in hd(\underline{b}(v))}\\
&\mbox{ $=$ }&\set{~v\in\Sigma^{*}~:~\exists~\sigma\in S\cap\Sigma,~q\in\Sigma^{*}~\mbox{ s.t. }~\sigma\cdot q\in \underline{b}(v)}\\
&\mbox{ $=$ }&\inv{b}(((S\cap\Sigma)\cdot\Sigma^{*}))
\end{array}
\]
i.e. if $w\in S$ but $w\not\in\Sigma$
then $\inv{a}(\set{w})=\emptyset$, if $\sigma\in S\cap\Sigma$, then $\sigma$ attacks every argument $v$ such that there in an element of $\underline{b}(\set{v})$ having the form $\sigma \cdot z$.

\noindent
\textbf{Case 2.8} $a=tl(b)$ (Fact \ref{conjectures-inverse}.\ref{Prop8})
\[
\inv{a}(S)~~=~~\set{~v\in\Sigma^{*}~:~\exists u\in S\mbox{ s.t. }u\in tl(\underline{b}(\set{v})}
\]
from which $\inv{a}(S)=\inv{b}(\Sigma\cdot S)$.
\end{proof}

\propnmapping*
\begin{proof}
By definition $\pi_{a}^{+}(S) = \set{~v\in\XC~:~\exists~u\in S\mbox{ s.t. }u \in \underline{a}(v) \cap \XC}$, which, since $S \subseteq \XC$, is equivalent to $\set{~v\in\XC~:~\exists~u\in S\mbox{ s.t. }u \in \underline{a}(v)} = \set{~v\in\Sigma^{*}~:~\exists~u\in S\mbox{ s.t. }u \in \underline{a}(v)} \cap \XC = \inv{a}(S) \cap \XC =\inv{a}_{\XC}(S)$.
\end{proof}

\subsection{Proofs of Section \ref{sec:computing}}
\label{sec:proofs-sect-computing}

\effectivealghs*
\begin{proof}
We first observe that validating an instance $\tuple{\tuple{\MC,a},\MC_S}$, where $\MC_S$ is a finite automaton accepting
$S$, as legal simply involves checking $L(\MC_S)\subseteq L(\MC)$, i.e. constructing an automaton $\MC_V$ accepting $L(\MC_S)\setminus L(\MC)$ and checking that $L(\MC_V) = \emptyset$ (see Fact \ref{fact:properties}(a) in \ref{subsection:flt}).

For (a), $S$ is conflict free if and only if $S^{+}\cap S=S^{-}\cap S = \emptyset$. Thus given a {\sc dfa}, $\MC_S$
with $L(\MC_S)\subseteq\XC$ it suffices to check that $L(\MC_S)\cap\pi_{a}^{-}(S)=\emptyset$, i.e. construct
$\MC_{cf}$ accepting $L(\MC_S)\cap\pi_{a}^{-}(S)$  and check that $L(\MC_{cf})=\emptyset$.

In (b), $x\in\FC(S)$ if and only if $(y\in\pi_{a}^{-}(\set{x}))\Rightarrow (y\in\pi_{a}^{+}(S))$ so that
$x\in\FC(S)$ if and only if $\pi_{a}^{+}(S)\supseteq\pi_{a}^{-}(\set{x})$ which can be verified by constructing
suitable automata $\MC_{S}^{+}$ for $\pi_{a}^{+}(S)$, $\MC_{x}^{-}$ for $\pi_{a}^{-}(\set{x})$ and 
checking that $L(\MC_{x}^{-})\setminus L(\MC_{S}^{+})=\emptyset$.

For (c), $S\in\EC_{adm}(\tuple{\XC,\AC})$ if and only if $S$ is conflict free, which can be verified using
the result of part (a) and $S^{+}\supseteq S^{-}$, i.e. every attacker $y$ of an argument in $S$ is counterattacked
by some argument $z$ of $S$. It follows that to check $S\in\EC_{adm}(\tuple{\XC,\AC})$ having
verified that $S$ is conflict free requires only checking $\pi_{a}^{-}(S)\setminus\pi_{a}^{+}(S)=\emptyset$.

Part (d) follows by checking that $S$ is conflict free and $S\cup\pi_{a}^{+}(S)=\XC$.

To show (e), 
first observe that $\XC\setminus\pi^{+}_{a}(S)$ consists of those arguments in $\XC$ that are \emph{not} attacked
by any argument in $S$. It follows that any argument that is attacked by some $y\in\XC\setminus\pi^{+}_{a}(S)$ cannot be acceptable wrt to $S$
since $S$ does not contain any counterattack. The set of arguments attacked by some $y\in\XC\setminus\pi^{+}_{a}(S)$ is just
$\pi^{+}_{a}(\XC\setminus\pi^{+}_{a}(S))$ and, hence, any argument that does not belong to this set, i.e. arguments in
$\XC\setminus\pi^{+}_{a}(\XC\setminus\pi^{+}_{a}(S))$ are acceptable to $S$. If $S$ is a regular language, then since all stages
preserve regularity, i.e. $\pi^{+}_{a}(S)$, $\XC\setminus\pi^{+}_{a}(S)$, $\pi^{+}_{a}(\XC\setminus\pi^{+}_{a}(S))$ and
$\XC\setminus\pi^{+}_{a}(\XC\setminus\pi^{+}_{a}(S))$ are all regular, from Thm.~\ref{thm:attack-comp} and the fact that there are effective algorithms for constructing a {\sc dfa} accepting $S_1 \setminus L_2$ (see Fact~\ref{fact:properties} in \ref{subsection:flt}) 
we can construct the required {\sc dfa}.

Finally (f) is immediate from (a) and (e) and the definition of $\EC_{comp}$.
\end{proof}

\thmsemidec*
\begin{proof}
The approach used is similar for both results and exploits the (propositional) form of the so-called \emph{Compactness Theorem}\footnote{The
property that $\varphi(Z)$, an infinite collection of finite clauses -- or, more generally, finite propositional formulae --
over an enumerable collection of propositional variables,
is satisfiable if and only if every \emph{finite} subset of clauses from $\varphi(Z)$ is satisfiable.}.

The \emph{lexicographic ordering}, $\leq_{{\rm lex}}$ of $\Sigma^{*}\setminus\set{\varepsilon}$ 
has $u\leq_{{\rm lex}}v$ if $|u|<|v|$ or (when $|u|=|v|$) if $u=\sigma_i x$, $v=\sigma_j y$ and $i<j$ or
when $u=\sigma_i x$, $v=\sigma_i y$ if either $x=y=\varepsilon$ or $x\leq_{{\rm lex}} y$. We use
$w_i$ to denote the $i$'th word in $\Sigma^{*}\setminus\set{\varepsilon}$ under this ordering.

Let $Z=\set{z_1,~z_2,\ldots,z_k,\ldots}$ be an enumerable infinite set of propositional variables and
define a bijective mapping $\chi~:~Z\leftrightarrow\Sigma^{*}$ via $\chi(z_i)=w_i$.
For part (a) consider the following collection of clauses $\varphi(Z)$:
\[
\varphi(Z)~~=~~CONF(Z)~\bigwedge~RANGE(Z)
\]
where
\[
CONF(Z)~~=~~\bigwedge_{\set{z_i~:~\chi(z_i)\in\XC}}~\bigwedge_{\set{z_j~:~\chi(z_j)~\in\XC~\mbox{ \& }\chi(z_j)\rightarrow_a\chi(z_i)}}
(\neg z_i\Or\neg z_j)
\]
\[
RANGE(Z)~~=~~\bigwedge_{\set{z_i~:~\chi(z_i)\in\XC}}~~
\left({~z_i~~\Or~~\LOR_{\set{z_j~:~\chi(z_j)~\in\XC~\mbox{ \& }\chi(z_j)\rightarrow_a\chi(z_i)}}~z_j}\right)
\]
Thus if $S\in\EC_{stab}(\tuple{\XC,\AC})$ then the assignment $z_i=\top$ iff $\chi(z_i)\in S$ will satisfy
$\varphi(Z)$ and, conversely, if $\tuple{\alpha_1,\alpha_2,\ldots,a_k,\ldots}$ is a satisfying assignment
to $Z$ for $\varphi(Z)$ then the subset $\set{\chi(z_i)~:~\alpha_i=\top}\in\EC_{stab}(\tuple{\XC,\AC})$.
It follows that $\EC_{stab}(\tuple{\XC,\AC})=\emptyset$ if and only if $\varphi(Z)$ is unsatisfiable, and hence
via the Compactness Theorem, if and only if there is a \emph{finite} subset of clauses from $\varphi(Z)$ that are
collectively unsatisfiable. 

For any subset $S$ of $\Sigma^{*}$ let
\[
Cl(S)~~=~~~ conf(S)~\bigwedge~range(S)
\]

where 

\[
conf(S)~~=~~\bigwedge_{\set{z_i~:~\chi(z_i)\in S}}~\bigwedge_{\set{z_j~:~\chi(z_j)~\in\XC~\mbox{ \& }\chi(z_j)\rightarrow_a\chi(z_i)}}
(\neg z_i\Or\neg z_j)
\]
\[
range(S)~~=~~\bigwedge_{\set{z_i~:~\chi(z_i)\in S}}~~
\left({~z_i~~\Or~~\LOR_{\set{z_j~:~\chi(z_j)~\in\XC~\mbox{ \& }\chi(z_j)\rightarrow_a\chi(z_i)}}~z_j}\right)
\]

For a finite $S$ both $conf(S)$ and $range(S)$ are finite since for each element $z_i$ of $S$ the set of elements $z_j$ corresponding to its attackers is finite.
Moreover note that the set of clauses in $Cl(S)$ is strictly monotonic wrt inclusion (since each additional element $z_i$ entails the addition of at least a clause in $range(S)$) and that for each clause in $CONF(Z)$ and $RANGE(Z)$ there is a $z_i$ such that if $z_i \in S$ then the clause belongs to $Cl(S)$.
 
Now consider the increasing (wrt inclusion) sequence of finite subsets $S$ of $\Sigma^{*}$ obtained by adding incrementally the $k$-th element of $\Sigma^{*} \setminus \set{\varepsilon} $ in the lexicographic order. 
Then for any finite subset $Q$ of clauses from $\varphi(Z)$ it is clearly the case that there is some $S$ in the sequence for which $Q\subseteq Cl(S)$.

These observations yield the method given in Alg.~\ref{algorithm:stable}.

\begin{algorithm}[h]
\caption{Semi-decision process for $\EC_{stab}(\tuple{\XC,\AC})=\emptyset$ in (finitary) {\sc afs}}\label{algorithm:stable}
\begin{algorithmic}[1]
\STATE $S~:=~\emptyset$; 
\STATE $k~:=~0$
\WHILE{ undecided }
\STATE $k~:=~k+1$;
\STATE $S~:=~S~\cup~\set{~w_k~}$;
\IF{ $Cl(S)$ is unsatisfiable }
\STATE \textbf{report true}
\ENDIF
\ENDWHILE
\end{algorithmic}
\end{algorithm}

To establish correctness it is sufficient to note that, by the compactness theorem, $\varphi(Z)$ is unsatisfiable iff some finite subset $Q$ of its
clauses is so, hence iff there is some finite set $S_Q\subseteq\XC$ with $Q\subseteq Cl(S_Q)$ yielding an unsatisfiable
subset of clauses. Since $S_Q$ is finite, such a subset will eventually have $S_Q\subseteq S$ in the algorithm iff
$\varphi(Z)$ is unsatisfiable, i.e. the Alg.~\ref{algorithm:stable} will terminate whenever $\EC_{stab}(\tuple{\XC,\AC})=\emptyset$.

For part (b), the formula $\varphi(Z)~=~CONF(Z)~\And~DEF(Z)~\And~IN(R)$ is used, where $CONF(Z)$ is as previously,
$DEF(Z)$ is
\[
\LAND_{\set{z_i~:~\chi(z_i)\in\XC}}~\LAND_{\set{z_j~:~\chi(z_j)\in\XC~\mbox{ \& }\chi(z_j)\rightarrow_{a}\chi(z_i)}}~
\left({~\neg z_i~\Or~\LOR_{~\set{z_k~:~\chi(z_k)\in\XC~\mbox{ \& }\chi(z_k)\rightarrow_{a}~\chi(z_j)}}~z_k~}\right)
\]
and
\[
IN(R)~~=~~\LOR_{\set{z~:~\chi(z)\in R}}~z
\]
A similar procedure is used to that of Alg.~\ref{algorithm:stable}, however $S$ in l.~1 is initiated
to $R$ (the finite subset of $\XC$ forming part of the problem instance) and $Cl(S)$ in l.~6 is replaced by
$conf(S) \wedge def(S) \wedge IN(R)$, so that $\varphi(Z)$ is satisfiable iff every finite subset of its clauses that include the
clause $IN(R)$ is satisfiable: note that the assignment $z_i~:=~\bot$ for all $i$ will satisfy every finite
subset of $\varphi(Z)\setminus\set{IN(R)}$.
\end{proof}

\propnfinitary*
\begin{proof}
Suppose $a\in\AC\EC(\Sigma)$ satisfies the conditions of the
proposition statement. Consider any $x\in\XC$. If, in contradiction to the
claim, $\pi_{a}^{-}(\set{x})$ is unbounded then $\underline{a}(\set{x})$ must yield an infinite language.
Let $a$ be a smallest (wrt size) member of $\AC\EC(\Sigma)$ with this property.
Clearly $size(a)>0$ since all $a$ with $size(a)=0$ have $|\underline{a}(\set{x})|=1$.
Then $a$ must have one of the forms $\set{b\cup c, K_{\Sigma}\cdot b, b\cdot K_{\Sigma}, K_{\Sigma}/b, b/K_{\Sigma}, b\cap K_{\Sigma}, tl(b), hd(b)}$, where $size(b) < size(a)$ and $size(c) < size(a)$, hence $|\underline{b}(\set{x})|$ and $|\underline{c}(\set{x})|$ are finite.
The expression, $K_{\Sigma}$,
uses only operations from $\set{\cdot,+}$ and it is easily shown that these cannot generate an infinite subset of $\Sigma^{*}$. Then it is easy to see that all the operators above give rise to a finite language, i.e. $\underline{a}(\set{x})$ is finite.
\end{proof}

\bibliographystyle{plain}
\bibliography{comma-arg,complexity-arg,bib,example,missing,motivation,review}
\end{document}